\documentclass[final,nopreprintline,1p,times,authoryear]{elsarticle}
\usepackage{amsmath,amssymb,mathtools,amsthm}
\usepackage{booktabs,array,tabularx}
\usepackage[dvipsnames]{xcolor}
\usepackage{graphicx}
\usepackage{microtype}
\usepackage{natbib}
\usepackage{hyperref}
\usepackage[nameinlink,capitalise]{cleveref}
\graphicspath{{graphics/}{paper/graphics/}}
\hypersetup{colorlinks=true,linkcolor=MidnightBlue,citecolor=MidnightBlue,urlcolor=MidnightBlue}
\newcommand{\R}{\mathbb{R}}
\newcommand{\E}{\mathbb{E}}
\newcommand{\ind}{\mathbb{I}}
\newtheorem{proposition}{Proposition}[section]

\crefname{proposition}{proposition}{propositions}
\crefname{lemma}{lemma}{lemmas}
\crefname{corollary}{corollary}{corollaries}
\title{Relational Compression\\[0.35ex]
{\large\normalfont A Framework for Relational Fidelity in Constrained Representations}}
\begin{document}
\begin{frontmatter}
\author{Yaniv Shulman}
\ead{yaniv@shulman.info}
\begin{abstract}
What should a compressed representation preserve when the information of interest lies in relationships among elements rather than in the elements themselves?
We formulate relational compression in the classical source--description--reconstruction sense, but with relational structure itself as the fidelity-bearing content.
Each instance specifies the source relation, retained description, reconstructed or evaluated relation, fidelity criterion, and constrained resource.
We use this interface to situate selected methods from graph summarization, spectral sparsification, similarity-preserving representation, and relational distillation within a common formulation while keeping their different reconstruction and resource assumptions explicit.

We develop finite-codeword collision as one concrete realization.
Same-codeword probability yields a relational geometry linking pair-specific alignment and separation to aggregate R\'enyi-2 occupancy and the spherical geometry of categorical assignments, with exact objective correspondences to squared-Euclidean centroid reconstruction and normalized graph association and cut.
Graph and image studies illustrate complementary routes within the finite-codeword family: graph- and teacher-defined relational requirements act directly on equality or collision, while reconstruction acts through a joint decoder.
Together, these results illustrate how distinct relational requirements can be formulated and tested within a common constrained-representation framework.
\end{abstract}
\end{frontmatter}
\section{Introduction}
\label{sec:introduction}

Compression asks which distinctions in a source should survive a constrained representation and which can be discarded.
In a reconstruction problem, this choice is expressed through the cost of reproducing the source inaccurately: a finite description is useful when it preserves the features that matter at an acceptable cost~\citep{gray1998quantization,cover2006elements}.
For many sources, however, the structure of interest lies in relationships among elements.
A graph specifies connections among vertices, a distance or similarity matrix records geometry among observations, and a teacher representation induces affinities among examples.
Preserving these relationships gives a natural compression question: how can their structure be represented more economically while retaining the distinctions selected by a fidelity criterion?

Several seemingly unrelated methods approach this question, some explicitly and others through representation learning.
Graph summaries retain reduced descriptions of adjacency structure~\citep{riondato2017graph}; similarity-preserving hashing represents observations by compact codes~\citep{weiss2008spectral}; and relational distillation transfers relationships from teacher to student representations~\citep{park2019relational,tung2019similarity}.
Their descriptions, fidelity criteria, and constrained resources differ, but they can be examined through the same set of choices: the source relation, the retained description, the relation recovered from it, the fidelity of that recovery, and its resource cost.
Making these choices explicit separates representation, decoder, fidelity, and resource: the same retained assignments can support different relational decoders, while similar nominal capacity or fidelity under one criterion need not imply the same code organization or preservation under another.

We formulate \emph{relational compression} in the classical source--description--reconstruction sense, with relational structure as the fidelity-bearing source content.
For a fixed or shared carrier set, its basic interface is
\begin{equation}
S\longrightarrow m_S\longrightarrow\widetilde S,
\label{eq:relational-interface-intro}
\end{equation}
where $S$ is the source relation, $m_S$ is its retained description, and $\widetilde S$ is the relation reconstructed or evaluated from that description.
A relational distortion specifies how preservation of source structure is assessed, while a resource functional measures the cost of retaining it.
Together, these components turn relational preservation into a constrained compression problem: what relational fidelity can be retained at a given description resource?
Specific instances of this pattern occur across established literatures, including explicit relational and graph-compression formulations discussed in \cref{sec:methods-map}.
The contribution here is to make the source relation, retained description, reconstructed or evaluated relation, fidelity criterion, and resource convention explicit as a common cross-domain interface while retaining the assumptions specific to each construction.

Within this framework, we develop finite-codeword equality as a representation mechanism.
A finite assignment groups elements according to their codewords; probabilistic assignments replace hard equality by the probability of a same-codeword collision.
This gives a setting in which to study how a source relation shapes the representation through alignment and separation, and how aggregate codeword use describes its organization.
The main contributions are:
\begin{itemize}
    \item \textbf{A framework for relational compression.}
    We formalize relational compression as a source--description--reconstruction problem in which relational structure is the fidelity-bearing content.
    The framework explicitly separates the source relation, retained description, reconstructed or evaluated relation, relational fidelity criterion, constrained resource, and shared context.
    This provides a common interface for formulating and comparing relation-preserving methods across domains while retaining the assumptions specific to each construction.

    \item \textbf{Finite-codeword collision as a reusable realization.}
    Within this framework, we develop same-codeword collision as a general representation-side mechanism for finite representations.
    The resulting construction connects pair-specific alignment and separation, aggregate R\'enyi-2 occupancy, reference-weighted affinity, and positive-spherical geometry, and exposes exact correspondences with squared-Euclidean centroid reconstruction and normalized graph association and cut.
    Controlled synthetic, graph, and image studies illustrate how the same collision machinery can be instantiated under different relational and reconstruction-defined fidelity requirements.
\end{itemize}

The remainder of the paper is organized as follows.
\Cref{sec:relational-compression-framework} defines the relational-compression problem and its general fidelity constructions.
\Cref{sec:methods-map} presents the method mappings and worked summary example.
\Cref{sec:collision-geometry,sec:collision-entropy,sec:alignment-separation} develop finite-codeword collision geometry and its objectives, followed by selected realizations in \cref{sec:relational-distortion-realizations} and experimental studies in \cref{sec:experiments}.

\section{A Framework for Relational Compression}
\label{sec:relational-compression-framework}

Relational compression requires identifying the relational structure to preserve, the description retained about it, how that description is decoded or evaluated, and how fidelity is traded against resource.
This section formalizes those components and introduces pairwise, content-based, and operator-based fidelities.
We begin with a simple block-valued summary example to show how they fit together.

Let $V$ be a finite carrier, let $S$ be a pairwise relation on $V$, let $c:V\to\{1,\ldots,K\}$ assign each element to a group, and let $\Xi=(\Xi_{ab})_{a,b=1}^K$ be a $K\times K$ table of reproduction values.
The retained description $m_S=(c,\Xi)$ defines the block-valued relational summary
\begin{equation}
\widetilde S(v,w)=\Xi_{c(v),c(w)}.
\label{eq:block-valued-relational-summary}
\end{equation}
The assignments determine which pairs share a reproduction value, while the table determines the value reproduced for those pairs.
A fidelity criterion measures how well $\widetilde S$ preserves $S$, and a resource measure accounts for the retained assignments and table values.
The same finite assignments could support a different decoder, so they do not by themselves uniquely determine the reconstructed relation.
This is a compression formulation because the full relation is represented through the constrained description $m_S=(c,\Xi)$, trading the resource required to retain that description against the fidelity of the relation reconstructed from it by a decoder.

\subsection{Relational sources}

A relational source consists of elements together with structure defined among them.
For a graph, the carrier elements are vertices and the relation records adjacency or edge weights; for a collection of observations, the carrier elements are the observations and the relation may record pairwise distances, similarities, or neighborhood structure.
Separating the elements from their relation identifies the content whose preservation will be assessed.

Let $V$ be a finite carrier set.
A pairwise relational object is
\begin{equation}
\mathfrak R=(V,S),
\qquad
S:V\times V\to\mathcal Y,
\end{equation}
where $S(v,w)$ is the relation value assigned to the pair $(v,w)$.
The codomain $\mathcal Y$ specifies the type of relation value, and the admissible class $\mathcal S_V\subseteq\mathcal Y^{V\times V}$ specifies its structural properties, such as symmetry or nonnegativity.
Reconstructed relations belong to a class $\widehat{\mathcal S}_V\subseteq\widehat{\mathcal Y}^{V\times V}$.
Thus a reconstructed object is $\widetilde{\mathfrak R}=(V,\widetilde S)$, with the same carrier and a reproduced relation.

We treat the carrier identities as fixed or shared between encoder and decoder.
The compression target is therefore $S$ conditional on $V$.
Additional shared context $C$ may contain a reproduction alphabet, a carrier ordering, or model parameters used across source objects.
The retained description contains the source-specific information required for reconstruction beyond this context.
We suppress $C$ in the notation when it is fixed.

In learned realizations, an element $v$ may be presented to the encoder through an observation $x_v$, with $X_V=(x_v)_{v\in V}$ denoting the collection of inputs.
These observations provide the information from which the retained description is constructed, while the relational structure $S$ specifies the relationships among carrier elements that are to be represented.
For example, a graph encoder may receive node features $x_v$ while the source relation $S$ is the graph adjacency or edge-weight structure.
The same per-object formulation applies to datasets of relational objects $(V_i,S_i)$ with varying carrier sets.

\subsection{Descriptions and reconstruction}

The block-valued summary above already separates the stored description from the relation reconstructed from it.
We now formalize that distinction.

A compression scheme specifies a description space $\mathcal M_V$, an encoder $\mathsf{Enc}_V$, and a decoder $\mathsf{Dec}_V:\mathcal M_V\to\widehat{\mathcal S}_V$.
For a source relation $S$, write
\begin{equation}
m_S=\mathsf{Enc}_V(S),
\qquad
\widetilde S=\mathsf{Dec}_V(m_S).
\label{eq:relational-reconstruction-interface}
\end{equation}
Additional encoder-side observations can be included through $\mathsf{Enc}_V(S,X_V)$.
The description may contain stored relation values, element assignments, coordinates, or parameters of a relational reconstruction rule.
The decoder can materialize $\widetilde S$ or evaluate its entries on demand through $(m_S,v,w)\mapsto\widetilde S(v,w)$.
Either form gives a relation determined by the description and shared context.

With a common source and reconstruction codomain, a scheme is \emph{lossless} on $\mathcal S_V$ when
\begin{equation}
\mathsf{Dec}_V(\mathsf{Enc}_V(S))=S
\qquad
\text{for every }S\in\mathcal S_V.
\label{eq:lossless-relational-compression}
\end{equation}
In the lossy setting, the reconstructed relation approximates the source according to a specified fidelity criterion.
For the block-valued summary in \cref{eq:block-valued-relational-summary}, the table reproduces exactly those source relations that are constant on its evaluated blocks; variation within a block becomes approximation error.
The description and decoder therefore determine the family of relational structures available for reconstruction.

This is the familiar encoder--description--decoder interface of classical compression applied to a relational source~\citep{cover2006elements}.
\Cref{tab:framework-components} summarizes its components together with the fidelity and resource measures introduced next.

\begin{table}[t]
\centering
\small
\setlength{\tabcolsep}{4pt}
\renewcommand{\arraystretch}{1.12}
\setlength{\extrarowheight}{1pt}
\begin{tabularx}{\linewidth}{>{\raggedright\arraybackslash}p{0.18\linewidth}
                            >{\raggedright\arraybackslash}X
                            >{\raggedright\arraybackslash}X}
\toprule
Role & Classical compression & Relational compression \\
\midrule
Source
& Source object $x$
& Relational object $(V,S)$, with $V$ fixed or shared \\
Description
& Retained message $m$
& Retained relational description $m_S$ \\
Reconstruction
& Reproduced object $\widehat x$
& Reconstructed relation $\widetilde S$ \\
Decoder
& Rule mapping $m$ to $\widehat x$
& Rule constructing $\widetilde S$ or evaluating its entries \\
Fidelity
& Distortion $d(x,\widehat x)$
& Relational distortion $\Delta_V(S,\widetilde S)$ \\
Resource
& Cost of the description
& Cost $\Omega(m_S)$ in specified units \\
Shared context
& Information available to both sides
& Carrier set and shared context $C$ \\
\bottomrule
\end{tabularx}
\caption{The compression interface for ordinary and relational source objects.}
\label{tab:framework-components}
\end{table}

Structured descriptions can determine many pair relations from comparatively few retained components; subsequent sections instantiate this principle with block summaries and finite codes.

\subsection{Resources and relational fidelity}
\label{sec:relational-distortion}

Given a description and decoder, compression becomes a trade-off between description cost and relational distortion: the resource measure records the former and the fidelity criterion the latter.

A resource functional
\begin{equation}
\Omega:\mathcal M_V\to\R_{\ge0}
\end{equation}
measures description cost in specified units, such as encoded bits, stored relation entries, or scalar coordinates.
A relational distortion
\begin{equation}
\Delta_V:\mathcal S_V\times\widehat{\mathcal S}_V\to\R_{\ge0}
\end{equation}
measures distortion of the source relation according to the chosen fidelity criterion.
Whenever exact source reconstruction is admissible, we take $\Delta_V(S,S)=0$.
More generally, zero distortion means that the chosen fidelity criterion assigns no distortion to the reconstruction; it need not imply literal equality between $S$ and $\widetilde S$.
The decoder determines $\widetilde S$ from the retained description and any shared context $C$, while fidelity is assessed against the source $S$.
Thus $\Delta_V$ specifies how well the reconstructed or evaluated relation preserves the source according to the declared criterion.
A training objective may equal this distortion, include it as one term, or optimize a surrogate for it.

For a source relation $S$, the best distortion attainable within the declared description space and decoder under a resource budget $B$ is
\begin{equation}
D_V^\star(B;S)
=
\inf_{\substack{m\in\mathcal M_V\\\Omega(m)\le B}}
\Delta_V\bigl(S,\mathsf{Dec}_V(m)\bigr).
\label{eq:lossy-relational-compression}
\end{equation}
An encoder optimized for this criterion selects, or approximates, a minimizing admissible description for each source relation.
As the budget increases, the feasible set expands, so the best attainable distortion is nonincreasing.
Thus $D_V^\star(B;S)$ describes the resource--distortion trade-off for the source: how much relational distortion remains at each available description budget.

The definition of $\Delta_V$ is deliberately general.
We next give several concrete fidelity constructions used in subsequent sections, beginning with a pairwise criterion.

\subsection{Pairwise relational fidelity}

For pairwise relational sources, a natural construction is to evaluate relation values locally and average their distortion over a declared distribution of carrier pairs:
\begin{equation}
\Delta_V(S,\widetilde S)
=
\E_{(\mathsf V,\mathsf V')\sim P_{\mathrm{pair}}}
\bigl[d_{\mathrm{rel}}(S(\mathsf V,\mathsf V'),\widetilde S(\mathsf V,\mathsf V'))\bigr].
\label{eq:pairwise-relational-fidelity}
\end{equation}
Here $P_{\mathrm{pair}}$ is a probability distribution on carrier pairs, $\mathsf V,\mathsf V'$ denote a random pair drawn according to this distribution, and $d_{\mathrm{rel}}:\mathcal Y\times\widehat{\mathcal Y}\to\R_{\ge0}$ is a local relation-value distortion function.
Uniform pair weights with squared error give average relation-value reconstruction error.
Pair distributions concentrated on graph edges, neighboring observations, or selected queries emphasize those portions of the source instead.
Thus $P_{\mathrm{pair}}$ determines which source relations receive weight, while $d_{\mathrm{rel}}$ determines how relational reconstruction error is measured locally.

This construction is useful when relational fidelity is naturally assessed through individual relation values.
Other applications evaluate larger-scale structural properties of the relation; we consider two such constructions next.

\subsection{Selected structural fidelity constructions}

Depending on which properties of the source relation are important, a criterion may instead evaluate retained structural content, the action of the relation on a family of probes, or another source-dependent property.
We develop two related constructions here because they are used in the graph realizations later in the paper, while other source-dependent fidelity constructions remain possible.

\paragraph{Content-based fidelity}

Some fidelity criteria are naturally expressed in terms of relational content retained during compression.
For example, deleting graph edges can remove direct connections and the paths supported by them.
A content functional assigns a value to the structure that remains, allowing fidelity to be measured by the decrease in that value.

Let $\preceq_{\mathrm{rel}}$ be a compatible, composable simplification relation on the family under consideration.
A nonnegative structural-content functional $\mathcal I$ is monotone when
\begin{equation}
S_2\preceq_{\mathrm{rel}}S_1
\quad\Longrightarrow\quad
\mathcal I(S_2)\le\mathcal I(S_1).
\label{eq:relational-content-monotone}
\end{equation}
For $\widetilde S\preceq_{\mathrm{rel}}S$, define
\begin{equation}
D_{\mathcal I}(S,\widetilde S)
=\mathcal I(S)-\mathcal I(\widetilde S).
\label{eq:relational-content-distortion}
\end{equation}
The resulting distortion is nonnegative and increases or stays constant along a sequence of admissible simplifications.
For $\mathcal I(S)>0$, the normalized version is
\begin{equation}
\overline D_{\mathcal I}(S,\widetilde S)
=1-\frac{\mathcal I(\widetilde S)}{\mathcal I(S)}.
\label{eq:normalized-relational-content-distortion}
\end{equation}
It measures the fraction of source content lost: exact retention gives zero, and a null reconstruction with zero content gives one.

Choosing $\mathcal I$ specifies the structural property to be preserved.
For example, total retained edge weight gives direct relation distortion, while path- or walk-based content values connections supported over several steps, and nonnegative weighted sums combine several such properties.
The graph realizations later in the paper use this construction to connect simple relational deletions with their effects on source structure.

\paragraph{Operator-based fidelity}

Given a fixed carrier $V$, we use a relation-derived operator to mean a linear map $T_S:\R^V\to\R^V$ determined by the source relation $S$.
For a graph, familiar examples include adjacency, transition, and Laplacian operators; the symmetric fidelity construction below and the graph realizations later in the paper primarily use the Laplacian.
A relation can also be evaluated through the behavior it supports on its carrier.
For a graph, a signal assigns a value to each vertex, and the graph Laplacian measures how that signal varies across connections~\citep{shuman2013emerging}.
Such signals act as probes: comparing their energy before and after compression measures which aspects of the source geometry have been retained.
This leads to a fidelity construction that will supply several of the graph objectives used later.

Let $T_S$ and $T_{\widetilde S}$ be finite-dimensional symmetric operators on signals over $V$, and suppose
\begin{equation}
0\preceq T_{\widetilde S}\preceq T_S,
\label{eq:operator-compression-order}
\end{equation}
where $\preceq$ is the Loewner order.
For a fixed probe $f$, the removed energy is
\begin{equation}
D_f(S,\widetilde S)
=f^\top(T_S-T_{\widetilde S})f.
\label{eq:operator-probe-distortion}
\end{equation}
The operator order ensures that this quantity is nonnegative.
It measures fidelity along the particular signal direction selected by $f$.

To evaluate a family of directions, let $F$ be a random probe with finite second moment $M=\E[FF^\top]$.
The average distortion is
\begin{equation}
\begin{aligned}
D_M(S,\widetilde S)
&=\E\bigl[F^\top(T_S-T_{\widetilde S})F\bigr]\\
&=\operatorname{tr}\bigl[(T_S-T_{\widetilde S})M\bigr].
\end{aligned}
\label{eq:operator-expected-distortion}
\end{equation}
The matrix $M$ determines how signal directions are weighted.
A rank-one choice $M=ff^\top$ recovers the fixed-probe distortion, while $M=I$ weights coordinate directions equally.
Source-derived choices can emphasize modes of the original relation; these are computed from that source and held fixed while candidate reconstructions are compared.
For fixed $M\succeq0$, the average operator distortion is itself a content-based distortion: defining $\mathcal I_M(S)=\operatorname{tr}(T_SM)$ makes \cref{eq:operator-expected-distortion} exactly the corresponding content difference.
The operator viewpoint remains useful because it provides a structured way to specify which signal directions should be preserved and supports, for example, a worst-case criterion that asks how much relative energy can be lost in any source-supported direction:
\begin{equation}
D_{\mathrm{wc}}(S,\widetilde S)
=\sup_{f:f^\top T_Sf>0}
\frac{f^\top(T_S-T_{\widetilde S})f}{f^\top T_Sf}.
\label{eq:operator-worst-case-distortion}
\end{equation}
Under \cref{eq:operator-compression-order}, this lies in $[0,1]$ for $T_S\ne0$; we set it to zero when $T_S=0$.
The average criterion expresses fidelity under a chosen probe distribution, while the worst-case criterion measures the largest relative distortion.
Spectral sparsification gives a related all-probe graph objective, using reweighted sparse reconstructions to preserve Laplacian quadratic forms across signal directions~\citep{spielman2011effective}.
These constructions illustrate several ways in which $\Delta_V$ can be specified.
Other fidelity criteria may be used when different relational properties are important.

\section{Existing Methods Through the Framework}
\label{sec:methods-map}

This section identifies the framework components in representative existing methods and then develops the block-valued summary of \cref{eq:block-valued-relational-summary} as a worked example.
Readers interested primarily in the finite-codeword construction may proceed directly to \cref{sec:collision-geometry}.

\subsection{Representative mappings}

Relational compression has explicit precedents, including the fuzzy-relational image-compression construction of \citet{hirota1999fuzzy} and graph-specific rate--distortion formulations~\citep{lynn2021compressibility,wafula2023rate}.
\Cref{tab:method-framework-map} places these alongside graph approximation, compact representation, and relational transfer while retaining each construction's own fidelity and resource convention.

\begin{table}[t]
\centering
\small
\setlength{\tabcolsep}{3pt}
\renewcommand{\arraystretch}{1.12}
\setlength{\extrarowheight}{1pt}
\begin{tabularx}{\linewidth}{>{\raggedright\arraybackslash}p{0.19\linewidth}
                            >{\raggedright\arraybackslash}X
                            >{\raggedright\arraybackslash}X
                            >{\raggedright\arraybackslash}X}
\toprule
Construction & Source and fidelity & Description and reconstruction & Resource \\
\midrule
Graph summaries~\citep{riondato2017graph}
& Adjacency structure; reconstruction or cut-norm error
& Vertex groups and superedge values; block-valued relation
& Assignments and summary values \\

Spectral sparsification~\citep{spielman2011effective}
& Laplacian quadratic forms; relative approximation
& Sparse weighted edges; reconstructed Laplacian
& Edge count or encoded weighted-edge description \\

Spectral hashing~\citep{weiss2008spectral}
& Source similarities; similarity-weighted code-distance criterion
& Binary codes; Hamming geometry
& Bits per element \\

$t$-SNE~\citep{vandermaaten2008visualizing}
& Pair-affinity distribution; KL discrepancy
& Low-dimensional coordinates; normalized pair affinity
& Coordinate dimension \\

Distance-wise relational distillation~\citep{park2019relational}
& Teacher pair distances; relational discrepancy
& Student features or student model; induced pair distances
& Feature or model representation \\

Centroid quantization~\citep{gray1998quantization}
& Squared source distances; within-code pairwise distortion
& Assignments and centroid reproductions
& Assignments and reproduction vectors \\

Normalized cut~\citep{shi2000normalized,dhillon2004kernel}
& Graph affinities; volume-normalized partition criterion
& Finite assignments; partition equality
& Finite partition cardinality \\
\bottomrule
\end{tabularx}
\caption{Representative methods expressed through the core components of the relational-compression framework: source relation, fidelity criterion, retained description, reconstruction, and resource.}
\label{tab:method-framework-map}
\end{table}

\paragraph{Graph descriptions}
Graph summarization is a broad literature on constructing smaller or more interpretable descriptions of graph structure, with objectives including storage reduction, query support, visualization, and structural pattern discovery~\citep{liu2018graph,shabani2024comprehensive}.
In the family studied by \citet{riondato2017graph}, vertex groups and superedge densities reconstruct a block-valued approximation to the source graph.
A pair's group indices select a stored reproduction value; the worked example below makes this decoder explicit.

Spectral sparsification emphasizes the behavior supported by a graph rather than the accuracy of individual adjacency entries.
\citet{spielman2011effective} construct sparse weighted subgraphs whose Laplacian quadratic forms approximate those of the source for every signal.
For a nonzero source Laplacian $L$ and reconstructed Laplacian $\widetilde L$ on the same carrier, with the same nullspace, a corresponding fidelity is
\begin{equation}
D_{\mathrm{spec}}(L,\widetilde L)
=\sup_{f:f^\top Lf>0}
\frac{\bigl|f^\top(L-\widetilde L)f\bigr|}{f^\top Lf}.
\end{equation}
The absolute value measures relative deformation in either direction, accommodating the reweighting of retained edges.
For deletion or attenuation, where $0\preceq\widetilde L\preceq L$, the operator distortion in \cref{eq:operator-worst-case-distortion} measures relative energy removed.
The sparse edge set and its weights determine the reconstructed Laplacian and the description cost.

\paragraph{Codes and coordinates}
Compact element representations allow relations to be evaluated without storing a value for every pair.
Spectral hashing~\citep{weiss2008spectral} uses binary codes whose Hamming geometry reflects source similarities.
The finite description is the collection of codes, and the number of bits per code specifies its per-element capacity.
A hash function used to encode new observations has a separate storage or shared-model convention.
The relational decoder here is Hamming distance: distinct codewords can represent nearby elements as well as distant ones.

The $t$-SNE construction~\citep{vandermaaten2008visualizing} gives a continuous-coordinate example.
Its source is a normalized pair-affinity distribution $S$ on distinct ordered pairs.
Coordinates $y_i$ determine the reconstructed distribution
\begin{equation}
\widetilde S_{ij}
=\frac{(1+\|y_i-y_j\|_2^2)^{-1}}
{\sum_{a\ne b}(1+\|y_a-y_b\|_2^2)^{-1}},
\qquad i\ne j,
\end{equation}
and the final objective is $D_{\mathrm{KL}}(S\|\widetilde S)$.
The stored coordinates suffice to evaluate the affinities, including their normalization.
This mapping uses coordinate dimension as the structural resource.
A finite-bit version additionally specifies coordinate precision and evaluates the relations produced by those quantized coordinates.

\paragraph{Relational transfer}
A teacher can supply the source relation even when the original observations remain available.
The distance-wise component of relational knowledge distillation~\citep{park2019relational} uses normalized teacher pair distances to supervise the relations induced by student features.
For a fixed evaluation collection, the student features and the prescribed distance-normalization rule determine the reconstructed relation.
An inductive interpretation instead retains a student model and computes the features from available observations.
The constrained representation is feature storage in the first case and the student model in the second; its size and precision specify the corresponding resource.
The method also includes an angle-wise term defined on triples of examples, providing a concrete higher-arity relational example.
\ref{app:higher-order-collision-geometry} develops higher-order collision constructions, while noting that source-relation arity and collision arity are distinct; the main development here focuses on the pairwise case.

\paragraph{Classical objective correspondences}
Relational structure can also emerge from reformulating an objective that begins with a different description of fidelity.
Squared-error centroid reconstruction admits an exact within-code pairwise form, while normalized cut evaluates how a partition preserves graph affinities after accounting for group volumes.
\Cref{sec:vq,sec:normalized-cut-realization} develop these correspondences using inverse-mass-weighted code affinity.
Their role is to expose common mathematical structure: the full quantizer retains reproduction vectors as well as assignments, whereas the partition can be evaluated against the source graph using its equality relation alone.
These examples illustrate why specifying a relational source and its fidelity can be informative even when the original objective is introduced through reconstruction or partitioning.

\subsection{A worked example: block-valued relational summaries}
\label{sec:block-summary-walkthrough}

We now return to the block-valued relational summary introduced in \cref{eq:block-valued-relational-summary} and make its fidelity and resource accounting explicit.
Following this description idea from graph summarization~\citep{riondato2017graph}, we derive the optimal block values under squared relational error and then account for their finite-precision storage.
The example shows how element assignments, together with a reproduction table, support a reconstructed relation richer than same-group equality.
By contrast, beginning with \cref{sec:collision-geometry}, the subsequent finite-codeword development focuses on equality and its probabilistic collision relaxation as the primitive representation-side relation.

Let $V=\{1,\ldots,n\}$ with $n\ge2$, and let $S$ be a symmetric real-valued relation with a fixed zero diagonal.
We evaluate the unordered distinct pairs
\begin{equation}
\mathcal P=\{(i,j):1\le i<j\le n\}.
\end{equation}
Choose an assignment $c:V\to\{1,\ldots,K\}$ and a symmetric table $\Xi\in\R^{K\times K}$.
The retained description is $m_S=(c,\Xi)$, and its decoder is
\begin{equation}
\widetilde S_{ij}=\Xi_{c(i),c(j)}\quad(i\ne j),
\qquad
\widetilde S_{ii}=0.
\label{eq:block-summary-decoder}
\end{equation}
An entry $\Xi_{aa}$ reproduces relations between distinct elements assigned to group $a$.
Thus the assignments determine which pairs share a reproduction, and the table determines the value reproduced.

The squared relational distortion is
\begin{equation}
D_{\mathrm{block}}(S;c,\Xi)
=\frac{1}{|\mathcal P|}
\sum_{(i,j)\in\mathcal P}
\bigl(S_{ij}-\Xi_{c(i),c(j)}\bigr)^2.
\label{eq:block-summary-distortion}
\end{equation}
For $a\le b$, define
\[
B_{ab}(c)
:=
\{(i,j)\in\mathcal P:
\min\{c(i),c(j)\}=a,\;
\max\{c(i),c(j)\}=b\},
\]
so $B_{ab}(c)$ contains exactly the evaluated pairs whose two assigned groups are $a$ and $b$, irrespective of their order.
Once the assignment is fixed, choosing the table becomes a collection of independent reproduction problems, one for each nonempty block.
The variance identity supplies both the optimal reproduction and its error decomposition.

\begin{proposition}[Optimal block reproduction]
\label{prop:block-means}
For a fixed assignment $c$, each block with $|B_{ab}(c)|>0$ is optimally reproduced by
\begin{equation}
\Xi_{ab}^\star
=\frac{1}{|B_{ab}(c)|}
\sum_{(i,j)\in B_{ab}(c)}S_{ij}.
\label{eq:block-optimal-mean}
\end{equation}
For any symmetric table $\Xi$,
\begin{equation}
\begin{aligned}
D_{\mathrm{block}}(S;c,\Xi)
={}&D_{\mathrm{block}}(S;c,\Xi^\star)\\
&+\frac{1}{|\mathcal P|}
\sum_{\substack{a\le b\\|B_{ab}(c)|>0}}
|B_{ab}(c)|(\Xi_{ab}-\Xi_{ab}^\star)^2.
\end{aligned}
\label{eq:block-mean-decomposition}
\end{equation}
Empty blocks may be assigned a fixed default value.
\end{proposition}

\begin{proof}
In each nonempty block, expand
$S_{ij}-\Xi_{ab}=(S_{ij}-\Xi_{ab}^\star)+(\Xi_{ab}^\star-\Xi_{ab})$.
The cross term vanishes because
$\sum_{(i,j)\in B_{ab}(c)}(S_{ij}-\Xi_{ab}^\star)=0$.
Summing the remaining squared terms over blocks gives \cref{eq:block-mean-decomposition} and the minimizing table.
\end{proof}

The optimal table averages the relation values that the assignment has placed together.
For binary adjacency, each reproduction is the edge density of its block.
For a weighted source, it is the corresponding mean relation value.
The first term in \cref{eq:block-mean-decomposition} measures variation within these blocks; the second measures the additional error from using any other table.
The same calculation extends directly to a nonuniform pair distribution $P_{\mathrm{pair}}$ from \cref{eq:pairwise-relational-fidelity}, yielding weighted block means.

For example, consider
\begin{equation}
S=\begin{pmatrix}
0&0&1&3\\
0&0&3&1\\
1&3&0&0\\
3&1&0&0
\end{pmatrix},
\qquad
c=(1,1,2,2).
\label{eq:block-toy-source}
\end{equation}
The two within-group relation values are zero, while the four cross-group values alternate between one and three.
Averaging uniformly over the six unordered pairs, the optimal table and reconstruction are
\begin{equation}
\Xi^\star=\begin{pmatrix}0&2\\2&0\end{pmatrix},
\qquad
\widetilde S=\begin{pmatrix}
0&0&2&2\\
0&0&2&2\\
2&2&0&0\\
2&2&0&0
\end{pmatrix}.
\label{eq:block-toy-reconstruction}
\end{equation}
Each cross-group pair has squared error one, giving total average distortion $4/6=2/3$.
The summary retains the source's cross-group organization while averaging the variation within it.
By contrast, masking the source to retain only pairs with $c(i)=c(j)$ gives the zero matrix for the same assignment and has distortion $20/6=10/3$.
The table decoder preserves the strong relations between groups that this mask removes.

\paragraph{Description cost and relation-value quantization}
To describe the summary with finitely many bits, we also specify how its table values are represented.
Choose a shared reproduction alphabet containing at most $2^t$ values, where $t\ge1$ is an integer.
Conditional on $n$, $K$, the carrier ordering, and this alphabet, a fixed-length scheme that stores every assignment and all upper-triangular table entries uses
\begin{equation}
\Omega_t(c,\Xi)
=n\lceil\log_2K\rceil+\frac{K(K+1)}{2}\,t
\label{eq:block-description-cost}
\end{equation}
bits.
The first term describes which group contains each element; the second describes the relations between groups.
This makes the allocation of resources explicit: a larger group alphabet permits a finer block structure but requires more assignment or table capacity.

For a fixed assignment, \cref{eq:block-mean-decomposition} also identifies the optimal quantized table.
Each active entry is a nearest permitted value to $\Xi_{ab}^\star$, and its contribution above the unquantized optimum is $|B_{ab}(c)|/|\mathcal P|$ times the squared quantization error.
Grouping and table precision therefore contribute separately to the relational distortion within one complete description.
Allowing at most $K$ groups gives nested unquantized approximation families as $K$ increases, since a finer table can reproduce a coarser one.
Combining this flexibility with \cref{eq:block-description-cost} yields a concrete instance of the resource--distortion problem in \cref{eq:lossy-relational-compression}.

\section{Finite-Codeword Collision Geometry}
\label{sec:collision-geometry}

Having defined relational fidelity independently of representation, we now study finite codewords as one concrete description family.
A hard code induces same-codeword equality and hence a partition into at most $K$ classes; probabilistic assignments relax this to the probability that independently drawn codewords collide.
This collision probability is the representation-side geometry used below, linking smooth optimization to hard partitions, structured binary parameterizations, and codeword-prevalence corrections.

\subsection{Probabilistic finite encoders}

A hard finite assignment is discrete, whereas learned representations are often optimized through continuous parameters.
We therefore begin by replacing each hard codeword with a categorical distribution over the same finite alphabet.
This retains the finite set of available states while providing a smooth representation of assignment uncertainty.

Let $\mathcal Z=\{1,\ldots,K\}$ and let an encoder map an input $x$ to
\begin{equation}
q_\theta(\cdot\mid x)\in\Delta^{K-1},\qquad
\Delta^{K-1}=\{q\in\R_{\ge0}^K:\textstyle\sum_z q_z=1\}.
\label{eq:discrete-encoder}
\end{equation}
The map $x\mapsto q_\theta(\cdot\mid x)$ may itself be deterministic.
Stochasticity enters only when a codeword is drawn from the resulting categorical distribution.

Consider inputs $x,x'$ associated with two carrier elements, and take conditionally independent draws
$Z\sim q_\theta(\cdot\mid x)$ and
$Z'\sim q_\theta(\cdot\mid x')$.
Their probability of receiving the same codeword is
\begin{equation}
k_\theta(x,x')
=\sum_zq_\theta(z\mid x)q_\theta(z\mid x')
=\Pr[Z=Z'\mid x,x'].
\label{eq:collision-kernel}
\end{equation}
Thus collision is a soft version of the equality relation supplied by a finite code.
It is large when the two assignment distributions place mass on the same codewords and becomes exact same-class membership when the distributions become deterministic.
Because this collision event involves two independently encoded inputs, it is the order-two, or pairwise, collision probability.
We write $k_{\theta,2}$ when the collision order needs to be explicit, but abbreviate it as $k_\theta$ throughout the main text.
More generally, an integer-order extension requires $\alpha\ge2$ independently encoded inputs to receive the same codeword; it is developed in \ref{app:higher-order-collision-geometry}.

The same expression also has a useful geometric interpretation.
It is the ordinary inner product of two categorical probability vectors, and is therefore a discrete probability-product kernel~\citep{jebara2004probability}.
In particular, it is symmetric, lies in $[0,1]$, and is positive semidefinite.
For any inputs $x_i$ and real coefficients $a_i$,
\begin{equation}
\sum_{i,j}a_i a_j k_\theta(x_i,x_j)
=\left\|\sum_i a_iq_\theta(\cdot\mid x_i)\right\|_2^2\ge0.
\end{equation}

Viewing collision as an inner product also shows what kinds of pairwise geometry this representation can express.
On a finite sample, let $Q$ contain the categorical assignment vectors as its rows, so the collision Gram matrix is $\Gamma_\theta=QQ^\top$.
Hence $\operatorname{rank}(\Gamma_\theta)\le K$: a $K$-codeword collision representation can express at most $K$ independent spectral directions of the pairwise relation.
Raw collision therefore represents a restricted family of relational structures.
The broader framework also includes decoders such as the block-valued construction in \cref{sec:block-summary-walkthrough}, which need not be limited to this particular form.

There is one further distinction specific to probabilistic assignments.
The diagonal of the collision kernel is
\begin{equation}
k_\theta(x,x)=\sum_zq_\theta(z\mid x)^2
=\|q_\theta(\cdot\mid x)\|_2^2\in[1/K,1].
\label{eq:self-collision}
\end{equation}
For a fixed input $x$, let $Z_1,Z_2\overset{\mathrm{iid}}{\sim}q_\theta(\cdot\mid x)$.
Then $k_\theta(x,x)=\Pr[Z_1=Z_2\mid x]$, so the diagonal of the collision kernel measures the concentration of the conditional assignment distribution.
A uniform assignment gives $k_\theta(x,x)=1/K$, while a deterministic assignment gives $k_\theta(x,x)=1$.
Hence the diagonal entries of a soft collision matrix lie in $[1/K,1]$, and become one in the deterministic hard-partition limit.

\subsection{Hard partitions and their relaxation}

The probabilistic formulation is useful for optimization, but the finite representation of primary interest is often a hard code.
The connection between the two is direct.
A deterministic assignment $c_\theta:\mathcal X\to\mathcal Z$ is represented by the one-hot distribution
$q_\theta(z\mid x)=\ind[z=c_\theta(x)]$.
Substituting this distribution into \cref{eq:collision-kernel} gives
\begin{equation}
k_\theta(x,x')=\ind[c_\theta(x)=c_\theta(x')].
\label{eq:hard-collision}
\end{equation}
The collision relation is then exactly membership in the same part of the partition induced by $c_\theta$.

The same conclusion holds as a limit.
Let $\{q_n(\cdot\mid x)\}_{n\ge1}$ be a sequence of categorical assignment distributions that, for every input under consideration, converges to the one-hot distribution at $c_\theta(x)$.
Then
\[
\sum_z q_n(z\mid x)q_n(z\mid x')
\longrightarrow
\ind[c_\theta(x)=c_\theta(x')].
\]
Collision can therefore be viewed as a continuous relaxation of hard same-partition membership.

At nonzero assignment uncertainty, collision is a graded affinity and generally lacks the transitivity of an equivalence relation.
The distinction between the soft assignment geometry and the hard representation is consequential: a soft encoder can distribute mass broadly across the alphabet even when a chosen hardening rule assigns many or all inputs to the same codeword.
Soft and hard relational statistics should therefore be interpreted separately.

\subsection{Factorized binary codes}

A categorical alphabet with $K$ states becomes expensive to represent explicitly when $K$ is large.
A useful specialization obtains an exponentially large finite alphabet from a small number of binary coordinates.
Let the code contain $b$ conditionally independent bits, so that
$\mathcal Z=\{0,1\}^b$ and $K=2^b$.
Writing
$p_r(x)=\Pr[Z_r=1\mid x]$ for the conditional probability that bit $r$ takes value one at input $x$, the probability of a complete binary word
$a=(a_1,\ldots,a_b)$ is
\begin{equation}
q_\theta(a\mid x)
=\prod_{r=1}^bp_r(x)^{a_r}(1-p_r(x))^{1-a_r},
\qquad a\in\{0,1\}^b.
\label{eq:factorized-code}
\end{equation}

The main computational benefit is that collision can be evaluated without enumerating all $2^b$ complete words.
Summing the equality probability over the factorized code distribution gives
\begin{equation}
k_\theta(x,x')
=\prod_{r=1}^b\left[p_r(x)p_r(x')+(1-p_r(x))(1-p_r(x'))\right].
\label{eq:bit-collision}
\end{equation}
Each factor is the probability that the two independently sampled values of bit $r$ agree.
Their product is therefore the probability that the \emph{entire} binary code agrees.
As every $p_r(x)$ approaches zero or one, the expression reduces to deterministic equality of the complete codewords.

The factorization assumes that the bits are independent conditional on each fixed input $x$, which parameterizes $q_\theta(\cdot\mid x)$ as a product distribution.
Across the input population, the probabilities $p_r(x)$ may co-vary, so the aggregate codeword distribution need not factorize across bit coordinates.

The binary parameterization provides a compact factorization of a categorical distribution over $K=2^b$ complete codewords.
Collision treats each complete word as one categorical state, so its categorical representation lives in $K$ dimensions.
The spherical construction developed later applies to these categorical probability vectors.

\subsection{Reference-weighted codeword affinities}

Raw collision treats agreement on every codeword equally.
Some relational objectives instead require agreement to be weighted by codeword prevalence.
Agreement on a codeword used by almost every element may carry less distinguishing information than agreement on a rarely occupied word.
This motivates measuring a shared assignment relative to a reference prevalence for that codeword.

Let $r=(r(z))_{z\in\mathcal Z}$ be a fixed positive reference distribution on $\mathcal Z$.
Define the reference-weighted affinity
\begin{equation}
a_r(x,x')
=\sum_z\frac{q_\theta(z\mid x)q_\theta(z\mid x')}{r(z)}
=\left\langle\frac{q_\theta(\cdot\mid x)}{\sqrt r},
\frac{q_\theta(\cdot\mid x')}{\sqrt r}\right\rangle.
\label{eq:reference-affinity}
\end{equation}
The square root and division in the vector expression are taken coordinatewise.
Like raw collision, this quantity is positive semidefinite because it is an inner product after coordinate rescaling.
Inverse-reference weighting can make this value exceed one, so we treat it as an affinity rather than a probability.

The hard-code interpretation makes the weighting especially transparent.
If two inputs receive the same deterministic word $z$, their affinity is $1/r(z)$; if they receive different words, it is zero.
Agreement on a word with small reference mass is therefore weighted more strongly than agreement on a common word.
For the uniform reference $U_K(z)=1/K$, every word receives the same correction and
$a_{U_K}=Kk_\theta$, so the weighting changes only the global scale.

A particularly important reference is supplied by the encoder's own aggregate codeword distribution.
Fix a probability measure $\mu$ on carrier-element inputs and define
\begin{equation}
q_Z(z)
:=\Pr(Z=z)
=\E_{X\sim\mu}[q_\theta(z\mid X)].
\label{eq:aggregate-code}
\end{equation}
This distribution records how frequently each codeword is used on average under the declared population.
Taking the reference to be this aggregate gives
\begin{equation}
a_{q_Z}(x,x')
=\sum_{z:q_Z(z)>0}\frac{q_\theta(z\mid x)q_\theta(z\mid x')}{q_Z(z)}.
\label{eq:qz-affinity}
\end{equation}
The affinity now corrects agreement by the actual prevalence of the shared codeword: collisions through frequently occupied states are downweighted, while collisions through less common states receive greater weight.

Inactive codewords are omitted from the sum.
Indeed, $q_Z(z)=0$ implies
$q_\theta(z\mid X)=0$ almost surely under $\mu$, so this support convention is well defined on the declared population.
Inputs outside the support of that population would require a separate convention.

This endogenous inverse-mass correction will be important later.
The same form appears in the exact centroid-reconstruction and normalized-association correspondences, where codeword mass determines how within-code agreement should be normalized.
An \emph{external} reference can separately specify a desired marginal organization when used in a prior-matching criterion.

\section{Codeword Organization and Spherical Geometry}
\label{sec:collision-entropy}

Collision also describes how a finite code is organized across a population: averaging it yields aggregate occupancy and R\'enyi-2 effective code count.
We then distinguish external-reference organization from endogenous mass correction and use radial normalization to expose the spherical geometry of categorical assignments.

\subsection{Population collision and effective occupancy}

For $X,X'\overset{\mathrm{iid}}{\sim}\mu$, the population collision probability depends only on the aggregate codeword distribution:
\begin{align}
\E[k_\theta(X,X')]
&=\sum_z\E[q_\theta(z\mid X)]\E[q_\theta(z\mid X')]\\
&=\sum_zq_Z(z)^2.
\label{eq:population-collision}
\end{align}
The right-hand side is the collision probability of the aggregate codeword distribution itself.
It is large when population mass is concentrated on a few codewords and smaller when usage is spread more broadly.
The iid population sampling is essential to this occupancy interpretation: independence makes the average collision depend only on the aggregate codeword probabilities $q_Z(z)$.
For pairs drawn from a relation-specific joint distribution, such as graph edges, the average collision also reflects dependence between the paired assignments and is therefore a different relational statistic.
The measure $\mu$ must therefore be part of the interpretation: it may be uniform over the carrier of a single object, degree-weighted over the vertices of a graph, or defined by a declared pooling of carrier elements across a collection of relational objects.

Because population collision is exactly the squared $\ell_2$ mass of the aggregate codeword distribution, it has a natural information-theoretic representation through order-two R\'enyi entropy~\citep{renyi1961measures}:
\begin{equation}
H_2(q_Z)=-\log\sum_zq_Z(z)^2
=-\log\E[k_\theta(X,X')].
\label{eq:collision-entropy}
\end{equation}
We write $R_2=H_2(q_Z)$ for this collision-occupancy measure.
Entropy is useful analytically, but the same quantity can be expressed as an effective number of equally occupied codewords:
\begin{equation}
K_{\mathrm{eff}}=\exp(H_2(q_Z))
=\frac{1}{\sum_zq_Z(z)^2},\qquad 1\le K_{\mathrm{eff}}\le K.
\label{eq:effective-code-count}
\end{equation}
Uniform use of exactly $M$ words gives $K_{\mathrm{eff}}=M$.
For nonuniform occupancy, the same expression gives the inverse-collision, or inverse-Simpson, effective count, which decreases as probability mass becomes more concentrated.

Collision-effective occupancy is only one summary of code use.
It is useful to place it alongside nominal alphabet size, active support, and Shannon entropy because these quantities coincide only in special cases:
\begin{equation}
\begin{aligned}
R_{\mathrm{nom}}&=\log K,&
R_0&=\log|\operatorname{supp}(q_Z)|,\\
R_1&=-\sum_zq_Z(z)\log q_Z(z),&
R_2&=-\log\sum_zq_Z(z)^2.
\end{aligned}
\label{eq:finite-state-rate-hierarchy}
\end{equation}
They satisfy $R_2\le R_1\le R_0\le R_{\mathrm{nom}}$.
All logarithms are natural unless specified otherwise, and dividing these quantities by $\log 2$ expresses them in bits.
These quantities summarize different aspects of codeword use.
When codewords are encoded under an appropriate probabilistic coding model, the Shannon entropy $H_1(q_Z)$ can also characterize the asymptotic average coding cost per codeword.
Any such encoded payload that forms part of the retained description contributes to the resource functional $\Omega$.

Aggregate occupancy also differs from conditional assignment uncertainty.
If every input has a deterministic assignment, self-collision is one even when population collision is small.
Conversely, the input-independent encoder $q_\theta(\cdot\mid x)=U_K$ has maximal aggregate $H_2$ but conveys no input dependence through its sampled codeword.
For soft encoders, broad aggregate occupancy can therefore coexist with a concentrated hard partition or weak input dependence.
The associated conditional-collision and divergence-based dependence quantities are recorded in \ref{app:conditional-collision-dependence}.

\subsection{External references and endogenous mass correction}

Aggregate collision measures concentration without specifying what pattern of codeword use is preferred.
In some settings, organization is meaningful only relative to a reference distribution.
An external reference can encode a desired marginal pattern, such as equal use of exchangeable codewords, whereas choosing the reference masses as $r(z)=q_Z(z)$ serves a different purpose: it compensates local agreement for the prevalence actually produced by the encoder.
For fixed positive reference masses $r(z)$, independent population averaging gives
\begin{equation}
\E[a_r(X,X')]=\sum_z\frac{q_Z(z)^2}{r(z)},
\qquad
D_2(q_Z\|r)=\log\E[a_r(X,X')].
\label{eq:renyi-reference-affinity}
\end{equation}
Thus reference-weighted affinities have a population interpretation through R\'enyi divergence~\citep{renyi1961measures,vanerven2014renyi}.
When the corresponding divergence is used as an organization penalty, the external reference specifies the target marginal organization of the alphabet.
For exchangeable words, uniform matching gives
\begin{equation}
D_2(q_Z\|U_K)=\log K-H_2(q_Z).
\label{eq:uniform-renyi}
\end{equation}
Minimizing this divergence encourages broader effective occupancy.
Adding $+\lambda H_2(q_Z)$ to a minimized objective does the opposite: it encourages concentration.
The sign of the regularizer therefore determines whether it encourages broad occupancy or concentration.

For comparison, $D_{\mathrm{KL}}(q_Z\|U_K)=\log K-H_1(q_Z)$ also has the uniform optimum.
Although $H_2\le H_1$ for each distribution, the two entropies can rank candidate distributions differently and induce distinct level sets around their common optimum.
The general identity $D_\alpha(q\|U_K)=\log K-H_\alpha(q)$ and the sign conventions are collected with the higher-order material in \ref{app:higher-order-collision-geometry}.

For the endogenous choice $r(z)=q_Z(z)$, inverse-mass normalization gives
\begin{equation}
\E[a_{q_Z}(X,X')]=1.
\label{eq:endogenous-reference-expectation}
\end{equation}
For almost every fixed $x$, even $\E_{X'}[a_{q_Z}(x,X')]=1$.
This unit expectation results from correcting local agreement by the occupancy induced by the encoder.
The aggregate masses $q_Z(z)$ therefore supply endogenous inverse-mass normalization when used as the reference and can separately enter a prior-matching penalty when compared with an external reference.

\subsection{Radial normalization and cosine similarity}

The raw collision kernel is the Euclidean inner product of categorical probability vectors, and therefore combines two geometric effects: the angular agreement between their directions and the magnitudes of their Euclidean norms.
For categorical distributions, these norms are themselves meaningful, since their squared values are self-collision probabilities and hence encode R\'enyi-2 concentration.
It is therefore useful to separate concentration from angular agreement when studying the geometry induced by collision.
Radial normalization does exactly this, mapping categorical probability vectors to the positive unit sphere, where normalized collision becomes cosine similarity.
For $q\in\Delta^{K-1}$, define
\begin{equation}
u(q)=\frac{q}{\|q\|_2}\in\mathbb S_+^{K-1},\qquad
\mathbb S_+^{K-1}=\{u\in\R_{\ge0}^K:\|u\|_2=1\}.
\label{eq:sphere-map}
\end{equation}
The map is bijective, with inverse
\begin{equation}
q=\frac{u}{\|u\|_1}.
\label{eq:sphere-map-inverse}
\end{equation}
A more standard spherical representation of probability distributions is the elementwise square-root map $q\mapsto\sqrt q$, for which the spherical inner product becomes the Bhattacharyya coefficient and which underlies Hellinger and information-geometric constructions~\citep{amari2016information}.
The radial normalization used here serves a different purpose: because collision is the ordinary inner product $q^\top q'$, normalizing by the Euclidean norm isolates the angular component of that collision geometry while preserving a bijective representation of the simplex.

Writing $k_2(q,q')=q^\top q'$, normalized collision is
\begin{equation}
u(q)^\top u(q')
=\frac{k_2(q,q')}{\sqrt{k_2(q,q)k_2(q',q')}}.
\label{eq:normalized-collision}
\end{equation}
Thus normalized collision is exactly cosine similarity on the nonnegative spherical orthant: it keeps the angular agreement between assignments, while raw collision also depends on the concentration of each assignment.\footnote{The negative logarithm of this normalized inner product is the discrete analogue of the Cauchy--Schwarz divergence of \citet{jenssen2006cauchy}.}
This connects the finite-codeword realization to hyperspherical alignment~\citep{wang2020alignment}, while retaining the categorical constraints.
The one-hot representations of hard words occupy the fixed coordinate directions $e_z$.
Relabeling complete codewords permutes these directions; it does not create an arbitrary rotationally invariant embedding family.
A factorized bit parameterization can impose further restrictions on which codeword relabelings preserve its parameterized family.

The normalized similarity factors out the explicit self-collision multipliers from the pairwise score.
\begin{equation}
\|q\|_2^2
=k_2(q,q)
=\sum_z q(z)^2
=\exp[-H_2(q)],
\qquad
\|q\|_2
=\exp[-H_2(q)/2].
\label{eq:collision-norm-entropy}
\end{equation}
If $\cos\vartheta=u(q)^\top u(q')$, then
$k_2(q,q')=q^\top q'=\|q\|_2\|q'\|_2\cos\vartheta$, and therefore
\begin{equation}
k_2(q,q')
=\exp\!\left[-\tfrac12\bigl(H_2(q)+H_2(q')\bigr)\right]\cos\vartheta.
\label{eq:collision-angular-decomposition}
\end{equation}
This decomposition shows that raw collision combines angular agreement with the R\'enyi-2 concentration of each assignment.
The spherical map itself is lossless on the simplex, since $q$ can be recovered from $u(q)$ through \cref{eq:sphere-map-inverse}, while cosine similarity isolates only the angular component.
Thus raw and normalized collision emphasize different aspects of the same categorical geometry.

\subsection{Entropy in spherical coordinates}

Having separated raw collision into angular agreement and R\'enyi-2 concentration in \cref{eq:collision-angular-decomposition}, we now express the entropy-based occupancy measures introduced in \cref{eq:finite-state-rate-hierarchy} in the same spherical coordinates.
This provides a common geometric language for two aspects of the finite-codeword representation: angles between spherical vectors describe normalized collision between assignments, while the geometry of an individual spherical vector describes the concentration and organization of the corresponding categorical distribution.
For $u=q/\|q\|_2$, the inverse relation gives
\begin{equation}
H_2(q)=2\log\|u\|_1.
\label{eq:h2-spherical}
\end{equation}
Let $u_U=\mathbf1/\sqrt K$ and let $\vartheta_U$ be its angle to $u$.
Since $\cos\vartheta_U=\|u\|_1/\sqrt K$,
\begin{equation}
D_2(q\|U_K)=-2\log\cos\vartheta_U
=\log\bigl(1+K\|q-U_K\|_2^2\bigr).
\label{eq:renyi-angular}
\end{equation}
These are exact changes of coordinates.
At the level of one input, increasing $H_2(q_\theta(\cdot\mid x))$ makes the assignment less concentrated.
At the level of the aggregate probabilities $q_Z(z)$, increasing $H_2(q_Z)$ spreads aggregate codeword occupancy.
The same algebra describes different objects depending on which distribution is inserted.

R\'enyi-2 has an especially simple angular description, while Shannon entropy provides another organization measure expressible through $u$:
\begin{equation}
H_1(q)=\log\|u\|_1
-\frac{1}{\|u\|_1}\sum_i u_i\log u_i,
\label{eq:shannon-spherical}
\end{equation}
with $0\log0=0$.
Unlike $H_2$, it depends on the full coordinate profile rather than only the angle to the uniform direction.
The two entropies share collapsed and uniform extrema but generally have different level sets.
The broader R\'enyi expression appears in \ref{app:renyi-spherical-coordinates}.

Taken together, these identities explain why order two provides a particularly coherent description of this finite-codeword realization: local equality probability, population collision, and normalized cosine geometry all arise from the same pairwise collision statistic.
The empirical and modelling choice between R\'enyi-2, Shannon, and other regularizers remains application dependent.

\section{Alignment, Separation, and Relational Requirements}
\label{sec:alignment-separation}

We now use the collision geometry to express source-specific relational requirements.
Favored and disfavored pairs control alignment and separation, while an optional marginal term controls how the alphabet is used; these distinct roles lead to the learning objectives below.

\subsection{From the source relation to a learning objective}

Let $(\mathsf V,\mathsf V')\sim P_{\mathrm{pair}}$ be a declared distribution on carrier pairs, and let $x_{\mathsf V},x_{\mathsf V'}$ denote their corresponding encoder-side observations.
A code-induced pair statistic $\psi_\theta$ may be raw collision, reference-weighted affinity, normalized collision, or a dissimilarity such as $-\log k_\theta$.
A generic pairwise learning objective is
\begin{equation}
\mathcal L_{\mathrm{rel}}(\theta)
=\E_{(\mathsf V,\mathsf V')\sim P_{\mathrm{pair}}}\left[\ell\bigl(S(\mathsf V,\mathsf V'),\psi_\theta(x_{\mathsf V},x_{\mathsf V'})\bigr)\right].
\label{eq:generic-relational-objective}
\end{equation}
The statistic may itself be the reconstructed relation, or it may control a decoder or selector used to form that relation.
Whether this objective equals the declared relational distortion depends on that reconstruction or evaluation convention; otherwise it may serve as a training surrogate for the desired fidelity.

For a favored-pair distribution $P_+$ on carrier pairs, an alignment loss decreases with the desired affinity:
\begin{equation}
\mathcal L_{\mathrm{align}}
=\E_{(\mathsf V,\mathsf V')\sim P_+}[\ell_+(\psi_\theta(x_{\mathsf V},x_{\mathsf V'}))].
\label{eq:general-alignment}
\end{equation}
For example, $\ell_+(k)=-\log k$ promotes collision, with infinite loss at zero overlap unless a numerical surrogate is explicitly substituted.
Using normalized collision instead promotes angular agreement without the same concentration factors.
For dissimilarity-valued statistics, the monotonic directions reverse.

\subsection{Pair-specific separation and aggregate organization}

A source can also identify pairs whose collision would erase a required distinction.
For an appropriate disfavored-pair distribution $P_-$ on carrier pairs,
\begin{equation}
\mathcal L_{\mathrm{sep}}^{\mathrm{rel}}
=\E_{(\mathsf V,\mathsf V')\sim P_-}[\ell_-(\psi_\theta(x_{\mathsf V},x_{\mathsf V'}))],
\label{eq:relational-separation}
\end{equation}
where $\ell_-$ increases with undesirable affinity.
This is a pair-specific condition.
By contrast,
\begin{equation}
\mathcal L_{\mathrm{org}}=D_2(q_Z\|r)
\label{eq:marginal-organization}
\end{equation}
compares population occupancy with an external reference.
Two partitions can have identical occupancy and different pairwise equivalences: marginal matching organizes population use, while pair-specific terms specify which source relations should survive.

A combined objective can therefore be written as
\begin{equation}
\mathcal L
=\mathcal L_{\mathrm{align}}
+\lambda_{\mathrm{rel}}\mathcal L_{\mathrm{sep}}^{\mathrm{rel}}
+\lambda_{\mathrm{org}}D_2(q_Z\|r),
\qquad\lambda_{\mathrm{rel}},\lambda_{\mathrm{org}}\ge0.
\label{eq:alignment-relational-marginal}
\end{equation}
Only the terms required by the learning problem need be present.
Uniform prior matching supplies one possible population-level noncollapse mechanism, while relational separation acts on specified pairs.
This clarifies the analogy with alignment and uniformity on a hypersphere~\citep{wang2020alignment}: the finite categorical case has a fixed available frame and an explicitly chosen marginal reference, rather than an unrestricted continuous geometry.

\subsection{Occupancy, fidelity, and resource constraints}
\label{sec:occupancy-frontiers}

Occupancy is a property of realized code use, not by itself the complete description resource.
We therefore study occupancy constraints alongside the relational distortion defined in \cref{sec:relational-distortion} within the chosen finite-codeword family.
Let $m_{S,\theta}$ denote the description induced by parameter setting $\theta$, and define $D_{\mathrm{rel}}(\theta):=\Delta_V\!\left(S,\mathsf{Dec}_V(m_{S,\theta})\right)$ as the relational distortion within this family, following the complete-description formulation in \cref{eq:lossy-relational-compression}.
Then
\begin{equation}
D_{\mathrm{rel}}^\star(K)=\inf_{\theta:|\mathcal Z|\le K}D_{\mathrm{rel}}(\theta).
\end{equation}
The other parts of the description and decoder family must be held fixed or explicitly controlled; otherwise an auxiliary payload could carry arbitrary source information while the code alphabet remains small.

For $R_2(\theta)=H_2(q_Z)$, define
\begin{equation}
D_{\le}^\star(R)=\inf_{\theta:R_2(\theta)\le R}D_{\mathrm{rel}}(\theta),\qquad
D_{\ge}^\star(R)=\inf_{\theta:R_2(\theta)\ge R}D_{\mathrm{rel}}(\theta).
\label{eq:relational-complexity-distortion-frontier}
\end{equation}
An upper occupancy bound and a lower occupancy requirement answer different questions.
The first limits realized code use; the second enforces noncollapse or organization.
An inverse requirement can also be written $R_2^\star(D)=\inf_{\theta:D_{\mathrm{rel}}(\theta)\le D}R_2(\theta)$.
These definitions are occupancy--fidelity relations; an operational coding model can additionally relate them to code length.

For direct positive-relation distortion, the one-codeword solution can have zero distortion, making the upper-bound problem trivial.
A lower bound on $R_2$ instead asks how much relational fidelity can be retained while using a noncollapsed alphabet.
At fixed $K$, an objective
\begin{equation}
D_{\mathrm{rel}}(\theta)+\lambda_{\mathrm{org}}D_2(q_Z\|U_K)
=D_{\mathrm{rel}}(\theta)-\lambda_{\mathrm{org}}R_2(\theta)
+\lambda_{\mathrm{org}}\log K
\label{eq:organization-lagrangian}
\end{equation}
encourages this second direction.
Because this objective encourages occupancy rather than imposing an upper resource budget, its empirical distortion trajectory can increase with $R_2$.
Nonconvex optimization and discrete feasible sets also mean that scalar penalties need not recover every constrained frontier point.

\section{Selected Realizations of Relational Fidelity}
\label{sec:relational-distortion-realizations}

Relational requirements may be supplied directly by a source relation or emerge from reconstruction.
This section instantiates both routes within the finite-codeword geometry, covering signed and baseline-relative pair requirements, graph-operator probes, and exact correspondences to centroid reconstruction and normalized cut.
These examples are representative rather than exhaustive; path, walk, and additional spectral constructions are retained in \ref{app:path-walk-content}.

\subsection{Selected constructions for relational requirements}

When pairwise weights are derived from a single real-valued relation, two useful cases are an intrinsically signed relation and a relation interpreted relative to a source-side baseline.

This role assignment is specific to the collision realization: for example, the block-valued decoder of \cref{sec:block-summary-walkthrough} can reproduce strong relations between different groups through its reproduction table $\Xi$, without interpreting group inequality itself as a separation requirement.

\paragraph{Intrinsically signed relations}
For an intrinsically signed $S$, the nonnegative source weights can be $S_+=[S]_+$ and $S_-=[-S]_+$.

\paragraph{Baseline-relative relations}
Alternatively, compare a real-valued relation with a source-side neutral baseline:
\begin{equation}
S_+=[S-S_{\mathrm{bg}}]_+,\qquad
S_-=[S_{\mathrm{bg}}-S]_+.
\label{eq:signed-relation}
\end{equation}
Relations above the baseline induce alignment and those below it induce separation, including when all original source values are nonnegative.
When a probability-law form is desired, nonzero pair weights can be normalized over the evaluated carrier pairs to obtain $P_+$ and $P_-$; weighted-sum objectives may instead use them directly.

The source-side baseline $S_{\mathrm{bg}}$ identifies unusually strong or weak source relationships, while $r(z)$ specifies a reference for codeword occupancy.
A representation-side collision reference can separately center a learning margin.
These are respectively a source-side relational baseline, a codeword-occupancy reference, and a representation-side collision threshold.

\subsection{Direct relation preservation and signed distinctions}
\label{sec:edge-cut-distortion}

The simplest graph realization treats retained edge weight itself as relational content and then uses collision to decide which source edges survive.

Let $G=(V,E,W)$ be an undirected graph with symmetric nonnegative weights, zero diagonal, and one copy of each undirected edge in $E$, and let $\widetilde G$ denote a reconstructed weighted graph on the same carrier with weight matrix $\widetilde W$.
A reconstruction obtained by deletion or attenuation satisfies $0\le\widetilde W\le W$ entrywise.
The direct edge content and its distortion are
\begin{equation}
\mathcal I_{\mathrm{edge}}(G)=\frac12\sum_{i,j}W_{ij},\qquad
D_{\mathrm{edge}}(G,\widetilde G)=\frac12\sum_{i,j}(W_{ij}-\widetilde W_{ij}).
\label{eq:edge-distortion}
\end{equation}
When $\mathcal I_{\mathrm{edge}}(G)>0$, the normalized direct-edge distortion is
\begin{equation}
D_E(G,\widetilde G)
:=\frac{D_{\mathrm{edge}}(G,\widetilde G)}{\mathcal I_{\mathrm{edge}}(G)}.
\label{eq:normalized-edge-distortion}
\end{equation}
The factor $1/2$ counts each undirected edge once.
This is a monotone content-based fidelity in the sense of \cref{eq:relational-content-distortion}.

A hard partition can retain only within-codeword edges,
\begin{equation}
\widetilde W_{ij}=W_{ij}\ind[c(i)=c(j)],
\label{eq:partition-retained-relation}
\end{equation}
so its distortion is the cut weight $\sum_{\{i,j\}\in E}W_{ij}\ind[c(i)\ne c(j)]$.
For probabilistic assignments, define $q_i:=q_\theta(\cdot\mid x_i)$ and $k_{\theta,ij}:=k_\theta(x_i,x_j)=q_i^\top q_j$.
Collision attenuation then gives
\begin{equation}
\widetilde W_{\theta,ij}=W_{ij}k_{\theta,ij},\qquad
D_{\mathrm{edge}}(\theta)=\frac12\sum_{i,j}W_{ij}(1-k_{\theta,ij}).
\label{eq:soft-cut-distortion}
\end{equation}
Its normalized form is
\[
D_E(\theta)
=\frac{
\sum_{\{i,j\}\in E}W_{ij}(1-k_{\theta,ij})
}{
\sum_{\{i,j\}\in E}W_{ij}
}.
\]
When used as a learning objective, minimizing this distortion is equivalent to maximizing collision on positively weighted source edges.
For independent sampled assignments, the soft expression also equals the expected hard cut weight on distinct vertices because this distortion is linear in the retained edge indicators; this equivalence does not generally extend to nonlinear graph fidelities.

The decoder convention matters here.
A reconstruction of the retained weighted graph includes its weights as stored or shared information.
An equality-only description can instead be evaluated by its cut distortion against the source graph.
These distinct descriptions can lead to the same scalar distortion.

Positive edge preservation alone allows complete collapse.
For nonnegative source weights $W_+$ and $W_-$ describing favored and disfavored collisions, a signed relational objective is
\begin{equation}
\mathcal L_{\mathrm{signed}}(\theta)
=\frac12\sum_{i,j}\left[W_{+,ij}(1-k_{\theta,ij})
+\lambda_{\mathrm{rel}}W_{-,ij}k_{\theta,ij}\right].
\label{eq:signed-relational-objective}
\end{equation}
The weights may be supplied directly or derived from $W-W_{\mathrm{bg}}$ as in \cref{eq:signed-relation}.
Although collision attenuation still gives $0\le\widetilde W\le W$ for a nonnegative source graph, the signed objective is not a monotone retained-content distortion: decreasing collision worsens fidelity on favored pairs but improves it on disfavored pairs.
With a hard equality decoder, it defines a relational fidelity criterion on the partition relation, closely related to the positive/negative disagreement perspective of correlation clustering~\citep{bansal2004correlation}.

\subsection{Graph probes and operator fidelity}
\label{sec:graph-fourier-distortion}

Direct edge fidelity asks which relations remain; operator fidelity instead asks which behaviors supported by those relations remain.

For the same undirected graph, let $d_i=\sum_jW_{ij}$ and $L=\operatorname{Diag}(d)-W$.
Its Dirichlet energy is
\begin{equation}
f^\top Lf=\frac12\sum_{i,j}W_{ij}(f_i-f_j)^2.
\label{eq:dirichlet-energy-vertex}
\end{equation}
With $L=U\Lambda U^\top$ and graph-Fourier coefficients $\widehat f=U^\top f$, the same scalar is
\begin{equation}
f^\top Lf=\sum_r\lambda_r|\widehat f_r|^2.
\label{eq:dirichlet-energy-fourier}
\end{equation}
This is the standard graph-signal/Fourier interpretation~\citep{shuman2013emerging}.
The relation determines which signal variations carry energy.

Deleting or attenuating edges gives $0\preceq\widetilde L\preceq L$, since $L-\widetilde L$ is the Laplacian of the removed weights.
A fixed probe therefore loses
\begin{equation}
D_f(G,\widetilde G)
=f^\top(L-\widetilde L)f
=\frac12\sum_{i,j}(W_{ij}-\widetilde W_{ij})(f_i-f_j)^2.
\label{eq:dirichlet-distortion-vertex}
\end{equation}
If $\widetilde L=\widetilde U\widetilde\Lambda\widetilde U^\top$, this is equally
\begin{equation}
D_f(G,\widetilde G)
=\sum_r\lambda_r|(U^\top f)_r|^2
-\sum_r\widetilde\lambda_r|(\widetilde U^\top f)_r|^2.
\label{eq:dirichlet-distortion-fourier}
\end{equation}
The vertex and Fourier expressions evaluate the same fidelity in two coordinate systems, using each operator's own eigenbasis.

Rather than selecting a single probe signal, we can specify a distribution over signals whose behavior we wish to preserve.
Let $F$ be such a random probe with second moment $M=\E[FF^\top]$.
Specializing \cref{eq:operator-expected-distortion} to the graph Laplacian gives
\[
D_M(G,\widetilde G)=\operatorname{tr}[(L-\widetilde L)M].
\]
For an edge $e=\{i,j\}$, let $b_e=e_i-e_j$ be an oriented incidence vector, with either orientation.
The Laplacian decomposes into individual edge contributions:
\[
L
=
\sum_{e=\{i,j\}\in E}W_{ij}b_eb_e^\top.
\]
The orientation is arbitrary because replacing $b_e$ by $-b_e$ leaves $b_eb_e^\top$ unchanged.
Using the rank-one trace identity $\operatorname{tr}(uv^\top A)=v^\top Au$ gives
\begin{align*}
\operatorname{tr}(LM)
&=
\sum_{e=\{i,j\}\in E}
W_{ij}\operatorname{tr}(b_eb_e^\top M)\\
&=
\sum_{e=\{i,j\}\in E}
W_{ij}b_e^\top M b_e.
\end{align*}
Because $M=\E[FF^\top]$,
\[
b_e^\top M b_e
=
\E[(F_i-F_j)^2].
\]
Define the source importance of edge $e$ by $s_e=W_{ij}b_e^\top M b_e\ge0$.
It is the contribution of that edge to the expected source Dirichlet energy under the chosen probe family.
Since $\sum_e s_e=\operatorname{tr}(LM)$, collision attenuation gives the normalized distortion
\begin{equation}
\overline D_M(\theta)
=\sum_{e=\{i,j\}\in E}\rho_e(1-k_{\theta,ij}),\qquad
\rho_e=\frac{s_e}{\sum_{e'\in E}s_{e'}},
\label{eq:probe-edge-family}
\end{equation}
provided the total expected source energy $\operatorname{tr}(LM)$ is positive.
Thus $M$ determines which source signal variations, and therefore which edges, matter most to the fidelity criterion, while the finite-code representation itself remains unchanged.

\paragraph{Isotropic and all-mode probes}
For $M=I$, $D_I=\operatorname{tr}(L-\widetilde L)=2D_{\mathrm{edge}}$.
Since $\operatorname{tr}(L)=2\mathcal I_{\mathrm{edge}}(G)$, the normalized isotropic-probe criterion is $D_I/\operatorname{tr}(L)=D_E$.
For a source graph with $n$ vertices and $\kappa_G$ connected components, suppose $\operatorname{rank}(L)=n-\kappa_G>0$.
A source-normalized alternative then uses $M=L^+/(n-\kappa_G)$, where $L^+$ is the Moore--Penrose pseudoinverse.
We denote this source-normalized all-mode criterion by $D_F$.
The spectral expansion
\[
L^+
=
\sum_{\lambda_r>0}\frac1{\lambda_r}u_ru_r^\top
\]
gives
\begin{align}
D_F(G,\widetilde G)
&=\frac{\operatorname{tr}[(L-\widetilde L)L^+]}{n-\kappa_G}\\
&=\frac1{n-\kappa_G}
\sum_{\lambda_r>0}
\frac{u_r^\top(L-\widetilde L)u_r}{\lambda_r}.
\label{eq:source-normalized-dirichlet-distortion}
\end{align}
The factor $1/\lambda_r$ normalizes each positive-eigenvalue source Laplacian mode by its source Dirichlet energy, while division by $n-\kappa_G$ averages over those modes.
Zero modes are excluded because they have zero source Dirichlet energy; for a disconnected graph, they include signals that are constant separately on each connected component.
For an individual positive-eigenvalue source mode $u_r$, its source Dirichlet energy is $u_r^\top Lu_r=\lambda_r$.
Evaluating the same fixed source signal on the reconstructed graph gives the Dirichlet energy $u_r^\top\widetilde Lu_r$, so
\[
\frac{u_r^\top(L-\widetilde L)u_r}{\lambda_r}
=1-\frac{u_r^\top\widetilde Lu_r}{\lambda_r}
\]
is the fraction of that mode's source Dirichlet energy removed by compression, and $D_F$ is the mean of these relative losses.
The quantity $u_r^\top\widetilde Lu_r$ is not generally an eigenvalue of $\widetilde L$, because a source eigenvector need not remain an eigenvector after compression.
Within a repeated positive-eigenvalue source eigenspace, the individual mode losses depend on the chosen orthonormal eigenbasis, but their sum over the entire eigenspace and hence its contribution to $D_F$ is basis invariant.
Because $\operatorname{tr}(LL^+)=\operatorname{rank}(L)=n-\kappa_G$, the normalization $M=L^+/(n-\kappa_G)$ ensures $\operatorname{tr}(LM)=1$, so the probe distribution has unit expected Dirichlet energy on the source graph.

The spectral form describes distortion mode by mode.
To obtain the complementary edgewise interpretation, decompose the removed Laplacian into its deleted or attenuated edge contributions:
\[
L-\widetilde L
=
\sum_{e=\{i,j\}\in E}(W_{ij}-\widetilde W_{ij})b_e b_e^\top.
\]
Substituting this decomposition into the trace form gives
\begin{equation}
D_F(G,\widetilde G)
=\frac1{n-\kappa_G}\sum_{e=\{i,j\}\in E}(W_{ij}-\widetilde W_{ij})R_{\mathrm{eff}}^G(i,j),
\qquad R_{\mathrm{eff}}^G(i,j)=b_e^\top L^+b_e,\quad e=\{i,j\}.
\label{eq:resistance-weighted-dirichlet-distortion}
\end{equation}
All resistances are computed in the original source graph.
For a hard partition,
\begin{equation}
D_F(G,c)=\frac1{n-\kappa_G}
\sum_{\substack{\{i,j\}\in E\\c(i)\ne c(j)}}
W_{ij}R_{\mathrm{eff}}^G(i,j).
\label{eq:hard-partition-dirichlet-distortion}
\end{equation}
Its numerator is the established resistance-weighted cut-size~\citep{cesabianchi2013random,gentile2013online}.
Here it is used as the all-mode source-normalized probe realization of relational fidelity.
Since
\[
\sum_{\{i,j\}\in E}W_{ij}R_{\mathrm{eff}}^G(i,j)
=
\operatorname{tr}(LL^+)
=
n-\kappa_G,
\]
the denominator is the total source resistance-weighted edge mass, giving $0\le D_F\le1$ under attenuation.

\paragraph{Source collision probes}
Beyond isotropic and all-mode choices, a source-derived probe family can emphasize distinctions already present in the graph's local transition structure.
Viewing each vertex's normalized edge weights as a distribution over destinations, two vertices have greater probability-product overlap when independent draws from their destination distributions tend to coincide.
For positive degrees, let $P_{ij}=W_{ij}/d_i$ be the source transition matrix.
Its rows are distributions over relational destinations, and
\begin{equation}
\Gamma_{\mathrm{src}}=PP^\top,\qquad
(\Gamma_{\mathrm{src}})_{ij}=\sum_rP_{ir}P_{jr}
\end{equation}
is the probability that independent destination draws from rows $i$ and $j$ coincide.
It is positive semidefinite and can be used as a source-frozen probe second moment.
For $\operatorname{tr}(L\Gamma_{\mathrm{src}})>0$,
\begin{equation}
D_C(G,\widetilde G)
=\frac{\operatorname{tr}[(L-\widetilde L)\Gamma_{\mathrm{src}}]}
{\operatorname{tr}(L\Gamma_{\mathrm{src}})}.
\label{eq:source-collision-covariance-distortion}
\end{equation}
For $e=\{i,j\}$, $b_e^\top PP^\top b_e=\|P_i-P_j\|_2^2$, so collision attenuation gives
\begin{equation}
D_C(\theta)=
\frac{\sum_{\{i,j\}\in E}W_{ij}\|P_i-P_j\|_2^2(1-k_{\theta,ij})}
{\sum_{\{i,j\}\in E}W_{ij}\|P_i-P_j\|_2^2}.
\label{eq:source-collision-edge-distortion}
\end{equation}
This fidelity emphasizes differences in source transition profiles across edges.
Here $\Gamma_{\mathrm{src}}=PP^\top$ is a source-derived probe second moment, encoding collision probabilities between transition distributions, whereas $k_{\theta,ij}=q_i^\top q_j$ is a representation-side collision probability that controls edge attenuation.
The two have the same probability-product form but act on different distributions and play different roles.

The full transition-profile construction compares the destination distributions themselves.
A complementary source-derived probe can instead summarize each vertex's local transition structure by a scalar measure of its effective relational diversity, and preserve how that quantity varies across the graph.
Specifically, a rank-one source probe uses local transition collision entropy,
\begin{equation}
h_i^{(2)}=-\log\sum_jP_{ij}^2.
\end{equation}
The source field $h_G^{(2)}$ is held fixed while the graph is compressed.
When its source Dirichlet energy is positive, define
\begin{equation}
D_{H_2}(G,\widetilde G)
=\frac{(h_G^{(2)})^\top(L-\widetilde L)h_G^{(2)}}
{(h_G^{(2)})^\top Lh_G^{(2)}}.
\label{eq:entropy-probe-distortion}
\end{equation}
Here $D_{H_2}$ measures distortion of the source transition-entropy field $h_G^{(2)}$ under graph compression.
Under collision attenuation this has edge weights proportional to $W_{ij}(h_i^{(2)}-h_j^{(2)})^2$.
It measures the lost gradient energy of the selected source-intrinsic field.
If the field is constant on connected components, the denominator is zero and this normalized fidelity is undefined.

For an unweighted simple graph, each transition row is uniform over $d_i$ neighbors, so $h_i^{(2)}=\log d_i$; its Shannon row entropy is identical.
Thus on such graphs the transition-entropy field reduces to the log-degree field, while weighted or otherwise nonuniform relations allow Shannon and R\'enyi-2 transition entropies to distinguish different local structure.

\paragraph{Worst-case fidelity}
The relative worst-case graph distortion is
\begin{equation}
D_{\mathrm{wc}}(G,\widetilde G)
=\sup_{f:f^\top Lf>0}
\frac{f^\top(L-\widetilde L)f}{f^\top Lf}.
\label{eq:graph-worst-case-spectral-distortion}
\end{equation}
If a connected source becomes disconnected, a signal constant on each new component but nonconstant on the original graph has zero reconstructed energy and positive source energy.
Hence $D_{\mathrm{wc}}=1$.
In particular, every nontrivial hard partition of a connected source that removes all cross-class edges has maximal distortion under this criterion.
That makes worst-case fidelity a useful boundary of the representation, but an uninformative ranking among such hard partition reconstructions.

\subsection{Centroid reconstruction as an exact objective correspondence}
\label{sec:vq}

Classical squared-error centroid reconstruction admits an exact pairwise relational formulation.
Consider $N$ observations $x_1,\ldots,x_N\in\R^d$, and define the source relation
\[
S(i,j)=\|x_i-x_j\|_2^2.
\]
For the probabilistic finite encoder of \cref{eq:discrete-encoder}, write
\[
q_{iz}:=q_\theta(z\mid x_i),
\qquad
q_{iz}\ge0,
\qquad
\sum_zq_{iz}=1,
\]
for the soft assignment weight of observation $x_i$ to codeword $z$.
Define the empirical aggregate codeword distribution by
\[
\hat q_Z(z)
:=
\frac1N\sum_{i=1}^Nq_{iz}.
\]
This is the finite-sample analogue of the population aggregate codeword distribution $q_Z$ in \cref{eq:aggregate-code}.
A codeword is active when $\hat q_Z(z)>0$.
For each active codeword, the assignment-weighted centroid is
\begin{equation}
\bar x_z
=
\frac{1}{N\hat q_Z(z)}
\sum_{i=1}^Nq_{iz}x_i,
\qquad
\hat q_Z(z)>0.
\label{eq:finite-soft-centroid}
\end{equation}
Define the average centroid reconstruction error by
\[
D_{\mathrm{cent}}
:=
\frac1N
\sum_{z:\hat q_Z(z)>0}
\sum_iq_{iz}\|x_i-\bar x_z\|_2^2.
\]
The decoder associates one reproduction vector $\widehat x_z\in\R^d$ with each codeword $z\in\mathcal Z$.
For arbitrary choices of these reproduction vectors, define the average reconstruction error by
\[
D_{\mathrm{dec}}
:=
\frac1N
\sum_{i,z}q_{iz}\|x_i-\widehat x_z\|_2^2.
\]
The key correspondence is that centroid reconstruction error can be written exactly as an inverse-marginally weighted pairwise squared-distance distortion under the same assignments.
More generally, reconstruction with arbitrary codeword reproductions decomposes into this centroid error plus the excess incurred when the reproduction vectors differ from their assignment-weighted centroids.
The following proposition collects the established weighted variance identities that yield these correspondences~\citep{gray1998quantization,dhillon2004kernel}.

\begin{proposition}[Centroid and pairwise forms]
\label{prop:centroid-pairwise}
For the assignments and reconstruction errors above,
\begin{equation}
D_{\mathrm{cent}}
=\frac1{2N^2}
\sum_{z:\hat q_Z(z)>0}\frac1{\hat q_Z(z)}
\sum_{i,j}q_{iz}q_{jz}\|x_i-x_j\|_2^2.
\label{eq:finite-soft-pairwise}
\end{equation}
The affinity $a_{\hat q_Z}$ is the empirical finite-sample specialization of the inverse-mass affinity in \cref{eq:qz-affinity}, so equivalently
\begin{equation}
D_{\mathrm{cent}}=\frac1{2N^2}\sum_{i,j}
\|x_i-x_j\|_2^2a_{\hat q_Z}(x_i,x_j).
\label{eq:finite-soft-qz-normalized}
\end{equation}
The arbitrary-reproduction error decomposes as
\begin{equation}
D_{\mathrm{dec}}
=D_{\mathrm{cent}}
+\sum_{z:\hat q_Z(z)>0}
\hat q_Z(z)\|\bar x_z-\widehat x_z\|_2^2.
\label{eq:decoder-centroid-decomposition}
\end{equation}
\end{proposition}

\begin{proof}
For each active codeword,
\begin{equation}
\sum_iq_{iz}\|x_i-\bar x_z\|^2
=\sum_iq_{iz}\|x_i\|^2
-N\hat q_Z(z)\|\bar x_z\|^2.
\end{equation}
Expanding the weighted double sum of $\|x_i-x_j\|^2$ yields twice $N\hat q_Z(z)$ times the same within-codeword variance expression.
Summing over codewords proves the first identity, and division by the empirical aggregate mass gives the inverse-marginal affinity form.
For the decoder identity, expand $x_i-\widehat x_z=(x_i-\bar x_z)+(\bar x_z-\widehat x_z)$; the weighted cross term vanishes, and $\sum_iq_{iz}=N\hat q_Z(z)$ gives the stated excess-error term.
\end{proof}

For a hard assignment $c$, let $V_z:=\{i:c(i)=z\}$.
For each nonempty class, the familiar special case is
\begin{equation}
\sum_{i\in V_z}\|x_i-\bar x_z\|^2
=\frac1{2|V_z|}\sum_{i,j\in V_z}\|x_i-x_j\|^2,
\label{eq:vq-pairwise}
\end{equation}
The variance identity therefore determines the inverse empirical-marginal factor.
Where the assignments are differentiable in the encoder parameters and active aggregate masses remain positive, the two forms also have identical encoder gradients.

The population form requires a finite second moment.
Let $X\sim\mu$ with $\E\|X\|^2<\infty$, let $Z\mid X\sim q_\theta(\cdot\mid X)$, and define $\bar x_z=\E[X\mid Z=z]$ for active words.
Then
\begin{equation}
\E\|X-\bar x_Z\|^2
=\frac12\E_{X,X'\overset{\mathrm{iid}}{\sim}\mu}
\left[\|X-X'\|^2a_{q_Z}(X,X')\right].
\label{eq:stochastic-vq-pairwise}
\end{equation}
To see this, condition on each active $z$, apply the two-independent-copies variance identity, and average with weight $q_Z(z)$.
The conditional law has density $q_\theta(z\mid x)/q_Z(z)$ with respect to $\mu$, producing exactly the inverse-mass factor.

This correspondence concerns a scalar fidelity and its optimal per-codeword reproduction.
Pointwise reconstruction uses the reproduction vectors, while within-code relational distortion can use the partition or affinity as its reconstructed statistic.
The identity relies on this per-codeword decoder structure; a decoder that jointly interprets multiple latent variables need not admit the same decomposition.

\subsection{Normalized association and normalized cut}
\label{sec:normalized-cut-realization}

Normalized cut is a classical graph-partitioning objective that penalizes edges crossing between groups relative to the total edge volume of those groups~\citep{shi2000normalized}.
Equivalently, normalized association rewards edge weight retained within each group after normalization by group volume.
We show here that the endogenous inverse-mass affinity of the finite-codeword construction recovers these objectives under graph-local degree weighting.

For an undirected nonnegative graph with positive total volume, let $d_i=\sum_jW_{ij}$, $\operatorname{vol}(G)=\sum_i d_i$, and $\pi_i=d_i/\operatorname{vol}(G)$.
This defines a degree-weighted distribution over the vertices.
For assignments $q_i(z)$, the aggregate codeword distribution under this degree weighting is
\begin{equation}
\bar q_G(z)=\sum_i\pi_iq_i(z)
=\frac{\sum_i d_iq_i(z)}{\operatorname{vol}(G)}.
\label{eq:graph-aggregate-code}
\end{equation}
Zero-degree vertices receive zero weight under this distribution.
Write $\mathcal Z_G^+=\{z:\bar q_G(z)>0\}$.
The endogenous inverse-mass affinity $a_{q_Z}$ of \cref{eq:qz-affinity}, built from the aggregate codeword distribution $q_Z(z)$ in \cref{eq:aggregate-code}, specializes here by replacing $q_Z(z)$ with the graph-local aggregate $\bar q_G(z)$ induced by the degree-weighted vertex distribution $\pi$.
\begin{equation}
a_{\bar q_G}(i,j)
:=\sum_{z\in\mathcal Z_G^+}\frac{q_i(z)q_j(z)}{\bar q_G(z)}.
\end{equation}

Let $\vec E$ contain both directions of every undirected edge.
The graph-local inverse-mass affinity gives
\begin{align}
\operatorname{NAssoc}_{\mathrm{soft}}(G)
&=\frac1{\operatorname{vol}(G)}\sum_{(i,j)\in\vec E}W_{ij}a_{\bar q_G}(i,j)\\
&=\sum_{z\in\mathcal Z_G^+}
\frac{\sum_{(i,j)\in\vec E}W_{ij}q_i(z)q_j(z)}{\sum_i d_iq_i(z)}.
\label{eq:soft-nassoc-affinity}
\end{align}
The second expression follows by substituting \cref{eq:graph-aggregate-code}, showing that the graph-local inverse-mass affinity recovers a familiar soft normalized-association surrogate.
To obtain the corresponding fixed-$K$ normalized-cut criterion, we must also specify how empty classes are treated.
We define
\begin{align}
\operatorname{Ncut}_{K,\mathrm{soft}}(G)
&=K-\operatorname{NAssoc}_{\mathrm{soft}}(G)\\
&=K-|\mathcal Z_G^+|
+\sum_{z\in\mathcal Z_G^+}
\frac{\sum_{(i,j)\in\vec E}W_{ij}q_i(z)(1-q_j(z))}
{\sum_i d_iq_i(z)}.
\label{eq:soft-ncut-relational-distortion}
\end{align}
The correction $K-|\mathcal Z_G^+|$ is zero when every soft class is active.
For strictly positive soft assignments every class is active; the correction extends the criterion to hard or limiting assignments with empty classes.

\begin{proposition}[Hard normalized-cut specialization]
\label{prop:normalized-cut}
For a hard assignment $c$ with $K$ positive-volume classes $V_z:=\{i:c(i)=z\}$, \cref{eq:soft-nassoc-affinity} becomes
\begin{equation}
\operatorname{NAssoc}(G)=\sum_{z=1}^K
\frac{\operatorname{assoc}(V_z,V_z)}{\operatorname{vol}(V_z)},
\qquad
\operatorname{Ncut}(G)=K-\operatorname{NAssoc}(G).
\label{eq:hard-ncut}
\end{equation}
These are classical normalized association and normalized cut.
Under the fixed-$K$ extension, an empty class contributes zero association and a unit penalty to normalized cut.
\end{proposition}

\begin{proof}
One-hot membership makes the numerator for class $z$ its directed internal edge weight and the denominator its volume.
The identity $\operatorname{cut}(V_z,V\setminus V_z)=\operatorname{vol}(V_z)-\operatorname{assoc}(V_z,V_z)$ gives the cut form.
The empty-class extension follows from the definition $K-\operatorname{NAssoc}$.
\end{proof}

The hard objective and its clustering correspondences are established~\citep{shi2000normalized,dhillon2004kernel}.
The soft surrogate is also used by neural partitioners such as GAP and the method of Gatti et al.~\citep{nazi2019gap,gatti2022deep}.
Here these established objectives are positioned within the common description of source fidelity and codeword affinity.

The soft ratio uses a deterministic soft mass in its denominator, whereas the volume of a randomly sampled hard class is itself random; it therefore differs from the expectation of hard normalized cut under sampled assignments.
The distinction between ratios of expectations and expected ratios motivates complementary probabilistic-cut formulations~\citep{ghriss2026beyond}.

\subsection{Common structure of the exact correspondences}

The centroid and normalized-association correspondences expose a common representation-side structure.
\begin{equation}
\begin{aligned}
D_{\mathrm{cent}}
&=\frac1{2N^2}\sum_{i,j}\|x_i-x_j\|^2a_{\hat q_Z}(x_i,x_j),\\
\operatorname{NAssoc}_{\mathrm{soft}}
&=\frac1{\operatorname{vol}(G)}\sum_{(i,j)\in\vec E}W_{ij}a_{\bar q_G}(i,j).
\end{aligned}
\label{eq:vq-ncut-relational-correspondence}
\end{equation}
Both use an endogenous inverse-aggregate-mass affinity, while the source-side weighting determines what agreement means: squared Euclidean distance penalizes grouping distant observations, whereas graph edge weight rewards grouping strongly linked vertices.
Thus the same inverse-aggregate-mass affinity construction can realize distinct relational fidelities without identifying their source geometries.

\section{Experimental Realizations of the Framework}
\label{sec:experiments}

The preceding theory leaves several components deliberately modular.
The five studies below vary those components while keeping finite-codeword assignments as the common representation family; \cref{tab:experimental-framework-summary} summarizes the roles used in the experiments.
Code for the experimental studies is available at \url{https://github.com/yaniv-shulman/relational-compression}.

\begin{table}[t]
\centering
\scriptsize
\setlength{\tabcolsep}{2.5pt}
\renewcommand{\arraystretch}{1.08}
\setlength{\extrarowheight}{1pt}
\begin{tabularx}{\linewidth}{>{\raggedright\arraybackslash}p{0.13\linewidth}
                            >{\raggedright\arraybackslash}p{0.18\linewidth}
                            >{\raggedright\arraybackslash}p{0.25\linewidth}
                            >{\raggedright\arraybackslash}X}
\toprule
Role & Classical compression & Relational compression & Finite-codeword realization used here \\
\midrule
Source / fidelity-bearing content
& Source object $x$ and the content reconstructed from it
& Relational object $(V,S)$, with $S$ the fidelity-bearing structure conditional on shared $V$
& Explicit point geometry, graph structure, or teacher-defined relations; Study 4 instead uses source-image reconstruction to constrain the finite representation through a joint spatial decoder \\
Encoder-side observations
& Usually the source observation itself
& $X_V=(x_v)_{v\in V}$ when encoder-side observations are present
& Point coordinates, graph or node structural features, or image pixels mapped to finite assignments \\
Retained description
& Retained message or code $m$
& Retained relational description $m_S$
& The source-specific $m_S$ comprises declared finite assignments, codes, or tokens plus any required auxiliary payload; $q_\theta(\cdot\mid x)$ is an optimization parameterization where applicable \\
Reconstruction / evaluation
& Reconstructed source $\widehat x$
& $\widetilde S=\mathsf{Dec}_V(m_S)$, materialized or evaluated on demand
& Same-code equality or collision may supply the evaluated relation directly or feed a richer structural or reconstruction decoder \\
Primitive representation-side relation
& No universal analogue is prescribed
& Depends on the chosen description mechanism
& Hard equality $\ind[c_\theta(x)=c_\theta(x')]$ and soft collision $k_\theta(x,x')$ \\
Fidelity
& Distortion $d(x,\widehat x)$
& Relational distortion $\Delta_V(S,\widetilde S)$
& Study-specific source fidelity acting on or through collision, including centroid/pairwise, graph, normalized-cut, reconstruction-defined, and signed-teacher criteria \\
Resource / representation constraint
& Encoded bits, rate, alphabet size, or another declared resource
& Description cost $\Omega(m_S)$
& Alphabet bound $K$, binary width, or nominal latent capacity constrain the declared description convention but need not equal the complete resource $\Omega(m_S)$; any additional retained payload must also be counted \\
Realized organization
& Realized code distribution or entropy under the coding convention
& Separate from description resource unless explicitly incorporated into $\Omega$
& $H_2(q_Z)$, $K_{\mathrm{eff}}$, and $D_2(q_Z\|r)$ when used \\
Shared context
& Codebook or model information available to encoder and decoder
& Shared carrier $V$ and context $C$, including reusable model parameters
& Shared neural encoder or decoder parameters and fixed conventions not included in each source-specific retained description \\
\bottomrule
\end{tabularx}
\caption{Summary of the compression roles used to interpret the experimental realizations. The final column specializes the general relational-compression interface to the finite-codeword constructions used in this paper. Realized codeword organization is reported separately from operational coding rate unless the declared resource convention identifies them.}
\label{tab:experimental-framework-summary}
\end{table}

The five studies provide a compact progression through these roles.
Study 1 provides a controlled numerical realization of the exact centroid--relation correspondence and decoder decomposition; Study 2 holds the finite representation mechanism fixed while exchanging graph-relational fidelity criteria; and Study 3 carries the normalized-cut realization into a shared inductive encoder evaluated on unseen graphs.
Studies 4 and 5 then contrast reconstruction-defined requirements on a finite image representation with explicit teacher-defined relational requirements.\footnote{We write $\lambda_{\mathrm{org}}$ for marginal organization. In Studies 2--4 it corresponds to \texttt{lambda\_sep} (``separation weight'') in the implementation, with no change to settings or objectives; Study 5's pair-specific separation is distinct.}

\subsection{Study 1: A controlled reconstruction--relational identity}
\label{sec:synthetic-experiment}

The first study provides a controlled numerical realization of \cref{prop:centroid-pairwise}, checking that the implemented centroid and inverse-mass-weighted pairwise forms agree as scalar objectives and in their encoder gradients.
It also isolates what fails when the inverse-mass factor is omitted.

The source consists of 768 points in $\R^2$ from an imbalanced six-component Gaussian mixture.
A finite encoder has $K=8$ possible codewords, whose probabilities are evaluated exactly rather than sampled.
The mixture imbalance encourages nonuniform empirical codeword occupancy, making the inverse-mass normalization consequential.
Full settings and diagnostics are in \ref{app:synthetic-details}.

With $\hat q_Z(z):=\frac{1}{N}\sum_iq_{iz}$, the implemented pairwise objective is the finite-sample inverse-mass identity already given in \cref{eq:finite-soft-pairwise}:
\[
D_{\mathrm{pair}}
:=
\frac{1}{2N^2}
\sum_{z:\hat q_Z(z)>0}
\frac{1}{\hat q_Z(z)}
\sum_{i,j}
q_{iz}q_{jz}\|x_i-x_j\|_2^2.
\]
Thus $D_{\mathrm{pair}}=D_{\mathrm{cent}}$.
The deliberately uncorrected comparison is
\[
D_{\mathrm{raw}}
:=
\frac{1}{2N^2}
\sum_{i,j}
\|x_i-x_j\|_2^2
\sum_zq_{iz}q_{jz}.
\]
Uniform empirical occupancy would give $D_{\mathrm{cent}}=KD_{\mathrm{raw}}$.
The comparison therefore includes this global factor $K=8$; nonuniform occupancy instead requires the code-dependent inverse masses.

Across the optimization and diagnostic checks, the corrected pairwise and centroid objectives agree to approximately single-precision tolerance, with maximum discrepancies below $2\times10^{-6}$ and encoder-gradient cosines of 1.000.
By contrast, $KD_{\mathrm{raw}}$ has mean absolute error 4.79 in the random-state check.
A global scale does not replace the inverse aggregate-mass correction.

With the encoder assignments fixed, learned codeword reproduction vectors numerically realize \cref{eq:decoder-centroid-decomposition}.
The excess term $\sum_{z:\hat q_Z(z)>0}\hat q_Z(z)\|\bar x_z-\widehat x_z\|_2^2$ decreases from 21.05 to $8.15\times10^{-5}$, while the largest decomposition residual is $1.91\times10^{-6}$.
The result separates distortion due to the assignments from distortion due to an imperfect reproduction codebook.
The identity applies to per-codeword centroid reproduction; joint neural decoding of many latent tokens has a different structure.
\Cref{fig:synthetic-collision-centroid} summarizes the exact objective correspondence, the failure of the uncorrected collision control, and the decoder decomposition.

\begin{figure}[t]
\centering
\includegraphics[width=\linewidth]{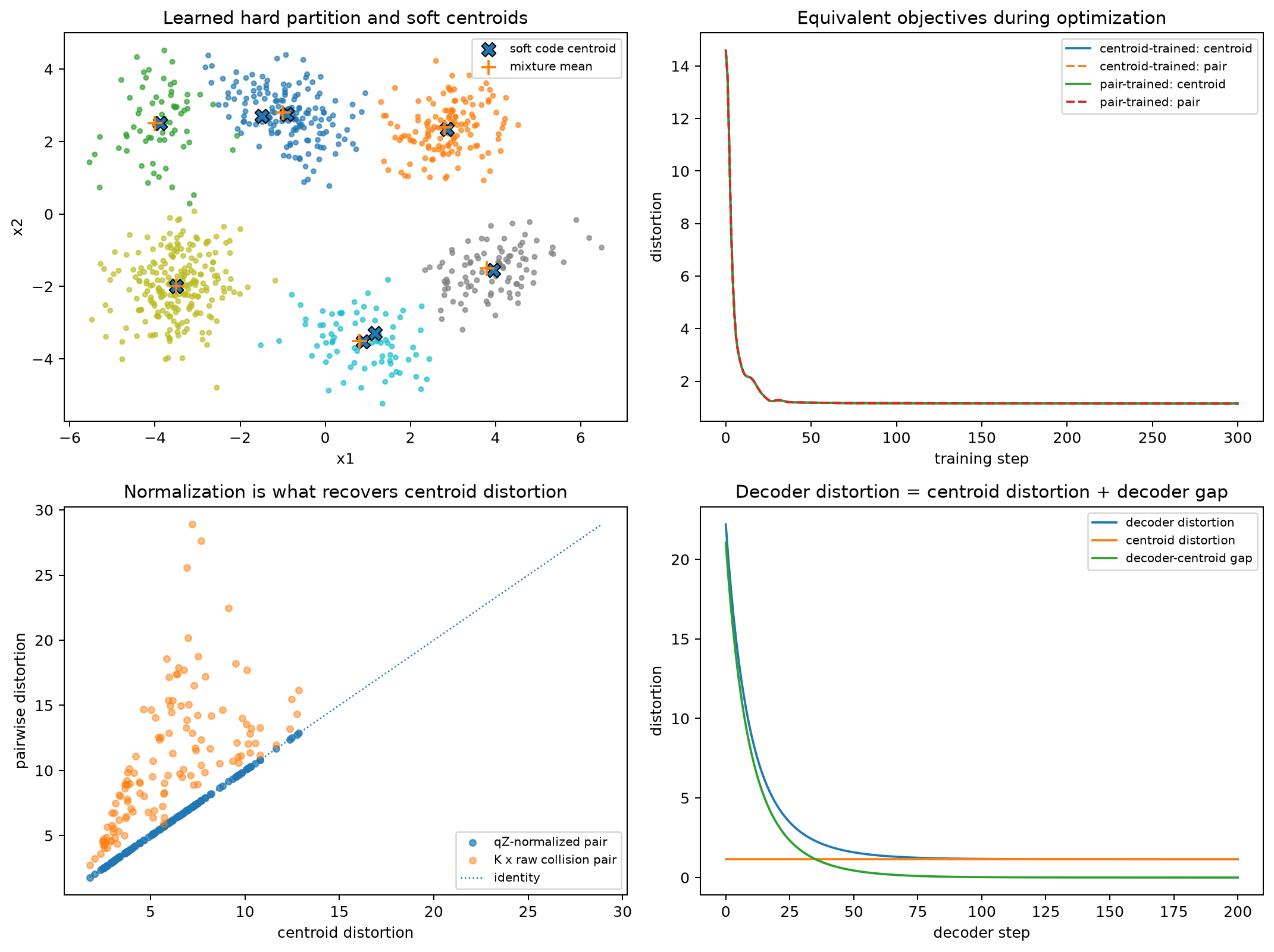}
\caption{Synthetic numerical realization of the centroid--pairwise correspondence. Top left: learned hard partition and soft centroids. Top right: the two exact objectives follow the same trajectory from identical initialization. Bottom left: corrected pairwise distortion lies on the centroid identity line, whereas $KD_{\mathrm{raw}}$ does not. Bottom right: decoder distortion decomposes into centroid distortion plus reproduction mismatch.}
\label{fig:synthetic-collision-centroid}
\end{figure}

The experiment verifies the inverse-mass centroid correspondence and separately exposes decoder reproduction mismatch, numerically illustrating why assignments and decoder reproductions are distinct parts of the description.

\subsection{Study 2: Exchanging relational fidelity on a fixed representation}
\label{sec:transductive-relational-distortion-experiment}

The second study demonstrates the framework's separation between the finite representation and the relational fidelity used to shape it.
Holding the graph-local $K=8$ representation and marginal-organization mechanism fixed, we exchange four source-derived graph fidelities and compare the resulting occupancy--fidelity behavior and partitions.
Each graph has a free $n\times K$ logit matrix with $K=8$, rather than a learned graph encoder.
There are no node features or labels, and no generalization problem.

For $q_i=\operatorname{softmax}(\ell_i/\tau)$ with $\tau=1$, all four source distortion criteria have the form
\begin{equation}
D_\rho(q)=\sum_{e=\{i,j\}\in E}\rho_e(1-q_i^\top q_j),\qquad
\rho_e\ge0,\quad\sum_{e\in E}\rho_e=1.
\label{eq:transductive-drho}
\end{equation}
The choices are normalized direct-edge distortion $D_E$ from \cref{eq:normalized-edge-distortion}, all-mode effective-resistance distortion $D_F$ from \cref{eq:source-normalized-dirichlet-distortion}, source-collision probe distortion $D_C$ from \cref{eq:source-collision-edge-distortion}, and entropy-field distortion $D_{H_2}$ from \cref{eq:entropy-probe-distortion} when its denominator is positive.
Here $D_{H_2}$ denotes distortion of the source transition-entropy field and is distinct from the code-occupancy statistic $H_2(\bar q_G)$ below.
The graph-local marginal is degree weighted, and each run minimizes
\begin{equation}
\mathcal L=D_\rho(q)+\lambda_{\mathrm{org}}D_2(\bar q_G\|U_K)
=D_\rho(q)-\lambda_{\mathrm{org}}H_2(\bar q_G)+\lambda_{\mathrm{org}}\log K.
\label{eq:transductive-relational-objective}
\end{equation}
Here $\bar q_G$ is the degree-weighted graph-local aggregate defined in \cref{eq:graph-aggregate-code}, and the second equality uses the uniform-reference R\'enyi identity in \cref{eq:uniform-renyi}.
Thus the sweep varies a noncollapse requirement while evaluating relational fidelity.

The source collections are 20 cleaned connected MalNet-Tiny test graphs, 20 cleaned COLLAB graphs, and 20 cleaned PROTEINS graphs~\citep{freitas2021large,morris2020tudataset,yanardag2015deep,borgwardt2005protein}.
To separate sensitivity to source edge weights from differences in topology, we also evaluate a weighted control on the same MalNet topologies, assigning deterministic positive lognormal edge weights.
Labels are ignored.
Because the graph-local objectives are optimized nonconvexly, every graph, fidelity, and organization weight uses the same multi-restart selection protocol, with selection by the target soft objective.
Full restart, refinement, graph-filtering, and sweep details are given in \ref{app:transductive-relational-distortion-details}.

\paragraph{Occupancy--fidelity trajectories}
At zero organization weight, every defined criterion selects the one-codeword solution, with zero positive-relation distortion and $K_{\mathrm{eff}}=1$.
Increasing organization moves the hard solutions toward broader codeword use with nonzero relational distortion.
On unweighted MalNet, $\lambda_{\mathrm{org}}=0.20$ yields hard $K_{\mathrm{eff}}\approx7.8$ across criteria, with own hard distortions 0.238 for $D_E$, 0.153 for $D_F$, 0.127 for $D_C$, and 0.096 for $D_{H_2}$.
These values quantify different source properties on their respective fidelity scales.

\begin{figure}[t]
\centering
\includegraphics[width=\linewidth]{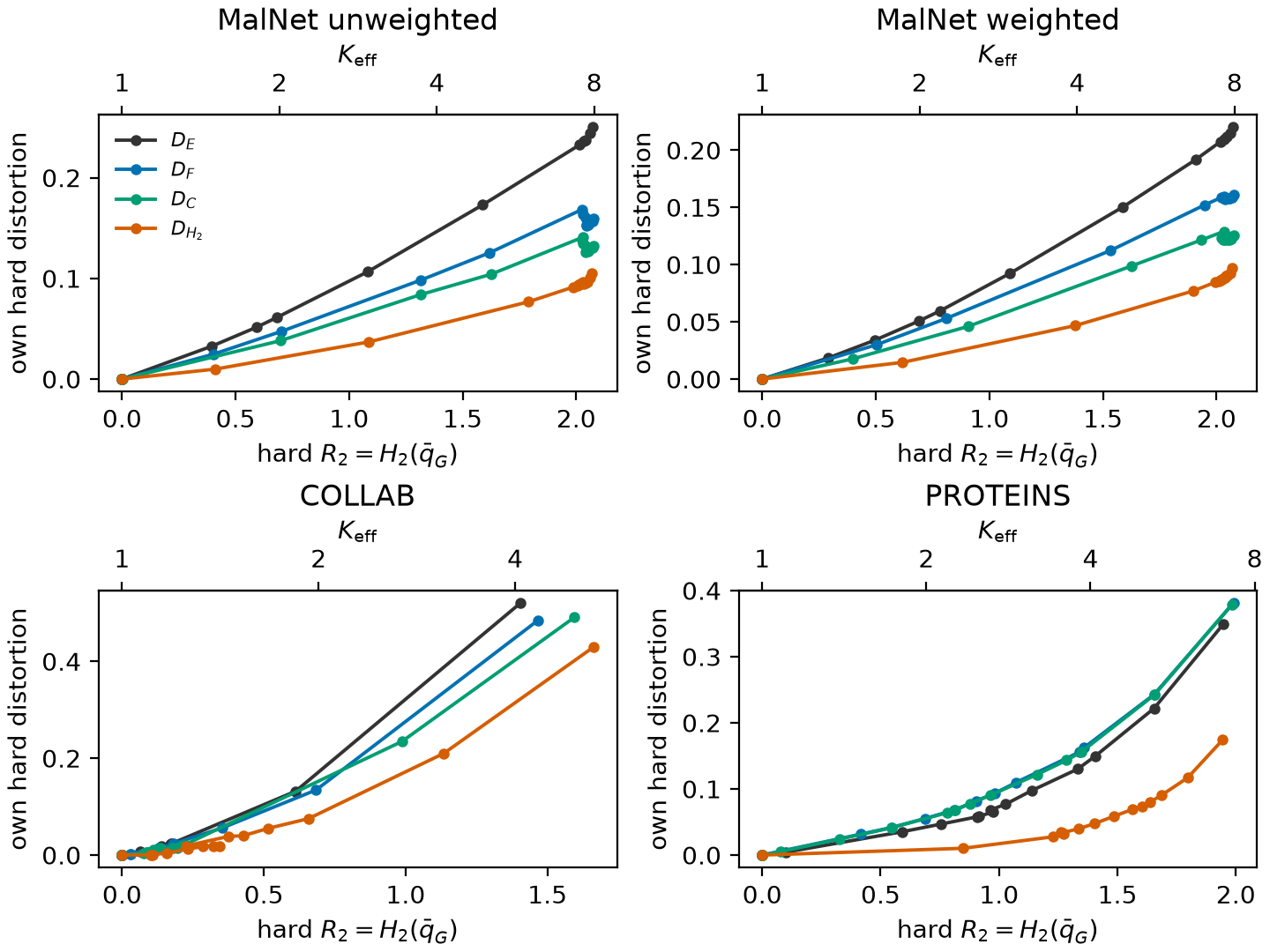}
\caption{Transductive collision-occupancy--fidelity trajectories. The horizontal coordinate is hard $H_2(\bar q_G)$ and the vertical coordinate is each criterion's own hard distortion. Points average the selected graphs at 15 organization weights. The secondary axis transforms the plotted mean entropy by exponentiation; it does not in general equal the arithmetic mean of per-graph effective counts.}
\label{fig:transductive-relational-distortion-sweep}
\end{figure}

In \cref{fig:transductive-relational-distortion-sweep}, complete collapse preserves all positive edges under the masking interpretation, while organization imposes distinctions that sever some of them, so the curves increase rather than follow the classical shape of an upper-rate distortion budget.
The experiment realizes the lower-occupancy-bound direction of \cref{eq:relational-complexity-distortion-frontier} through a penalty.
Its trajectories are empirical optimization results from a nonconvex solver rather than certified constrained frontiers.

\paragraph{Agreement between source criteria}
Some different fidelity constructions give similar edge priorities on particular sources.
\Cref{tab:transductive-distortion-correlations} compares their normalized importance vectors and their rankings of 256 approximately node-balanced random hard partitions per graph.
$D_F$ and $D_C$ are strongly aligned on unweighted MalNet and PROTEINS, with mean rank correlations 0.975 and 0.981.
The weighted MalNet control reduces the correlation to 0.724, showing that source weights materially affect their agreement.

\begin{table}[t]
\centering\small
\begin{tabular}{lrrrr}
\toprule
Source & \shortstack{Importance cosine\\$(D_F,D_C)$} & \shortstack{Spearman\\$(D_F,D_C)$} & \shortstack{Spearman\\$(D_F,D_{H_2})$} & \shortstack{Spearman\\$(D_C,D_{H_2})$} \\
\midrule
MalNet unweighted & 0.979 & 0.975 & 0.666 & 0.616 \\
MalNet weighted & 0.744 & 0.724 & 0.500 & 0.550 \\
COLLAB & 0.833 & 0.814 & 0.523 & 0.774 \\
PROTEINS & 0.983 & 0.981 & 0.616 & 0.608 \\
\bottomrule
\end{tabular}
\caption{Source-fidelity agreement. Importance cosine compares the normalized source-importance vectors for $D_F$ and $D_C$; the remaining columns report Spearman correlations of distortion values over random hard partitions, averaged over graphs. COLLAB correlations involving $D_{H_2}$ use the 18 sources where that normalized fidelity is defined.}
\label{tab:transductive-distortion-correlations}
\end{table}

The entropy-field criterion is generally less aligned with the other probes.
On PROTEINS at $\lambda_{\mathrm{org}}=0.20$, its solutions have hard $K_{\mathrm{eff}}\approx5.59$ and own distortion 0.091, while the other criteria have effective counts near 4.2 and own distortions 0.149--0.163.
Because these differences mix occupancy and fidelity, the occupancy-matched analysis below is needed to compare preserved structure at similar complexity.
COLLAB also has a slower transition, reaching mean hard $K_{\mathrm{eff}}\ge4$ only at weight 0.50, with less sharp soft assignments at high weights.

\paragraph{Occupancy-matched partitions and transfer}
To separate those effects, solutions are selected independently from the existing sweeps at hard $H_2$ targets from 0.25 to 2.0 nats.
A pair is compared only when each solution is within 0.10 nats of the target and their entropies differ by at most 0.10 nats.
Adjusted Rand index (ARI) evaluates label-invariant partition agreement.
Cross-distortion excess uses one evaluation fidelity at a time: $\Delta_{A\to B}=D_B(c_A)-D_B(c_B)$ and $\Delta_{B\to A}=D_A(c_B)-D_A(c_A)$.
These are approximate occupancy matches of heuristic solutions; the reported excess is descriptive and can be slightly negative.
\Cref{fig:transductive-rate-matched-partition-agreement} summarizes the resulting hard-partition agreement across occupancy targets and criterion pairs.

\begin{figure}[t]
\centering
\includegraphics[width=\linewidth]{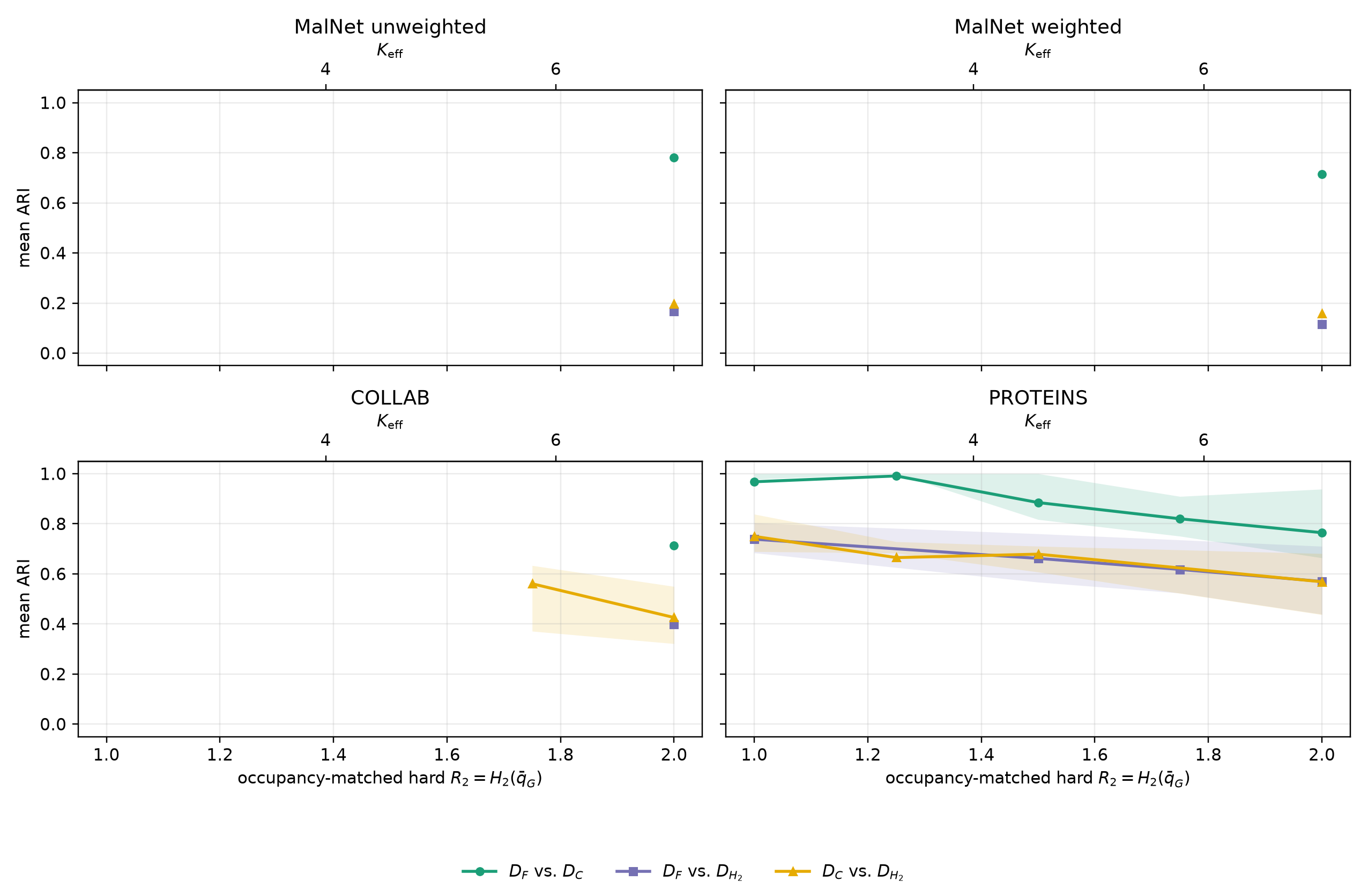}
\caption{Occupancy-matched hard-partition agreement. Mean ARI is shown with interquartile bands for criterion pairs meeting the 0.10-nat matching rule; a point is displayed only when at least five graphs match. Missing low-entropy MalNet points reflect limited sweep coverage.}
\label{fig:transductive-rate-matched-partition-agreement}
\end{figure}

On PROTEINS, $D_F$ and $D_C$ produce mean ARI 0.968 at $H_2=1.0$ and 0.991 at $H_2=1.25$, with near-zero transfer excess.
At $H_2=2.0$, their ARI is 0.764 while the two excesses remain only 0.0046 and 0.0084.
Thus partitions can differ without substantially changing either criterion's preserved content.
By contrast, at $H_2=2.0$ on unweighted MalNet, $D_F$ versus $D_{H_2}$ has ARI 0.166 and $D_C$ versus $D_{H_2}$ has ARI 0.196, compared with 0.781 for $D_F$ versus $D_C$.
The full matched table, six-pair diagnostics, and transfer matrix are retained in \ref{app:transductive-relational-distortion-details}.

Thus, with representation and marginal organization held fixed, changing source fidelity changes source-edge priorities and can change the resulting partition; agreement in priorities, assignments, and cross-fidelity transfer need not coincide.

\subsection{Study 3: An inductive realization of normalized-cut fidelity}
\label{sec:malnet-ncut-experiment}

The third study moves from graph-specific optimization to a shared encoder evaluated on unseen graphs.
It demonstrates an inductive realization of the normalized-cut correspondence: a single finite-codeword encoder is trained with graph-local soft normalized cut while marginal organization is varied, and the resulting hard partitions are evaluated by normalized-cut fidelity, code occupancy, and the post-hoc fidelity $D_F$.
The numerical comparison below uses a per-graph spectral normalized-cut reference.
For related differentiable graph-partitioning constructions, we refer the reader to GAP, the method of Gatti et al., and MinCutPool~\citep{nazi2019gap,gatti2022deep,bianchi2020spectral}.
The probabilistic-cut work of \citet{ghriss2026beyond} addresses expected discrete cut objectives with random denominators, whereas this study uses the deterministic soft-mass surrogate.

\paragraph{Source objects and preprocessing}
MalNet-Tiny function-call graphs are split using the official train, validation, and test assignments~\citep{freitas2021large}.
The inputs, losses, and selection criteria exclude malware labels.
Each directed graph is converted to a simple undirected edge union, self-loops and duplicate edges are removed, and the largest connected component is reindexed.
Requiring at least 512 nodes after cleaning leaves 1995 training, 275 validation, and 584 test graphs.
The exclusions are documented in \ref{app:malnet-ncut-details}.

Node features are structural.
On the cleaned graph they include a constant, degree, $\log(1+d_i)$, graph-normalized degree, and random-walk return estimates for $\operatorname{diag}(P^t)$, where $P$ is the transition matrix of the cleaned undirected graph.
The estimates are cached deterministically for each graph.
Exact return probabilities are permutation-equivariant node features; the finite cached probe realization retains exact samplewise equivariance when its probes are transported with a relabeling.

Directed in-degree, out-degree, total degree, their logarithmic and normalized forms, and PageRank relative to the uniform baseline are computed on the original directed graph before symmetrization and carried to the retained vertices.
Explicit node identifiers and Laplacian eigenvectors are also excluded from the model inputs.
The encoder can use directed structural context even though fidelity is evaluated on the cleaned undirected graph.
This available-input convention is part of the inductive realization.

\paragraph{Encoder and objective}
A GraphGPS-style encoder~\citep{rampasek2022recipe} combines local GIN message passing~\citep{xu2019powerful} with global factorized efficient attention~\citep{shen2021efficient}, followed by a graph-context-conditioned $K=8$ assignment head.
Full architecture and training settings are given in \ref{app:malnet-ncut-details}.
Soft assignments are $q_i=\operatorname{softmax}(\ell_i)$; hard assignments take the maximizing codeword.

For each graph, $\bar q_G$ is the degree-weighted graph-local aggregate defined in \cref{eq:graph-aggregate-code}.
The minibatch loss is
\begin{equation}
\mathcal L_B=\frac1{|B|}\sum_{G\in B}
\left[\operatorname{Ncut}_{K,\mathrm{soft}}(G)
+\lambda_{\mathrm{org}}D_2(\bar q_G\|U_K)\right].
\label{eq:malnet-ncut-loss}
\end{equation}
Here $\operatorname{Ncut}_{K,\mathrm{soft}}$ is the fixed-$K$ soft normalized-cut criterion of \cref{eq:soft-ncut-relational-distortion}, while $D_2(\bar q_G\|U_K)$ uses the uniform-reference organization term of \cref{eq:uniform-renyi}.
The graph-local aggregate supplies the endogenous inverse-mass construction of \cref{eq:soft-nassoc-affinity}, while the external uniform reference specifies marginal organization.
The loss is averaged over graphs only after these graph-local quantities are formed.
When $\lambda_{\mathrm{org}}=0$, the marginal organization term is absent, so training minimizes only the graph-local soft normalized-cut criterion.
Training uses five seeds.

\paragraph{Evaluation}
Validation mean hard fixed-$K$ normalized cut, following \cref{eq:soft-ncut-relational-distortion,eq:hard-ncut}, selects checkpoints.
An empty hard class contributes zero association and hence a unit penalty, so using fewer than eight classes receives the corresponding fixed-$K$ cost.
Hard evaluation uses the induced partition relation against the cleaned source graph.
Post-hoc $D_F$ from \cref{eq:hard-partition-dirichlet-distortion} evaluates a second relational property after training.

A spectral reference computes a normalized-Laplacian embedding and deterministic $K$-means discretization separately for each cleaned graph, following the standard spectral-clustering construction~\citep{ng2001spectral}, and provides a transductive objective comparison.
The learned model instead applies one parameter set after structural preprocessing of each held-out graph.

\begin{table}[t]
\centering\small
\begin{tabular}{lrrrr}
\toprule
$\lambda_{\mathrm{org}}$ & Val Ncut & Test Ncut & Test $D_F$ & Test median Ncut \\
\midrule
0 & $1.129\pm0.063$ & $1.160\pm0.043$ & $0.0606\pm0.0115$ & $1.083\pm0.043$ \\
0.03 & $1.137\pm0.060$ & $1.171\pm0.062$ & $0.0657\pm0.0179$ & $1.093\pm0.061$ \\
0.10 & $1.158\pm0.029$ & $1.186\pm0.033$ & $0.0926\pm0.0055$ & $1.116\pm0.031$ \\
0.20 & $1.136\pm0.039$ & $1.170\pm0.030$ & $0.1117\pm0.0017$ & $1.114\pm0.035$ \\
0.50 & $1.285\pm0.026$ & $1.319\pm0.022$ & $0.1475\pm0.0020$ & $1.270\pm0.017$ \\
\midrule
Spectral reference & 1.143 & 1.159 & 0.0805 & 1.104 \\
\bottomrule
\end{tabular}
\caption{MalNet-Tiny fidelity results. Learned rows show mean and sample standard deviation across five seeds. Lower is better for all displayed fidelities. The deterministic spectral reference is evaluated once on the same cleaned splits. $D_F$ is post-hoc.}
\label{tab:malnet-ncut-results}
\end{table}

The organization diagnostics report soft $K_{\mathrm{eff}}$, the collision-effective code count defined in \cref{eq:effective-code-count}, and soft $D_2$, meaning $D_2(\bar q_G\|U_K)$ with the uniform-reference identity in \cref{eq:uniform-renyi}.

\begin{table}[t]
\centering\small
\begin{tabular}{lrrr}
\toprule
$\lambda_{\mathrm{org}}$ & Largest hard volume fraction & Soft $K_{\mathrm{eff}}$ & Soft $D_2$ \\
\midrule
0 & $0.840\pm0.042$ & $1.42\pm0.14$ & $1.745\pm0.093$ \\
0.03 & $0.807\pm0.072$ & $1.54\pm0.26$ & $1.679\pm0.149$ \\
0.10 & $0.678\pm0.023$ & $2.00\pm0.08$ & $1.416\pm0.043$ \\
0.20 & $0.544\pm0.006$ & $2.45\pm0.03$ & $1.195\pm0.012$ \\
0.50 & $0.375\pm0.012$ & $3.55\pm0.10$ & $0.820\pm0.030$ \\
\midrule
Spectral reference & 0.575 & -- & -- \\
\bottomrule
\end{tabular}
\caption{Organization of the same MalNet-Tiny solutions. Quantities are computed per graph before averaging. Mean effective count and mean divergence need not obey a nonlinear identity after averaging. The reference has a hard effective count of 2.78; it has no soft-assignment counterpart in these columns.}
\label{tab:malnet-organization-results}
\end{table}

\paragraph{Fidelity and organization}
Pure normalized cut attains test hard Ncut $1.160\pm0.043$, close to the spectral mean 1.159.
The learned partition is more concentrated: its largest hard class contains mean volume fraction $0.840\pm0.042$, and soft $K_{\mathrm{eff}}$ is only $1.42\pm0.14$ despite nearly all eight hard classes being active.
The two solutions therefore have similar objective values but different representation organization.

Increasing organization weight to 0.20 reduces the largest hard volume fraction to $0.544\pm0.006$ and raises soft $K_{\mathrm{eff}}$ to $2.45\pm0.03$, while test Ncut remains $1.170\pm0.030$.
At weight 0.50, the volume fraction falls to $0.375\pm0.012$ and effective count rises to $3.55\pm0.10$, but Ncut increases to $1.319\pm0.022$.
The occupancy trend is stable across the sweep, while small Ncut differences among moderate weights should not be overinterpreted relative to seed variation.
\Cref{fig:malnet-ncut-sweep} summarizes this occupancy--fidelity behavior across the organization sweep.

\begin{figure}[t]
\centering
\includegraphics[width=\linewidth]{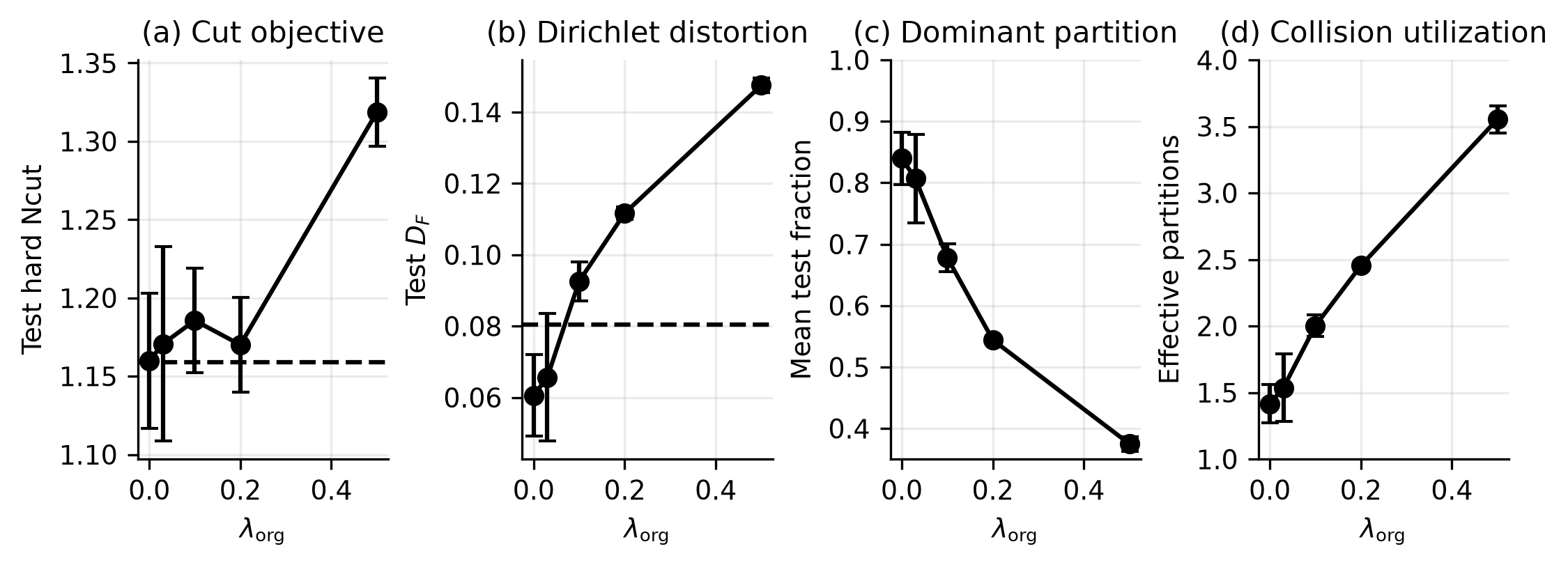}
\caption{MalNet-Tiny organization sweep. Error bars show sample standard deviations over five seeds. Increasing the marginal-organization weight reduces dominant-class volume and broadens soft effective occupancy; post-hoc $D_F$ shows the accompanying change in another relational fidelity. Dashed lines show transductive spectral summaries.}
\label{fig:malnet-ncut-sweep}
\end{figure}

The post-hoc probe also distinguishes partitions that look similar under normalized cut.
Weights 0.03 and 0.20 give almost identical test Ncut, $1.171\pm0.062$ and $1.170\pm0.030$, but $D_F$ values $0.0657\pm0.0179$ and $0.1117\pm0.0017$.
After averaging seeds and pairing by test graph, the latter has higher $D_F$ on 581 of 584 graphs, with mean difference 0.0460; the corresponding mean Ncut difference is $-0.0006$.
Normalized cut and effective-resistance distortion therefore define different fidelities.

Likewise, the pure learned solution has lower mean $D_F$ than the spectral reference, 0.0606 versus 0.0805, but is substantially more concentrated.
The comparison matches nominal assignment capacity but not realized occupancy; the learned and spectral procedures also have different deployment and model-cost conventions.
\Cref{fig:malnet-ncut-partition-example} illustrates cases near the 25th, 50th, 75th, and 95th percentiles of the learned-minus-spectral normalized-cut gap among held-out graphs with at most 900 nodes.

\begin{figure}[t]
\centering
\includegraphics[width=\linewidth]{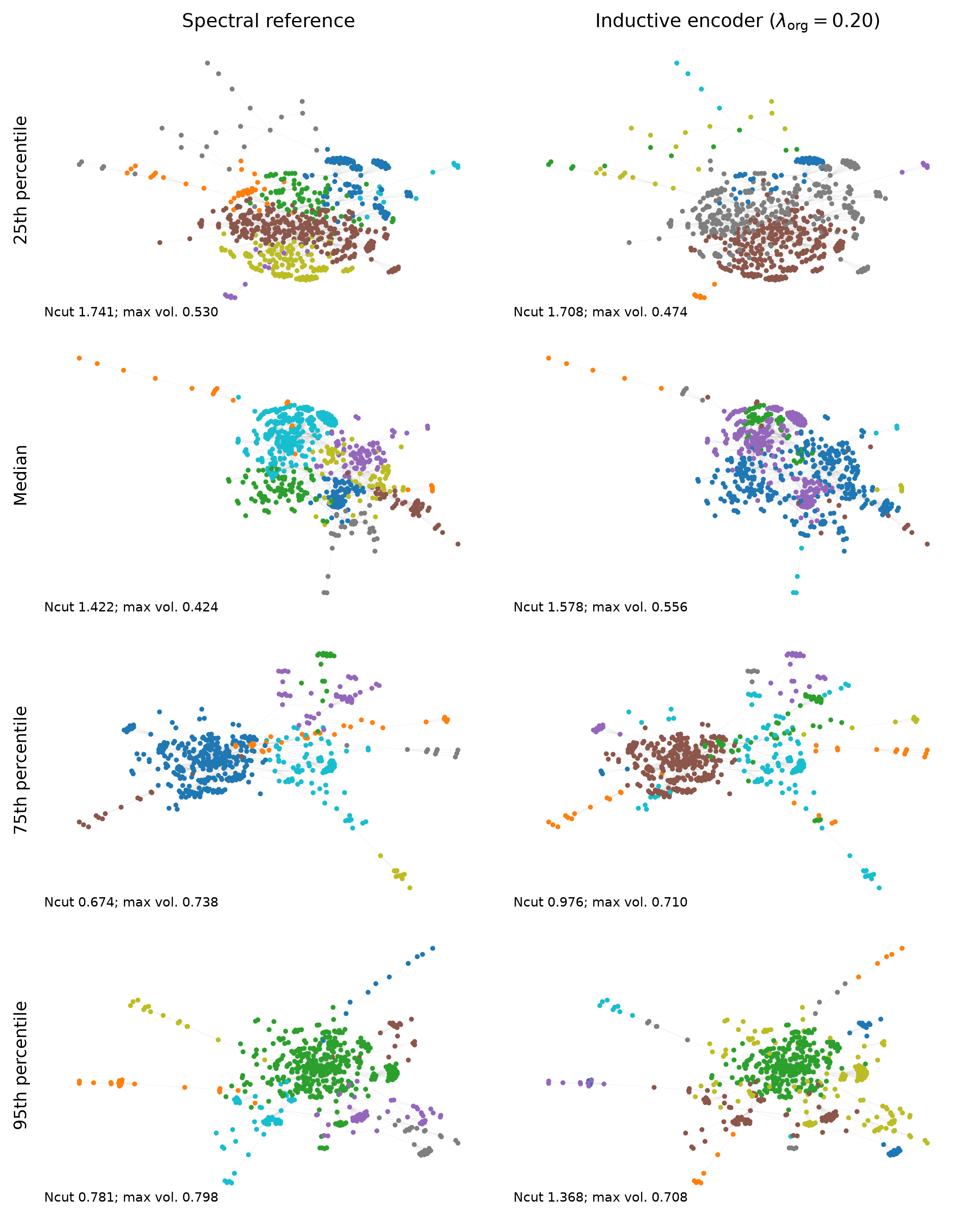}
\caption{Representative MalNet-Tiny partitions at organization weight 0.20, selected near the 25th, 50th, 75th, and 95th percentiles of the learned-minus-spectral hard Ncut gap among held-out graphs with at most 900 nodes. Each row shares a force-directed layout. Partition labels and color assignments are independently permuted for display, so colors do not correspond across panels.}
\label{fig:malnet-ncut-partition-example}
\end{figure}

Together, the study shows that a shared encoder can learn the normalized-cut realization on unseen graphs, while occupancy and post-hoc $D_F$ reveal distinctions not captured by normalized cut alone.

\subsection{Study 4: Reconstruction-defined image representations}
\label{sec:flowers102-experiment}

The fourth study demonstrates a different route by which a source requirement can shape a finite-codeword representation: pixel reconstruction acts through a joint spatial decoder rather than directly on same-code relations.
Holding the 16-bit bottleneck architecture fixed, we vary marginal organization to examine how reconstruction quality and hard-code occupancy interact; a 32-bit run provides a separate capacity check.
Unlike the analytical centroid realization, the synthesis transform jointly interprets neighboring spatial tokens, so no exact per-codeword centroid identity is assumed.
Reconstruction MSE is the source fidelity, while hard collision entropy and effective code count describe representation organization.

Binary and discrete image autoencoders have substantial precedents.
Recurrent binary bottlenecks, VQ-VAE, soft-to-hard vector quantization, and learned transform codecs already combine reconstruction with constrained latent representations~\citep{toderici2016variable,vanDenOord2017neural,agustsson2017soft,balle2017end,balle2018variational}.
Related variational representation-learning work explicitly studies the rate--distortion trade-off between compression and reconstruction accuracy~\citep{alemi2018fixing}.
The Latent Bernoulli Autoencoder is a particularly close precedent: it trains an end-to-end reconstruction autoencoder with a hard $\{-1,+1\}$ latent obtained by zero thresholding and a straight-through gradient surrogate~\citep{fajtl2020latent}.
Codebook-free alternatives include finite scalar quantization, lookup-free quantization, and binary spherical quantization~\citep{mentzer2024finite,yu2024language,zhao2025binary}, while token-prior regularization also precedes this study~\citep{zhang2023regularized}.
The experiment isolates a collision-based organization mechanism in a purpose-limited architecture.

\paragraph{Data and representation}
Oxford Flowers102 contains 8189 images from 102 categories~\citep{nilsback2008automated}; labels are not used.
For this reconstruction study, the official test split supplies the 6149 training images, the official validation split supplies 1020 validation images, and the official training split supplies 1020 test images.
This nonstandard split is reported explicitly as part of the reconstruction protocol.
Training uses image augmentation; validation and test use deterministic $256\times256$ center-crop preprocessing.

The convolutional analysis transform produces $b$ logits at each position of a $32\times32$ latent grid.
Each position is a spatial token with $b$ hard bits, and code statistics pool tokens across images and positions.
The nominal raw latent capacity is $b/64$ bits per input pixel: 0.25 bpp for 16 bits and 0.50 bpp for 32 bits.
These are fixed latent capacities before any entropy coding.

The convolutional synthesis transform jointly interprets the spatial tokens to produce a sigmoid RGB output.
GDN/IGDN follows established nonlinear transform-coding practice~\citep{balle2017end}.
There are no encoder--decoder skip connections, so every reconstruction passes through the bottleneck.
The encoder and decoder are shared across images; the per-image nominal latent bpp counts the latent bits, while model size belongs to the complete coding convention.

At evaluation, bit $r$ at position $u$ is $h_r(x,u)=1$ if its logit $\ell_r(x,u)>0$, and $-1$ otherwise.
Hard tokens therefore use the centered sign convention $\{-1,+1\}^b$, a relabeling of the binary alphabet $\{0,1\}^b$ in \cref{eq:factorized-code}.
Centering the two states symmetrically about the zero threshold is convenient for the relaxation through which gradients are propagated during training, while hard equality and collision statistics are unchanged by the relabeling.
Training uses the mean-group relaxation described in \ref{app:flowers102-details}.
DiffPrune provides earlier work on deterministic approximate binary gates~\citep{shulman2020diffprune}.
The training surrogate used here more directly adapts the sign-partitioned group-mean transformation of \citet{shulman2021exact}, originally introduced for binary network parameters, to spatial latent variables, and combines it with the concentration penalty of \cref{eq:concentration-reduction} rather than reproducing the original continuation formulation unchanged.
Binary latent autoencoders and sign-based discrete tokenizers precede this study; to our knowledge, prior discrete-autoencoder work has not used this group-coupled relaxation to train a hard-sign latent bottleneck.

\paragraph{Collision-based organization}
The training loss has the form
\begin{equation}
\mathcal L=\mathcal L_{\mathrm{mse}}
+\lambda_{\mathrm{conc}}\mathcal L_{\mathrm{conc}}
+\lambda_{\mathrm{org}}\widetilde D_2.
\label{eq:flowers-objective}
\end{equation}
The last term acts on a differentiable threshold-margin surrogate, while evaluation uses literal collisions of the hard sign tokens.
Let $s_r>0$ be a detached per-channel RMS logit scale maintained by an exponential moving average.
For gain $\gamma>0$, define
\begin{equation}
\eta_r(x,u):=\ell_r(x,u)/s_r,\qquad
\tilde p_r(x,u)=\sigma(\gamma\eta_r(x,u)).
\label{eq:flowers-threshold-probability}
\end{equation}
Equivalently, an independent logistic threshold perturbation $T_r$ with location zero and scale $s_r/\gamma$ gives $\tilde p_r=\Pr[T_r<\ell_r]$.
Independence is across bit draws and independently encoded tokens; using one shared random threshold would define a different collision model.

For two tokens, whole-code collision under these factorized probabilities is
\begin{equation}
\widetilde k_{ij}
=\prod_{r=1}^b\left[\tilde p_{ir}\tilde p_{jr}+(1-\tilde p_{ir})(1-\tilde p_{jr})\right].
\label{eq:flowers-threshold-collision}
\end{equation}
This is the factorized binary collision of \cref{eq:bit-collision}, instantiated for the code in \cref{eq:factorized-code} with the smooth threshold probabilities $\tilde p_r$.
The sampled organization loss is
\begin{equation}
\widetilde D_2=b\log2+\log\widehat{\E}[\widetilde k_{ij}],
\label{eq:flowers-separation}
\end{equation}
Since $K=2^b$, the population identities in \cref{eq:population-collision,eq:uniform-renyi} give $D_2(q_Z\|U_K)=b\log2+\log\E[k_\theta(X,X')]$.
The implemented $\widetilde D_2$ replaces that population collision by the smooth threshold model and a finite minibatch pair estimate, so it is an optimization surrogate rather than the hard-code $H_2$ and $K_{\mathrm{eff}}$ statistics reported at evaluation.
Because the logarithm is applied to a finite sampled collision estimate, $\widetilde D_2$ is not generally an unbiased estimator of the corresponding population R\'enyi divergence, nor is a finite-sample realization guaranteed to inherit its nonnegativity.
The reported hard $H_2$ and $K_{\mathrm{eff}}$ are computed from the empirical distribution of complete hard spatial tokens pooled across images and positions in the evaluation split, using \cref{eq:collision-entropy,eq:effective-code-count}; this pooled token collection is the declared evaluation population.

\paragraph{Ablations and results}
The 16-bit sweep uses weights $0$, $10^{-7}$, $10^{-6}$, and $10^{-5}$, plus a 32-bit capacity run at $10^{-6}$.
All settings use the same data split and one seed, 1337; validation hard MSE selects checkpoints.
Large changes therefore illustrate mechanism while small differences remain statistically unresolved.
Full checkpoint, MSE, MS-SSIM, and optimization settings are retained in \ref{app:flowers102-details}.

\begin{table}[t]
\centering\small
\begin{tabular}{lrrrrr}
\toprule
Setting & Bits & Nominal bpp & PSNR & Hard $H_2$ & $K_{\mathrm{eff}}$ \\
\midrule
No organization & 16 & 0.25 & 19.86 & 1.98 & 7.2 \\
Weaker, $10^{-7}$ & 16 & 0.25 & 24.36 & 7.00 & 1\,098 \\
Baseline, $10^{-6}$ & 16 & 0.25 & 24.83 & 9.50 & 13\,388 \\
Stronger, $10^{-5}$ & 16 & 0.25 & 24.56 & 10.29 & 29\,569 \\
Capacity, $10^{-6}$ & 32 & 0.50 & 25.88 & 12.46 & 257\,753 \\
\bottomrule
\end{tabular}
\caption{Flowers102 hard-code reconstruction and organization at validation-selected checkpoints. PSNR is in dB and $H_2$ in nats. Each row is one run. The nominal bpp describes raw latent bits before entropy coding.}
\label{tab:flowers102-results}
\end{table}

Without organization, reconstruction learns a nontrivial partition but uses only 38 distinct codewords, with effective count about 7.2, hard collision $1.38\times10^{-1}$, and 19.86 dB PSNR.
Thus reconstruction supplies a reason to distinguish inputs while allowing concentrated use of the available 16-bit alphabet.
Several bit channels become saturated in this run.

A weight of $10^{-7}$ raises PSNR to 24.36 dB and effective count to about $1.1\times10^3$.
The $10^{-6}$ setting gives the best 16-bit reconstruction among these runs, 24.83 dB and 0.896 MS-SSIM, with effective count 13\,388.
Increasing the weight to $10^{-5}$ further broadens occupancy to 29\,569 effective codes while PSNR decreases to 24.56 dB.
The few-tenths-of-a-decibel difference is unresolved in this single-seed study, while the sweep suggests that utilization can continue increasing after reconstruction quality plateaus.
\Cref{fig:flowers102-dynamics} shows how reconstruction and hard-code utilization evolve across the organization sweep.
\begin{figure}[t]
\centering
\includegraphics[width=\linewidth]{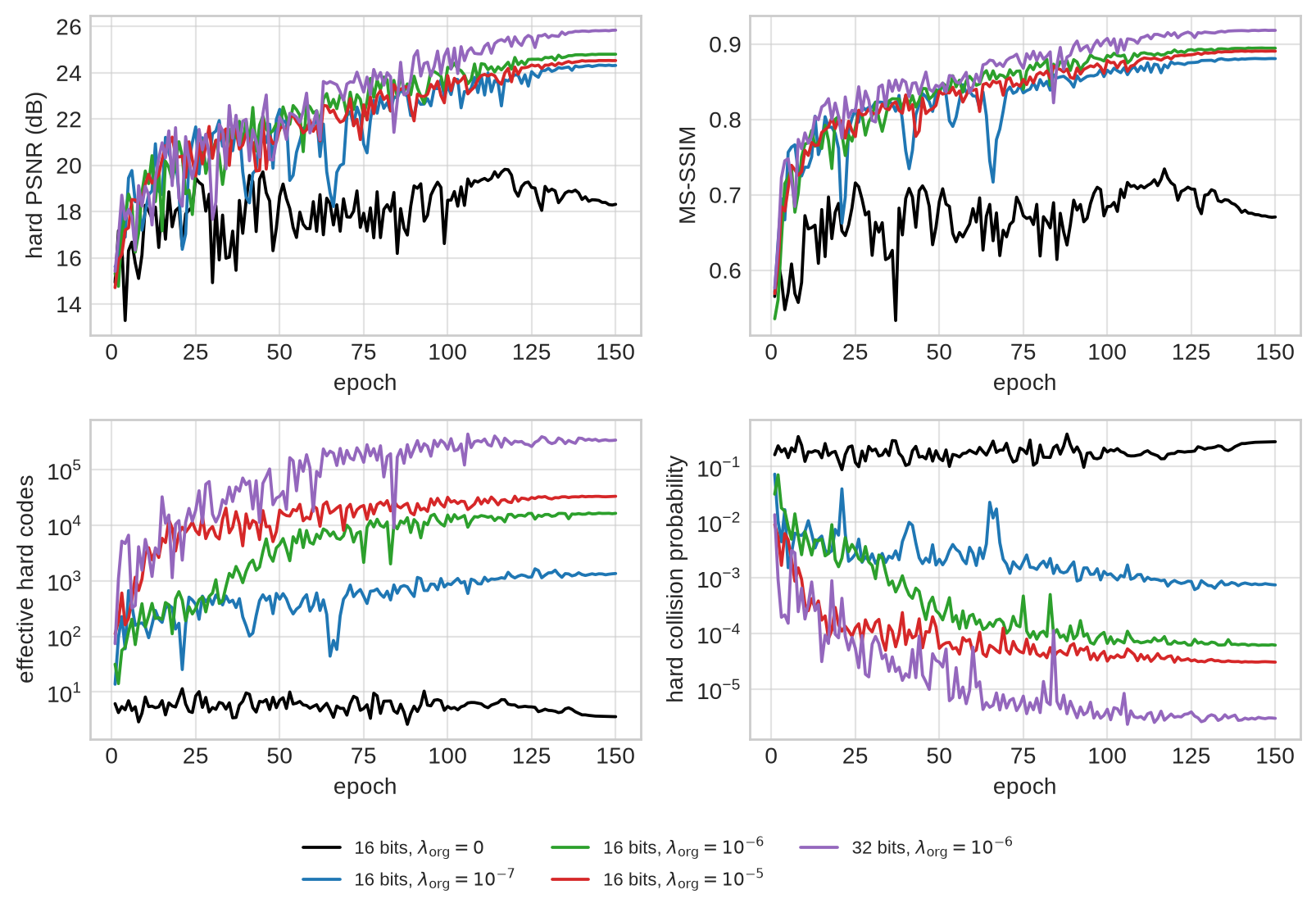}
\caption{Validation dynamics of the image reconstruction ablations. Utilization is measured from hard spatial tokens; smooth probabilities supply the training surrogate. The zero-weight run is highly concentrated; larger weights broaden occupancy. The 32-bit run has a larger nominal bottleneck capacity.}
\label{fig:flowers102-dynamics}
\end{figure}

Support size and effective occupancy tell different stories.
The baseline and stronger 16-bit settings activate 65\,208 and 65\,528 of 65\,536 possible words, but their effective counts are only 13\,388 and 29\,569.
The 32-bit run improves reconstruction to 25.88 dB and 0.920 MS-SSIM with effective count about $2.58\times10^5$.
It serves as a capacity check in which the nominal bpp, code space, and uniform-reference target all change.
Empirical occupancy also remains sample dependent, especially in a very large alphabet.
\Cref{fig:flowers102-16bit-tradeoff} summarizes the four fixed-capacity 16-bit operating points.
\begin{figure}[t]
\centering
\includegraphics[width=0.72\linewidth]{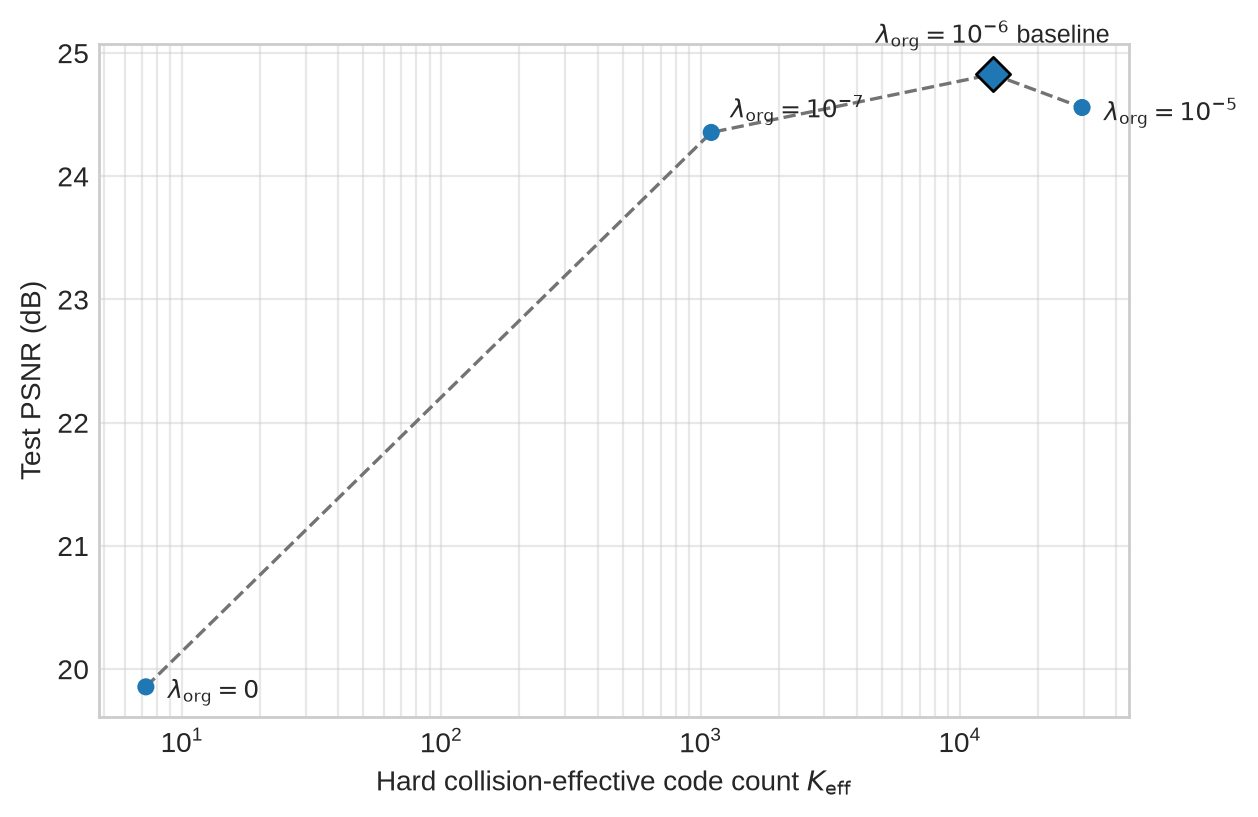}
\caption{The four 16-bit test operating points. Organization initially improves both effective code use and PSNR; beyond the baseline it increases utilization while reconstruction slightly worsens. The single-seed operating points provide a descriptive trajectory rather than a statistically resolved frontier.}
\label{fig:flowers102-16bit-tradeoff}
\end{figure}

Qualitative hard-code reconstructions are retained in \cref{fig:flowers102-reconstruction-examples}.
Taken together, reconstruction can train a hard finite representation through a joint decoder without ensuring broad alphabet use; marginal organization changes that use separately, while token collision remains an organization statistic rather than the reconstruction fidelity itself.
The 32-bit run separately probes increased nominal capacity.

\subsection{Study 5: Teacher-defined relations without reconstruction}
\label{sec:flowers102-teacher-experiment}

The fifth study asks whether an externally supplied source relation can organize a finite hard representation without reconstruction or a marginal-organization penalty.
A frozen visual teacher defines baseline-relative favored and disfavored relations among spatial patches, and an independent student encoder is trained only through pair-specific whole-code collision.
Hard relation-conditioned collision provides the primary evaluation.
A decoder fitted only after freezing the encoder supplies a secondary cross-fidelity probe of recoverable image structure.

Transferring relations rather than individual features is established in relational and similarity-preserving distillation~\citep{park2019relational,tung2019similarity}.
Dense visual precedents include STEGO's distillation of feature correspondences, TokenCut's use of transformer patch affinities, and DeepCut's graph clustering of pretrained image features~\citep{hamilton2022unsupervised,wang2022selfsupervised,aflalo2023deepcut}.
Vector-quantized teacher features also appear in BEiT v2~\citep{peng2022beitv2}.
The selected realization combines an independent image encoder, cross-image baseline-relative relations, and a whole-code collision objective.

\paragraph{Source relation and student}
The data split and image preprocessing match those of \cref{sec:flowers102-experiment}.
A frozen DINO ViT-S/8 teacher~\citep{caron2021emerging} sees the same $256\times256$ image crop as the student, with the teacher's input normalization.
Its final normalized spatial tokens form a $32\times32$ grid, matching the student's latent grid.
Let $e_i$ be the frozen teacher embedding of spatial patch token $i$.

The student uses the same convolutional analysis transform as Study 4, with $b=16$ logits per spatial position.
Hard evaluation uses zero-threshold signs.
During representation learning,
\begin{equation}
p_r(i)=\sigma(\ell_r(i)/\tau_c),\qquad \tau_c=0.25.
\end{equation}
There is no mean-group relaxation, concentration term, marginal organization penalty, entropy model, or reconstruction decoder at this stage.

\paragraph{One teacher relation, two relational roles}
Each minibatch supplies a bounded sample of spatial tokens, and valid candidates for token $i$ form $\mathcal N_i$.
Self-pairs and same-image pairs are excluded in the reported run.
For teacher temperature $\tau_T=0.1$, define
\begin{equation}
S^{\mathrm{patch}}_{ij}
=\frac{\exp(e_i^\top e_j/\tau_T)}
{\sum_{j'\in\mathcal N_i}\exp(e_i^\top e_{j'}/\tau_T)},
\qquad
S_{\mathrm{bg},ij}=\nu_i=|\mathcal N_i|^{-1}.
\label{eq:teacher-row-relation}
\end{equation}
The teacher mapping and relation-construction rule are fixed, but $S^{\mathrm{patch}}$ is row-normalized within each sampled candidate set and is therefore not one immutable affinity matrix over the full dataset.
The source-side baseline is uniform over valid candidates.
The resulting favored and disfavored pair weights are detached constants during optimization:
\begin{equation}
S_{+,ij}=[S^{\mathrm{patch}}_{ij}-S_{\mathrm{bg},ij}]_+,\qquad
S_{-,ij}=[S_{\mathrm{bg},ij}-S^{\mathrm{patch}}_{ij}]_+.
\label{eq:teacher-signed-weights}
\end{equation}
Thus a single teacher relation supplies both favored and disfavored collisions through the baseline-relative construction of \cref{eq:signed-relation}.

The cosine and candidate mask are symmetric, but row normalization makes $S^{\mathrm{patch}}$ directional.
The student collision is symmetric and both directions are included.
Accordingly, each side of the loss aggregates the corresponding directional weight masses on an unordered pair.
This aggregation is performed separately for favored and disfavored weights and can differ from thresholding a single symmetrized signed relation.

The differentiable whole-code probability is
\begin{equation}
\widetilde k_{ij}=\prod_{r=1}^b[p_r(i)p_r(j)+(1-p_r(i))(1-p_r(j))].
\end{equation}
This is the factorized whole-code collision of \cref{eq:bit-collision}, evaluated using the student's smooth Bernoulli probabilities.
The representation-side neutral collision level is the population collision induced by the uniform codeword reference $U_K$:
\[
\sum_z U_K(z)^2=\frac{1}{K}=2^{-b}.
\]
Define
\begin{equation}
\delta_{ij}:=\operatorname{logit}(\widetilde k_{ij})-\operatorname{logit}(1/K).
\label{eq:teacher-centered-margin}
\end{equation}
The implementation clips the probability away from zero and one before evaluating the logit; the numerical convention is described in \ref{app:flowers102-teacher-details}.
The source baseline $S_{\mathrm{bg},ij}$ determines the favored and disfavored requirements, while $1/K$ centers the representation-side margin.
The uniform distribution $U_K$ is used here only to define this neutral whole-code collision level: the objective does not match the aggregate distribution $q_Z$ to $U_K$ and has no $\mathcal L_{\mathrm{org}}$ term.

The two sides of the objective are
\begin{equation}
\begin{aligned}
\mathcal L_{\mathrm{align}}
&=\frac{\sum_{ij}S_{+,ij}\operatorname{softplus}(-\delta_{ij})}{\sum_{ij}S_{+,ij}},\\
\mathcal L_{\mathrm{sep}}^{\mathrm{rel}}
&=\frac{\sum_{ij}S_{-,ij}\operatorname{softplus}(\delta_{ij})}{\sum_{ij}S_{-,ij}},\\
\mathcal L_{\mathrm{teacher}}
&=\tfrac12\left(\mathcal L_{\mathrm{align}}+\mathcal L_{\mathrm{sep}}^{\mathrm{rel}}\right),
\end{aligned}
\label{eq:dino-teacher-loss}
\end{equation}
with zero-mass sides omitted from the average and zero loss when neither side is active.
Favored relations encourage collision above the reference, while disfavored relations encourage it below the reference.
The objective preserves baseline-relative distinctions without reconstructing teacher cosine values.

\paragraph{Evaluation and results}
The representation checkpoint is selected by validation teacher loss.
Hard sign tokens provide the reported occupancy statistics and relation-conditioned collision rates.
Let $c(i)\in\mathcal Z$ denote the complete hard codeword assigned to patch token $i$ by zero-thresholding the student logits.
The relation-conditioned collision rates are
\[
\widehat\kappa_+
=\frac{\sum_{ij}S_{+,ij}\ind[c(i)=c(j)]}{\sum_{ij}S_{+,ij}},
\qquad
\widehat\kappa_-
=\frac{\sum_{ij}S_{-,ij}\ind[c(i)=c(j)]}{\sum_{ij}S_{-,ij}}.
\]
The pooled empirical hard-code collision is
\[
\widehat\kappa_{\mathrm{pop}}
=\sum_z\widehat q_Z(z)^2
=e^{-\widehat H_2}.
\]
Favored-pair enrichment and disfavored-pair suppression are, respectively,
\[
\frac{\widehat\kappa_+}{\widehat\kappa_{\mathrm{pop}}},
\qquad
1-\frac{\widehat\kappa_-}{\widehat\kappa_{\mathrm{pop}}}.
\]
These ratios use the pooled empirical hard-code collision as their baseline and are descriptive rather than calibrated against a candidate-matched null.

\begin{table}[t]
\centering\small
\begin{tabular}{lrr}
\toprule
Metric & Validation & Test \\
\midrule
Representation epoch & \multicolumn{2}{c}{146} \\
Teacher loss & 0.389 & 0.389 \\
Favored-pair hard collision & $1.35\times10^{-2}$ & $1.29\times10^{-2}$ \\
Disfavored-pair hard collision & $1.37\times10^{-4}$ & $1.47\times10^{-4}$ \\
Aggregate hard collision & $1.03\times10^{-3}$ & $1.04\times10^{-3}$ \\
Favored-pair enrichment & 13.13 & 12.40 \\
Disfavored-pair suppression & 0.866 & 0.858 \\
Hard $H_2$ (nats) & 6.88 & 6.87 \\
$K_{\mathrm{eff}}$ & 974 & 965 \\
Active hard words & 43\,036 & 43\,028 \\
\midrule
Secondary decoder epoch & \multicolumn{2}{c}{140} \\
Decoder MSE & 0.01134 & 0.01121 \\
Decoder PSNR (dB) & 19.45 & 19.50 \\
Decoder MS-SSIM & 0.714 & 0.718 \\
\bottomrule
\end{tabular}
\caption{DINO-defined relational learning on Flowers102. Representation and secondary decoder checkpoints use separate validation criteria. The encoder is frozen during decoder training. The study uses one seed.}
\label{tab:flowers102-teacher-results}
\end{table}

On test images, teacher-favored pairs collide with probability $1.29\times10^{-2}$, compared with aggregate collision $1.04\times10^{-3}$.
Teacher-disfavored pairs collide with probability $1.47\times10^{-4}$.
The corresponding enrichment is $12.4\times$ and suppression is 85.8\%.
The weighted mean teacher cosine is 0.579 on favored relations and 0.311 on disfavored relations; both are positive, and their separation confirms that the baseline-relative split behaves as intended.

The hard representation has about 965 collision-effective words across 43\,028 active words, reflecting broad support with concentrated occupancy.
The result shows that pair-specific alignment and relational separation can organize a noncollapsed code without an explicit marginal entropy penalty.
The evaluation covers the declared relational distinctions and their hard-collision diagnostics rather than numerical reconstruction of the full teacher geometry.

\begin{figure}[t]
\centering
\includegraphics[width=0.68\linewidth]{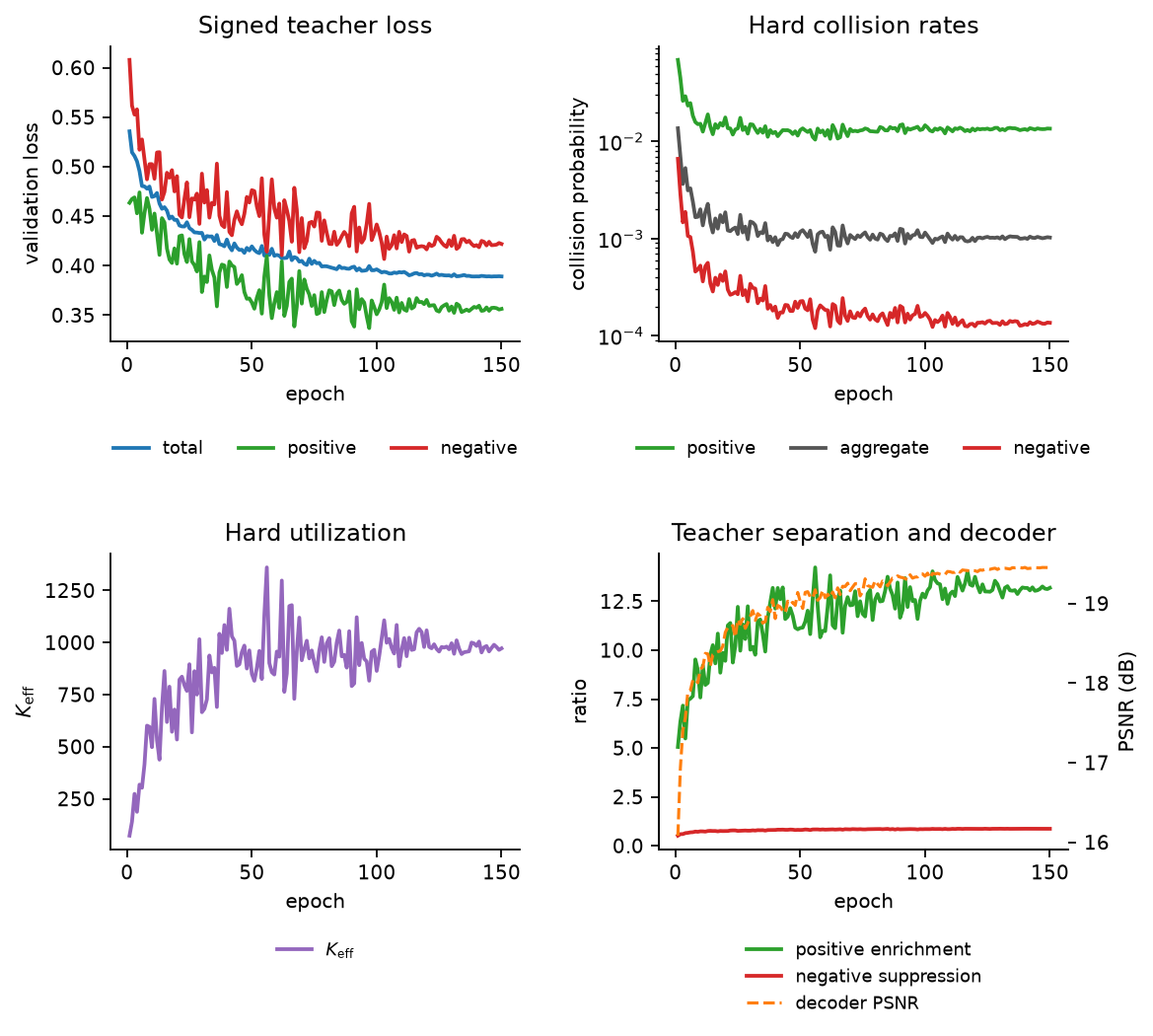}
\caption{Validation dynamics of the baseline-relative teacher study. Hard collisions on favored pairs are enriched and those on disfavored pairs are suppressed relative to the pooled occupancy baseline as the code develops. The secondary decoder curve belongs to a separate stage after the encoder is frozen.}
\label{fig:flowers102-teacher-dynamics}
\end{figure}

After representation learning, a decoder with the synthesis-transform form of Study 4 is trained on frozen hard codes and selected by validation MSE.
It achieves test PSNR 19.50 dB and MS-SSIM 0.718.
This shows that the retained representation supports nontrivial image reconstruction even though pixel fidelity did not shape it directly.
In the absence of random-encoder and shuffled-teacher decoder controls, the result serves as a compatibility probe between two fidelities rather than a causal attribution or comparison with the reconstruction-trained codec.
Qualitative examples are retained in \cref{fig:flowers102-teacher-reconstruction-examples}.

The learned code remains noncollapsed while placing substantially more collisions on teacher-favored pairs and fewer on teacher-disfavored pairs than the pooled occupancy baseline, demonstrating that pair-specific relational requirements can directly organize finite-codeword equality without a decoder in representation learning.

\section{Conclusion}
\label{sec:conclusion}

Relational compression treats relational structure as fidelity-bearing content by making the source relation, retained description, reconstructed or evaluated relation, fidelity criterion, and constrained resource explicit.
Finite-codeword collision provides one developed realization: same-codeword probability connects local relational requirements with aggregate R\'enyi-2 occupancy and positive-spherical geometry; centroid reconstruction and normalized association/cut provide exact correspondences; and controlled synthetic, graph, and image studies show how reconstruction-defined, graph-defined, and teacher-defined requirements can shape finite representations within the same formulation.

The framework also opens opportunities to compare description mechanisms under complete-description budgets rather than occupancy alone, develop richer relation-level decoders when equality is insufficient, and extend relational compression toward broader source classes and operational coding formulations.
More generally, it provides a common way to ask which relational distinctions should be retained by a constrained representation and what description is sufficient to retain them.

\clearpage
\bibliographystyle{plainnat}
\bibliography{refs}
\clearpage
\appendix
\section{Higher-Order Collision Geometry}
\label{app:higher-order-collision-geometry}

The main development uses the pairwise collision kernel $k_\theta\equiv k_{\theta,2}$, but the same event has an integer-order extension.
For integer $\alpha\ge2$ and conditionally independent draws $Z_m\sim q_\theta(\cdot\mid x_m)$,
\begin{equation}
k_{\theta,\alpha}(x_1,\ldots,x_\alpha)
=\sum_z\prod_{m=1}^{\alpha}q_\theta(z\mid x_m)
=\Pr[Z_1=\cdots=Z_\alpha\mid x_1,\ldots,x_\alpha].
\label{eq:higher-order-collision}
\end{equation}
A sequence converging to deterministic assignments gives the indicator that every element of the tuple belongs to the same hard class.
At order two, $k_{\theta,2}(x,x')=k_\theta(x,x')$, recovering \cref{eq:collision-kernel}.
Collision arity is a property of this representation statistic; it need not equal the arity of a source relation.

For even order $\alpha=2m$, define the elementwise product feature
\begin{equation}
\phi_m(x_1,\ldots,x_m)
=q_\theta(\cdot\mid x_1)\odot\cdots\odot q_\theta(\cdot\mid x_m).
\end{equation}
Then
\begin{equation}
k_{\theta,2m}(x_1,\ldots,x_{2m})
=\langle\phi_m(x_1,\ldots,x_m),\phi_m(x_{m+1},\ldots,x_{2m})\rangle.
\label{eq:even-order-kernel}
\end{equation}
It is an ordinary positive-semidefinite kernel on $m$-tuples.
Odd order $2m+1$ gives the analogous inner product of $\phi_m$ and $\phi_{m+1}$, naturally a cross-affinity between tuple spaces of different arities rather than a symmetric kernel on individual inputs.

\paragraph{Triplet-structured relational requirements}
A learning objective may involve three inputs while being composed entirely from pairwise collision geometry.
Writing $q_a:=q_\theta(\cdot\mid x_a)$, $q_p:=q_\theta(\cdot\mid x_p)$, and $q_n:=q_\theta(\cdot\mid x_n)$, and reusing the spherical map $u(q)$ from \cref{eq:sphere-map}, one triplet-structured angular objective is
\begin{equation}
\mathcal L_{\mathrm{trip}}
=
\left[
\zeta-u(q_a)^\top u(q_p)+u(q_a)^\top u(q_n)
\right]_+,
\qquad \zeta>0.
\end{equation}
This objective compares pairwise normalized collisions from \cref{eq:normalized-collision}; it is not a third-order collision statistic.
Angular comparison of high-dimensional graph representations also appears in graph summarization through cosine-distance spherical clustering~\citep{tsalouchidou2020scalable}.
There it acts on graph-derived connectivity-history vectors, whereas here it acts on categorical assignment distributions, for which $1-u(q)^\top u(q')=1-\cos\vartheta$ is cosine distance and the underlying unnormalized inner product has a collision-probability interpretation.

A genuine third-order collision instead requires all three independently encoded inputs to receive the same codeword:
\begin{equation}
k_{\theta,3}(x_a,x_p,x_n)
=
\sum_z q_\theta(z\mid x_a)q_\theta(z\mid x_p)q_\theta(z\mid x_n)
=
\Pr[Z_a=Z_p=Z_n].
\end{equation}
The related anchor--favored--disfavored event is
\begin{equation}
\Pr[Z_a=Z_p\ne Z_n]
=
\sum_z q_\theta(z\mid x_a)q_\theta(z\mid x_p)
\bigl(1-q_\theta(z\mid x_n)\bigr)
=
k_\theta(x_a,x_p)-k_{\theta,3}(x_a,x_p,x_n),
\end{equation}
with hard-code limit
\[
\ind[c_\theta(x_a)=c_\theta(x_p),\;c_\theta(x_a)\ne c_\theta(x_n)].
\]
Thus source-relation arity, learning-objective arity, and collision-event order are distinct modelling choices.

For iid inputs $X_1,\ldots,X_\alpha\sim\mu$,
\begin{equation}
\E[k_{\theta,\alpha}(X_1,\ldots,X_\alpha)]
=\sum_zq_Z(z)^\alpha.
\label{eq:higher-order-population-collision}
\end{equation}
The factorization follows by independence of the $\alpha$ inputs.
At $\alpha=2$, this identity recovers \cref{eq:population-collision}.
Consequently, integer-order R\'enyi entropy~\citep{renyi1961measures} satisfies
\begin{equation}
H_\alpha(q_Z)
=\frac1{1-\alpha}\log\sum_zq_Z(z)^\alpha
=-\frac1{\alpha-1}\log\E[k_{\theta,\alpha}].
\label{eq:renyi-alpha-collision}
\end{equation}
The case $\alpha=2$ recovers \cref{eq:collision-entropy}.
The literal $\alpha$-way same-code event is defined here for integer $\alpha\ge2$, whereas the R\'enyi family extends to appropriate real orders under its usual conventions.
Shannon entropy arises through the continuous limit $\alpha\to1$, not from a nontrivial one-sample collision event.

For a fixed positive reference $r$, define
\begin{equation}
a_{\alpha,r}(x_1,\ldots,x_\alpha)
=\sum_zr(z)^{1-\alpha}\prod_{m=1}^{\alpha}q_\theta(z\mid x_m).
\end{equation}
$a_{2,r}(x,x')=a_r(x,x')$ recovers the reference-weighted affinity in \cref{eq:reference-affinity}.
Its expectation is $\sum_zq_Z(z)^\alpha r(z)^{1-\alpha}$, giving
\begin{equation}
D_\alpha(q_Z\|r)
=\frac1{\alpha-1}\log\E[a_{\alpha,r}].
\label{eq:renyi-alpha-reference-affinity}
\end{equation}
For a uniform reference,
\begin{equation}
D_\alpha(q_Z\|U_K)=\log K-H_\alpha(q_Z).
\label{eq:uniform-renyi-alpha}
\end{equation}
At $\alpha=2$, this reduces to \cref{eq:uniform-renyi}.
This includes the usual limiting and support conventions at orders one and zero.
At fixed $K$, minimizing this divergence as an organization penalty encourages larger occupancy entropy, whereas adding $+\lambda H_\alpha(q_Z)$ to a minimized objective encourages concentration.

\subsection{R\'enyi entropy in spherical coordinates}
\label{app:renyi-spherical-coordinates}

The integer-event interpretation and the real-order entropy family should be distinguished.
For $\alpha>0$, $\alpha\ne1$, let $u=u(q)$ be the spherical map from \cref{eq:sphere-map}.
Substitution of $q_i=u_i/\|u\|_1$ yields
\begin{equation}
H_\alpha(q)
=\frac1{1-\alpha}\left[\log\sum_i u_i^\alpha-\alpha\log\|u\|_1\right].
\label{eq:renyi-spherical}
\end{equation}
At order two, $\sum_i u_i^2=1$ gives $H_2(q)=2\log\|u\|_1$.
Taking the order-one limit gives \cref{eq:shannon-spherical}.
All these quantities can be expressed in the same coordinates, but higher power sums generally depend on more than one ordinary spherical angle.

\section{Conditional Collision and Dependence}
\label{app:conditional-collision-dependence}

The aggregate population collision and $H_2(q_Z)$ developed in \cref{eq:population-collision,eq:collision-entropy} describe how a finite code is occupied and organized across a population, but they do not by themselves determine how strongly the sampled code depends on its input.
In particular, broad aggregate occupancy can arise even from an input-independent encoder.
This appendix records two collision-related dependence functionals that make this distinction explicit: one compares conditional self-collision with aggregate population collision, whereas the other uses R\'enyi divergence of the induced input--code joint distribution from its independent product.
It also distinguishes input--code dependence from dependence between relational endpoints, for which collision alignment observes only the same-codeword event.

Let $X,X'\overset{\mathrm{iid}}{\sim}\mu$.
Using the collision kernel and self-collision in \cref{eq:collision-kernel,eq:self-collision}, together with the aggregate code distribution and population collision identity in \cref{eq:aggregate-code,eq:population-collision}, define
\begin{equation}
\kappa_{\mathrm{self}}
:=\E_X[k_\theta(X,X)]
=\E_X\sum_zq_\theta(z\mid X)^2,
\qquad
\kappa_{\mathrm{pop}}
:=\E_{X,X'\overset{\mathrm{iid}}{\sim}\mu}[k_\theta(X,X')]
=\sum_zq_Z(z)^2.
\end{equation}
The conditional collision-entropy functional is
\begin{equation}
H_2^{\mathrm{coll}}(Z\mid X)=-\log \kappa_{\mathrm{self}},
\end{equation}
Because the logarithm is taken after averaging conditional self-collision, this functional is generally not equal to $\E_X[H_2(q_\theta(\cdot\mid X))]$.
Using the collision entropy $H_2(q_Z)$ from \cref{eq:collision-entropy}, a corresponding difference is
\begin{equation}
I_2^{\mathrm{coll}}(X;Z)
=H_2(q_Z)-H_2^{\mathrm{coll}}(Z\mid X)
=\log\frac{\kappa_{\mathrm{self}}}{\kappa_{\mathrm{pop}}}.
\label{eq:collision-information}
\end{equation}
It is nonnegative because
\begin{equation}
\kappa_{\mathrm{self}}-\kappa_{\mathrm{pop}}
=\sum_z\operatorname{Var}_\mu(q_\theta(z\mid X))\ge0.
\end{equation}
It vanishes exactly when the assignment vector is constant almost surely, equivalently when the sampled code is independent of the input under this population model.
For deterministic assignments, $\kappa_{\mathrm{self}}=1$ and the difference equals $H_2(q_Z)$.
No general data-processing or unique R\'enyi-mutual-information characterization is asserted for this entropy difference.

Let $P_X$ denote the distribution of $X$, so $P_X=\mu$, and let $P_{XZ}$ be the joint distribution induced by $P_X$ and $q_\theta(\cdot\mid X)$; its $Z$-marginal is the aggregate distribution $q_Z$ from \cref{eq:aggregate-code}.
Using the inverse-mass affinity of \cref{eq:qz-affinity}, a divergence-based quantity has a different form:
\begin{equation}
I_2^D(X;Z)
=D_2\!\left(P_{XZ}\,\middle\|\,P_X\otimes q_Z\right)
=\log\E_X[a_{q_Z}(X,X)].
\label{eq:divergence-renyi-information-affinity}
\end{equation}
The expansion follows on the active support from the density ratio $q_\theta(z\mid x)/q_Z(z)$.
For deterministic assignments, its exponential is the number of active words, so $I_2^D=R_0=\log|\operatorname{supp}(q_Z)|$ from \cref{eq:finite-state-rate-hierarchy} rather than $H_2(q_Z)$.
The two functionals illustrate why an unqualified reference to R\'enyi mutual information would be ambiguous here.

There is a further distinction between input--code dependence and dependence of relational endpoints.
For $(\mathsf V,\mathsf V')\sim P_{\mathrm{pair}}$ as in \cref{eq:generic-relational-objective}, with endpoints encoded independently from $q_\theta(\cdot\mid x_{\mathsf V})$ and $q_\theta(\cdot\mid x_{\mathsf V'})$, the full joint distribution of $(Z,Z')$ can contain information beyond the equality event.
Collision alignment measures only
\begin{equation}
\Pr[Z=Z']=\E_{(\mathsf V,\mathsf V')\sim P_{\mathrm{pair}}}[k_\theta(x_{\mathsf V},x_{\mathsf V'})].
\end{equation}
It is an equality-specific statistic, not the full mutual information between the endpoints.

\section{Additional Monotone Content Examples}
\label{app:path-walk-content}

The main text defines content-based relational fidelity by choosing a nonnegative structural-content functional $\mathcal I$ that is monotone under an admissible simplification and measuring the resulting content loss through \cref{eq:relational-content-monotone,eq:relational-content-distortion}.
The direct-edge realization in \cref{eq:edge-distortion} uses retained edge weight, but the same construction can emphasize structure induced by several relations rather than individual edge entries.
This appendix records shortest-path, walk, and natural-connectivity examples to make that latitude concrete; they are illustrative alternatives and are not evaluated in the experiments.

For a graph $G=(V,E)$ with positive edge lengths and an edge-deletion simplification $\widetilde G$ on the same carrier, let $d_G(i,j)$ be shortest-path distance, using infinity when no path exists.
If $\varphi:[0,\infty]\to\R_{\ge0}$ is nonincreasing and $\varphi(\infty)=0$, then
\begin{equation}
\mathcal I_\varphi(G)=\sum_{i<j}\varphi(d_G(i,j))
\end{equation}
is monotone under edge deletion.
Deleting edges cannot shorten a path.
For example, with $\varphi(d)=e^{-d/\tau}$ and $\tau>0$, the corresponding specialization of the generic content distortion in \cref{eq:relational-content-distortion} is
\begin{equation}
D_\varphi(G,\widetilde G)
=\sum_{i<j}\left[e^{-d_G(i,j)/\tau}-e^{-d_{\widetilde G}(i,j)/\tau}\right].
\end{equation}
This values the geometry supported by paths rather than only the removed edge entries.
For weight attenuation, a relationship between weights and lengths would additionally need to be specified.

For the nonnegative weighted adjacency $W$ used in \cref{sec:edge-cut-distortion}, and an admissible attenuation with reconstructed adjacency $\widetilde W$, finite walk content can be written
\begin{equation}
\mathcal I_{\mathrm{walk}}(G)
=\sum_{\ell=1}^{L_0}\alpha_\ell\mathbf1^\top W^\ell\mathbf1,
\qquad\alpha_\ell\ge0.
\end{equation}
Each term counts weighted walk mass of one length.
If $0\le\widetilde W\le W$ entrywise, induction using nonnegative matrix multiplication gives $0\le\widetilde W^\ell\le W^\ell$ entrywise.
Hence the content is monotone under attenuation.
The truncation $L_0$ is used to avoid confusing the walk length limit with the graph Laplacian $L$.

Natural connectivity is an established spectral example~\citep{wu2010natural}.
For a finite undirected unweighted graph $G=(V,E)$ with adjacency $A$ and $n=|V|$,
\begin{equation}
\mathcal I_{\mathrm{NC}}(G)
=\log\left(\frac{\operatorname{tr}(e^A)}{n}\right)
=\log\left(\frac1n\sum_r e^{\lambda_r(A)}\right).
\end{equation}
The expansion $\operatorname{tr}(e^A)=\sum_{\ell\ge0}\operatorname{tr}(A^\ell)/\ell!$ connects it to closed walks.
It is zero for the empty graph and strictly increases when an edge is added in this standard setting.
Natural connectivity supplies an established structural-content functional for this construction.

Direct edge content in \cref{eq:edge-distortion} treats each primitive relation according to its weight, shortest-path content responds to induced connectivity geometry, finite walk content measures multistep structure, and natural connectivity provides an all-orders closed-walk spectral example.
A fidelity criterion may combine monotone contents with nonnegative coefficients, with its interpretation tied to the declared combination of relational properties.

\section{Reported Metrics and Numerical Conventions}
\label{app:metric-definitions}

This appendix collects the definitions of recurring metrics reported in the experiments and the aggregation conventions needed to interpret them.
Metrics defined earlier retain the same notation and are referenced back to their original definitions.

\subsection{Hard finite codes}

For $N$ evaluated hard codewords $c(i)\in\mathcal Z$, let
\begin{equation}
\widehat q_Z(z)
=\frac1N\sum_{i=1}^N\ind[c(i)=z].
\end{equation}
The distribution $\widehat q_Z$ is the empirical codeword distribution over the evaluated hard tokens and is the empirical counterpart of the aggregate distribution $q_Z$ introduced in \cref{eq:aggregate-code}.

Aggregate hard collision is the probability that two independent draws from $\widehat q_Z$ receive the same codeword:
\begin{equation}
\widehat\kappa_{\mathrm{pop}}
=\sum_z\widehat q_Z(z)^2,\qquad
\widehat H_2
=-\log\widehat\kappa_{\mathrm{pop}},\qquad
K_{\mathrm{eff}}
=\widehat\kappa_{\mathrm{pop}}^{-1}.
\end{equation}
Here $\widehat\kappa_{\mathrm{pop}}$ is the empirical same-code collision, $\widehat H_2$ is the hard-code collision entropy, and $K_{\mathrm{eff}}$ is the collision-effective number of codewords.
They are the empirical counterparts of the population quantities in \cref{eq:population-collision,eq:collision-entropy,eq:effective-code-count}.
The reported number of active hard words is $|\{z:\widehat q_Z(z)>0\}|$.

For the $32\times32$ latent grid on $256\times256$ images,
\[
\mathrm{bpp}
=\frac{32^2b}{256^2}
=\frac{b}{64}.
\]
This is the nominal raw latent bit capacity used in Study~4 before entropy coding.

\subsection{Graph partitions}

For a hard assignment $c$, use the classes $V_z=\{i:c(i)=z\}$ from \cref{prop:normalized-cut}, with
\[
\operatorname{vol}(G)=\sum_i d_i,\qquad
\operatorname{vol}(V_z)=\sum_{i\in V_z}d_i.
\]
Using directed symmetric edge incidences,
\[
\operatorname{assoc}(V_z,V_z)
=\sum_{(i,j)\in\vec E:i,j\in V_z}W_{ij}.
\]
The reported hard fixed-$K$ normalized cut is the specialization in \cref{eq:hard-ncut}:
\[
\operatorname{NAssoc}(G)
=\sum_{z:\operatorname{vol}(V_z)>0}
\frac{\operatorname{assoc}(V_z,V_z)}{\operatorname{vol}(V_z)},\qquad
\operatorname{Ncut}(G)
=K-\operatorname{NAssoc}(G).
\]
Under the fixed-$K$ convention, an empty class contributes zero association and therefore one unit to $\operatorname{Ncut}(G)$.
The reported largest hard volume fraction is $\max_z\operatorname{vol}(V_z)/\operatorname{vol}(G)$, and active hard classes are those with positive volume.
The graph-fidelity values reported in Study~2 are the hard specializations of $D_E$, $D_F$, $D_C$, and $D_{H_2}$ defined in \cref{eq:normalized-edge-distortion,eq:source-normalized-dirichlet-distortion,eq:source-collision-edge-distortion,eq:entropy-probe-distortion}.
The post-hoc $D_F$ reported in Study~3 uses the hard-partition form in \cref{eq:hard-partition-dirichlet-distortion}.

For graph-local assignments, $\bar q_G(z)$ is the degree-weighted occupancy defined in \cref{eq:graph-aggregate-code}.
The reported organization metrics are
\[
H_2(\bar q_G)
=-\log\sum_z\bar q_G(z)^2,\qquad
K_{\mathrm{eff}}
=\left(\sum_z\bar q_G(z)^2\right)^{-1},
\]
\[
D_2(\bar q_G\|U_K)
=\log K-H_2(\bar q_G)
=\log\left(K\sum_z\bar q_G(z)^2\right).
\]
Following \cref{eq:collision-entropy,eq:effective-code-count,eq:uniform-renyi}, $H_2(\bar q_G)$ is graph-local collision occupancy, $K_{\mathrm{eff}}$ is the corresponding collision-effective number of classes, and $D_2(\bar q_G\|U_K)$ measures deviation from uniform graph-local occupancy.
These quantities are computed graph-locally before averaging across evaluated graphs.
Consequently, the mean of $K_{\mathrm{eff}}$ is not generally the exponential of the mean $H_2$, and mean $D_2$ cannot be converted to a mean effective count by exponentiation.
The learned MalNet results report the mean and sample standard deviation across five seeds, while the spectral reference is deterministic for the stated solver configuration.

\subsection{Teacher relations and reconstruction}

As introduced in \cref{sec:flowers102-teacher-experiment}, the favored and disfavored weighted hard-code collision rates are
\begin{equation}
\begin{aligned}
\widehat\kappa_+&=\frac{\sum_{ij}S_{+,ij}\ind[c(i)=c(j)]}{\sum_{ij}S_{+,ij}},\\
\widehat\kappa_-&=\frac{\sum_{ij}S_{-,ij}\ind[c(i)=c(j)]}{\sum_{ij}S_{-,ij}}.
\end{aligned}
\end{equation}
The pooled empirical hard-code collision is $\widehat\kappa_{\mathrm{pop}}$ as defined above.
Favored-pair enrichment is $\widehat\kappa_+/\widehat\kappa_{\mathrm{pop}}$, while disfavored-pair suppression is $1-\widehat\kappa_-/\widehat\kappa_{\mathrm{pop}}$.
The reported favored and disfavored teacher-cosine summaries are the corresponding relation-weighted means:
\[
\frac{\sum_{ij}S_{\pm,ij}\,e_i^\top e_j}
{\sum_{ij}S_{\pm,ij}}.
\]
Validation and test teacher loss use the alignment and relational-separation terms of \cref{eq:dino-teacher-loss}, with their numerators aggregated using the corresponding $S_+$ and $S_-$ weight masses before forming the final average.

For a unit-range RGB image $x$ and reconstruction $\widehat x$,
\[
\operatorname{MSE}(x,\widehat x)
=\frac{1}{3HW}\|x-\widehat x\|_2^2.
\]
Reported MSE averages this squared pixel error over the evaluated images.
For unit-range images,
\[
\operatorname{PSNR}
=-10\log_{10}(\operatorname{MSE}),
\]
so higher PSNR corresponds to lower MSE.
MS-SSIM is the standard multiscale structural similarity diagnostic of \citet{wang2003multiscale}.

\section{Synthetic Study Details}
\label{app:synthetic-details}

The synthetic study uses one fixed finite dataset and evaluates every point pair exactly.
Two copies of the same encoder are initialized identically and optimized using the centroid and inverse-mass-weighted pairwise forms, respectively.
The separate decoder calculation holds the encoder assignments fixed.
These checks verify the finite-sample centroid--pairwise identity in \cref{eq:finite-soft-pairwise} and the decoder decomposition in \cref{eq:decoder-centroid-decomposition}.
If $g_{\mathrm{cent}}$ and $g_{\mathrm{pair}}$ are the flattened encoder gradients of the centroid and pairwise objectives, their reported gradient agreement uses cosine similarity and relative error
\[
\frac{\|g_{\mathrm{cent}}-g_{\mathrm{pair}}\|_2}
{\|g_{\mathrm{cent}}\|_2}.
\]

\begin{table}[htbp]
\centering\small
\begin{tabularx}{\linewidth}{Xr}
\toprule
Diagnostic & Observed value \\
\midrule
Maximum $|D_{\mathrm{cent}}-D_{\mathrm{pair}}|$, centroid-trained trajectory & $9.54\times10^{-7}$ \\
Maximum $|D_{\mathrm{cent}}-D_{\mathrm{pair}}|$, pair-trained trajectory & $9.54\times10^{-7}$ \\
Larger trajectory mean of $|D_{\mathrm{cent}}-D_{\mathrm{pair}}|$ & $4.20\times10^{-8}$ \\
Random encoder states: maximum $|D_{\mathrm{cent}}-D_{\mathrm{pair}}|$ & $1.91\times10^{-6}$ \\
Random encoder states: mean $|D_{\mathrm{cent}}-D_{\mathrm{pair}}|$ & $2.64\times10^{-7}$ \\
Initial gradient cosine / relative error & $1.000\,/\,1.54\times10^{-7}$ \\
Final gradient cosine / largest relative error across final models & $1.000\,/\,1.35\times10^{-6}$ \\
Final parameter RMS difference between the centroid- and pairwise-trained encoders & $3.56\times10^{-8}$ \\
Random states: mean $|D_{\mathrm{cent}}-KD_{\mathrm{raw}}|$ & $4.79$ \\
Random states: maximum $|D_{\mathrm{cent}}-KD_{\mathrm{raw}}|$ & $21.69$ \\
Decoder excess term in \cref{eq:decoder-centroid-decomposition}, steps $0\to200$ & $21.05\to8.15\times10^{-5}$ \\
Maximum decoder decomposition residual & $1.91\times10^{-6}$ \\
\bottomrule
\end{tabularx}
\caption{Synthetic numerical diagnostics. Values near $10^{-6}$ reflect single-precision arithmetic.}
\label{tab:synthetic-collision-centroid}
\end{table}

\begin{table}[htbp]
\centering\small
\begin{tabularx}{\linewidth}{>{\raggedright\arraybackslash}p{0.28\linewidth}X}
\toprule
Component & Setting \\
\midrule
Data & 768 points in $\R^2$ from a six-component Gaussian mixture \\
Mixture weights & $(0.30,0.22,0.17,0.13,0.10,0.08)$ \\
Realized component counts & $(230,169,131,100,77,61)$ \\
Encoder & MLP with two hidden layers, width 32, $\tanh$ nonlinearities \\
Code distribution & Exact enumeration of $2^3=8$ factorized Bernoulli words \\
Code temperature & 0.7 \\
Encoder optimization & SGD, learning rate 0.02, momentum 0.9, 300 steps \\
Random-state check & 128 random affine-logit maps; weight/bias scale log-uniform in $[0.25,4.0]$ \\
Decoder check & Learned $8\times2$ reproduction codebook, SGD learning rate 0.15, 200 steps \\
Decoder initialization & Gaussian reproduction vectors with scale 3.0 \\
Seed & 1337 \\
\bottomrule
\end{tabularx}
\caption{Settings for Study 1.}
\label{tab:synthetic-settings}
\end{table}

\section{Transductive Relational-Fidelity Study Details}
\label{app:transductive-relational-distortion-details}

The transductive study optimizes a separate logit matrix for each source graph, fidelity, and organization weight.
The two MalNet collections use the same 20 cleaned MalNet-Tiny test topologies.
The weighted variant assigns deterministic positive lognormal weights, with scale $\sigma=0.75$, mean-normalized within each graph, and global seed 91337.
COLLAB and PROTEINS each contribute 20 cleaned largest connected components from TUDataset.
All labels are removed or ignored.

The solver uses ten random initializations and one collapse-biased initialization, retaining the best target soft objective encountered in each optimization.
After these ordinary restarts, each target is warm-started from the other criteria's selected logits; if this improves any target, one closure pass is run from the updated solutions.
Final selection uses the target soft objective across all candidates.

\begin{table}[htbp]
\centering\small
\begin{tabularx}{\linewidth}{>{\raggedright\arraybackslash}p{0.28\linewidth}X}
\toprule
Component & Setting \\
\midrule
Collections & 20 MalNet-Tiny test LCCs; the same 20 weighted topologies; 20 COLLAB; 20 PROTEINS \\
Base graph-selection seed & 1337; deterministic collection-specific offsets for COLLAB and PROTEINS \\
Minimum source size & 512 nodes for MalNet; 16 nodes for TUDataset, after LCC extraction \\
Training fidelities & $D_E$, $D_F$, $D_C$, and $D_{H_2}$ where defined \\
Entropy-field exclusions & Omit $D_{H_2}$ when its source Dirichlet denominator is zero; two selected COLLAB graphs are excluded from this criterion \\
Representation & Free graph-local $n\times8$ categorical logits; no learned encoder \\
Assignment temperature & $\tau=1.0$ \\
Organization weights & $0$, $0.03$, $0.05$, $0.07$, $0.08$, $0.09$, $0.10$, $0.11$, $0.12$, $0.14$, $0.16$, $0.18$, $0.20$, $0.30$, $0.50$ \\
Objective & $D_\rho(q)+\lambda_{\mathrm{org}}D_2(\bar q_G\|U_8)$ \\
Optimizer & Adam, learning rate 0.05, 600 steps, no scheduler \\
Random initialization & Gaussian logits with standard deviation $10^{-2}$ \\
Random restart seeds & 1337, 2024, 31415, 2718, 1618, 9001, 42, 73, 101, 211 \\
Collapse-biased restart & Initial logit bias 4.0 toward one state \\
Warm starts & Cross-criterion pass from the other selected solutions; one closure pass if the first pass improves any target \\
Selection & Best target soft objective across ordinary restarts and cross-criterion candidates \\
Diagnostics & Hard/soft cross-evaluation, source-importance similarities, and 256 size-balanced random partitions per graph \\
Random-partition seed & $1337+\text{graph index}$ \\
\bottomrule
\end{tabularx}
\caption{Settings for Study 2.}
\label{tab:transductive-relational-distortion-settings}
\end{table}

\subsection{Matched occupancy and cross-distortion transfer}

Hard partitions are compared at target $H_2(\bar q_G)$ values from 0.25 to 2.0 nats.
Each selected solution must be within 0.10 nats of the target, and the compared pair must be within 0.10 nats of one another.
Each criterion selects its solution independently from the available sweep before this matching rule is applied.
At least five matched graphs are required to display a point in the agreement figures.
Consequently, different source/criterion comparisons can use different subsets, and missing targets reflect available coverage rather than an interpolation through unobserved solutions.

\Cref{tab:transductive-rate-matched-agreement} preserves the detailed comparison at 2.0 nats.
Let $c_A$ and $c_B$ denote the partitions trained under fidelities $A$ and $B$, respectively.
The directional transfer differences are
\[
\Delta_{A\to B}=D_B(c_A)-D_B(c_B),
\qquad
\Delta_{B\to A}=D_A(c_B)-D_A(c_A).
\]
Each difference is measured on the destination fidelity's scale.
Because matching is approximate and selection uses the soft training objective, these descriptive differences can be slightly negative.

\begin{table}[htbp]
\centering\small
\setlength{\tabcolsep}{4pt}
\begin{tabular}{llrrrr}
\toprule
Source & Pair & $n$ & Mean ARI & $\Delta_{A\to B}$ & $\Delta_{B\to A}$ \\
\midrule
MalNet unweighted & $D_F$ / $D_C$ & 20 & 0.781 & -0.0025 & 0.0060 \\
MalNet unweighted & $D_F$ / $D_{H_2}$ & 20 & 0.166 & 0.0919 & 0.0698 \\
MalNet unweighted & $D_C$ / $D_{H_2}$ & 20 & 0.196 & 0.1068 & 0.0744 \\
MalNet weighted & $D_F$ / $D_C$ & 20 & 0.714 & 0.0011 & 0.0063 \\
MalNet weighted & $D_F$ / $D_{H_2}$ & 20 & 0.116 & 0.0785 & 0.0638 \\
MalNet weighted & $D_C$ / $D_{H_2}$ & 20 & 0.157 & 0.0834 & 0.0864 \\
COLLAB & $D_F$ / $D_C$ & 5 & 0.714 & 0.0258 & 0.0234 \\
COLLAB & $D_F$ / $D_{H_2}$ & 5 & 0.398 & 0.1058 & 0.0922 \\
COLLAB & $D_C$ / $D_{H_2}$ & 6 & 0.426 & 0.0571 & 0.0288 \\
PROTEINS & $D_F$ / $D_C$ & 16 & 0.764 & 0.0046 & 0.0084 \\
PROTEINS & $D_F$ / $D_{H_2}$ & 14 & 0.570 & 0.1163 & 0.0506 \\
PROTEINS & $D_C$ / $D_{H_2}$ & 13 & 0.568 & 0.1169 & 0.0690 \\
\bottomrule
\end{tabular}
\caption{Hard-partition agreement and cross-distortion transfer at target collision occupancy $H_2(\bar q_G)=2.0$ nats. $n$ counts matched graphs. ARI is invariant to code-label permutation.}
\label{tab:transductive-rate-matched-agreement}
\end{table}

Each cell in the full transfer table averages the accepted graph--target matches across all occupancy targets for that training/evaluation pair.
Each column holds one evaluation fidelity fixed, each row supplies the partition trained using a particular fidelity, and the diagonal is zero by construction.
Coverage can therefore differ between cells.
Comparisons are made within columns because each column is expressed on its evaluation fidelity's scale.

\begin{table}[htbp]
\centering\small
\setlength{\tabcolsep}{4pt}
\begin{tabular}{llrrrr}
\toprule
Source & Training fidelity & Eval. $D_E$ & Eval. $D_F$ & Eval. $D_C$ & Eval. $D_{H_2}$ \\
\midrule
MalNet unweighted & $D_E$ & 0.0000 & -0.0017 & -0.0020 & 0.0732 \\
MalNet unweighted & $D_F$ & 0.0163 & 0.0000 & -0.0024 & 0.0911 \\
MalNet unweighted & $D_C$ & 0.0260 & 0.0058 & 0.0000 & 0.1057 \\
MalNet unweighted & $D_{H_2}$ & 0.0999 & 0.0690 & 0.0738 & 0.0000 \\
\midrule
MalNet weighted & $D_E$ & 0.0000 & 0.0033 & -0.0023 & 0.0658 \\
MalNet weighted & $D_F$ & 0.0131 & 0.0000 & 0.0011 & 0.0785 \\
MalNet weighted & $D_C$ & 0.0149 & 0.0063 & 0.0000 & 0.0834 \\
MalNet weighted & $D_{H_2}$ & 0.1010 & 0.0638 & 0.0864 & 0.0000 \\
\midrule
COLLAB & $D_E$ & 0.0000 & 0.0071 & 0.0509 & 0.2296 \\
COLLAB & $D_F$ & 0.0044 & 0.0000 & 0.0293 & 0.1946 \\
COLLAB & $D_C$ & 0.0381 & 0.0241 & 0.0000 & 0.0597 \\
COLLAB & $D_{H_2}$ & 0.1016 & 0.0746 & 0.0101 & 0.0000 \\
\midrule
PROTEINS & $D_E$ & 0.0000 & -0.0009 & 0.0042 & 0.0758 \\
PROTEINS & $D_F$ & 0.0036 & 0.0000 & 0.0048 & 0.0848 \\
PROTEINS & $D_C$ & 0.0102 & 0.0027 & 0.0000 & 0.0807 \\
PROTEINS & $D_{H_2}$ & 0.0594 & 0.0489 & 0.0577 & 0.0000 \\
\bottomrule
\end{tabular}
\caption{Cross-distortion transfer, averaged over pairwise hard-$H_2$-matched graph--occupancy comparisons. Each entry is excess hard distortion relative to the solution trained for that column's evaluation fidelity. Compare rows within a column.}
\label{tab:transductive-cross-distortion-transfer}
\end{table}

\begin{figure}[htbp]
\centering
\includegraphics[width=0.82\linewidth]{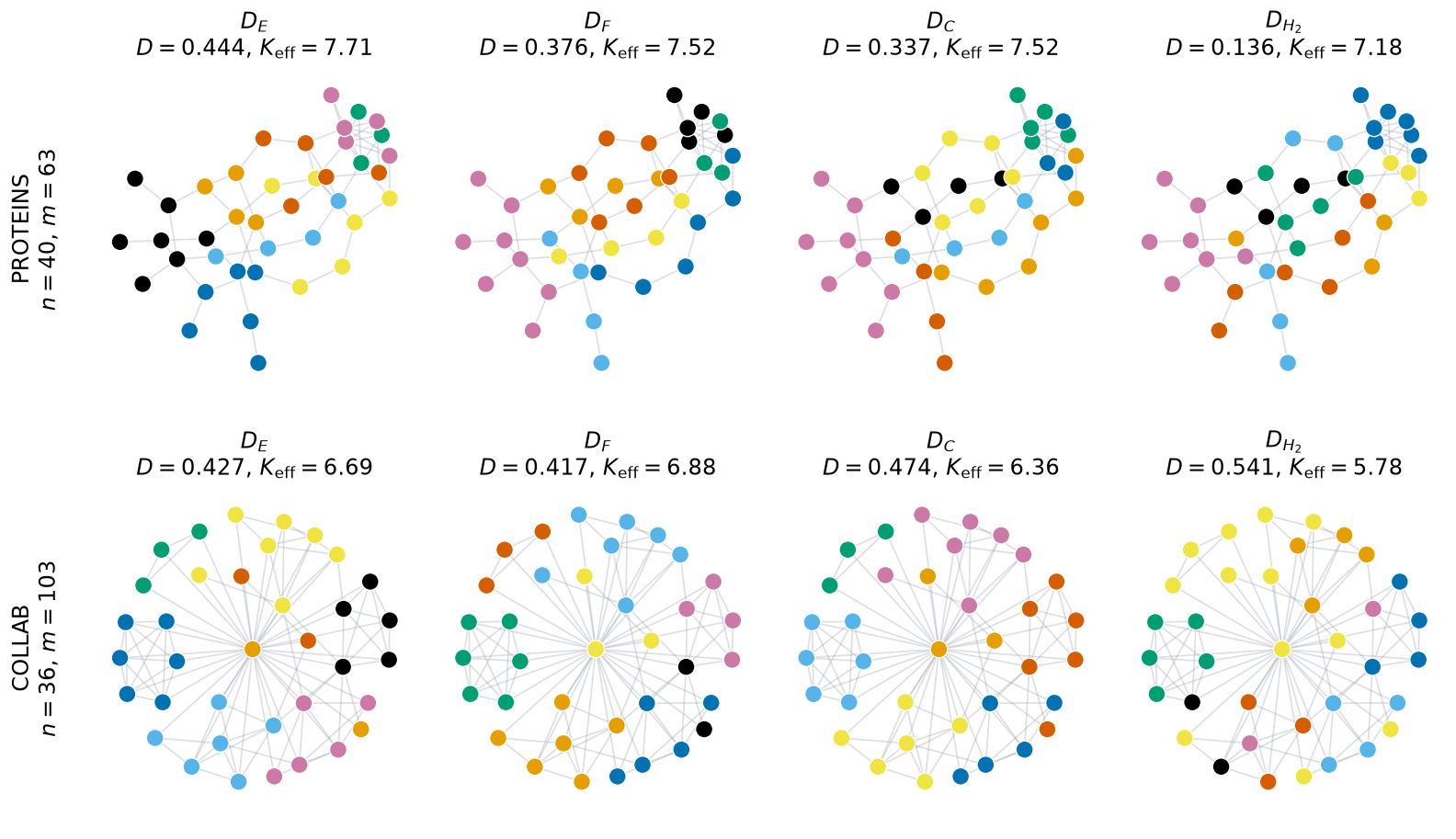}
\caption{Example partitions at marginal-organization weight 0.50. Rows show a selected PROTEINS graph and a selected COLLAB graph; columns change only the source fidelity optimized by the transductive solver. Positions are fixed within a row, while partition labels and palettes are independently permuted for display. Annotations give each criterion's own hard distortion and hard effective count.}
\label{fig:transductive-relational-distortion-partitions}
\end{figure}

\begin{figure}[htbp]
\centering
\includegraphics[width=0.76\linewidth]{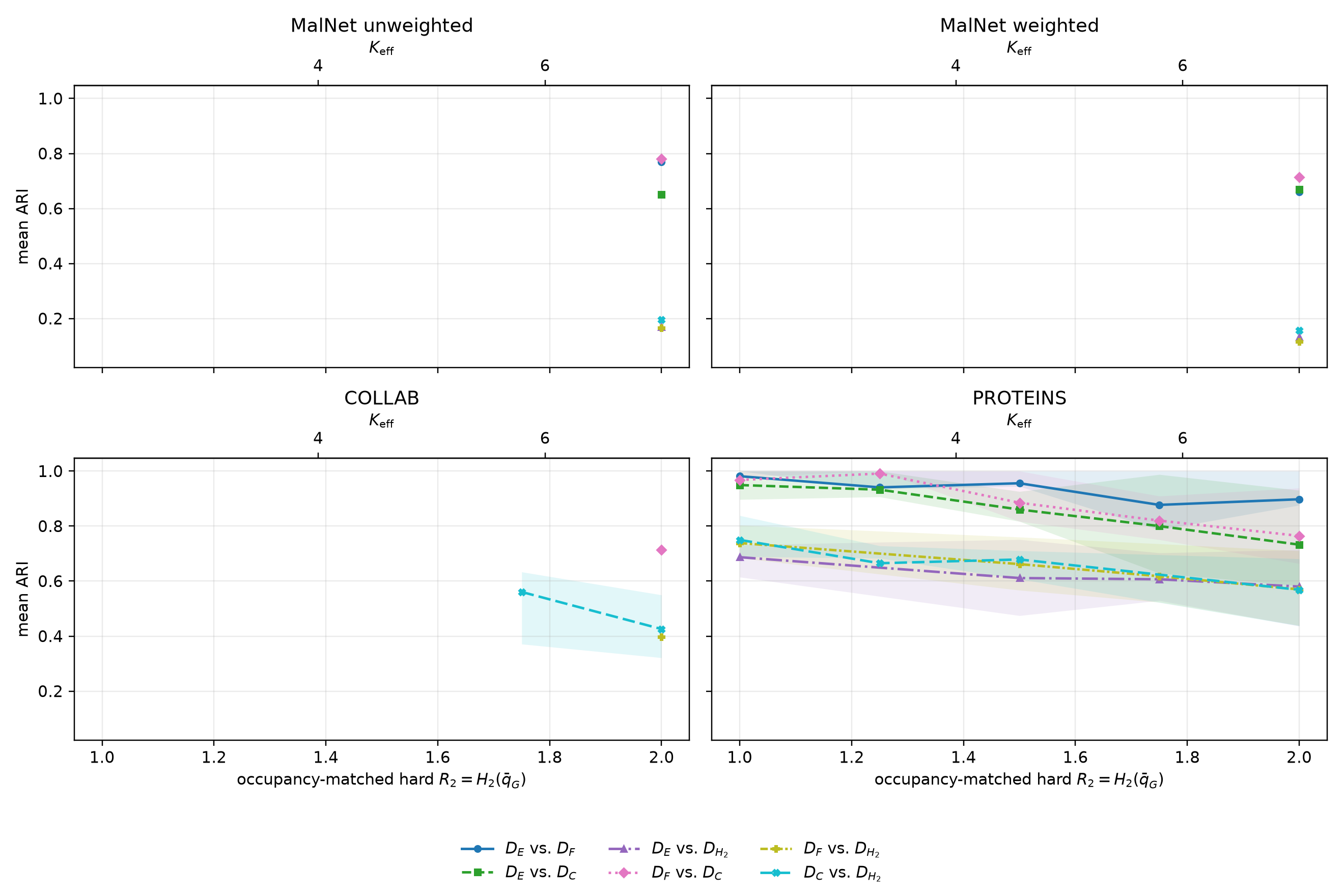}
\caption{All six pairwise ARI comparisons in Study 2. The main text highlights three pairs; this figure retains the complete set. Points require at least five graphs meeting the pairwise hard-entropy matching rule.}
\label{fig:transductive-full-rate-matched-ari}
\end{figure}

\begin{figure}[htbp]
\centering
\includegraphics[width=0.76\linewidth]{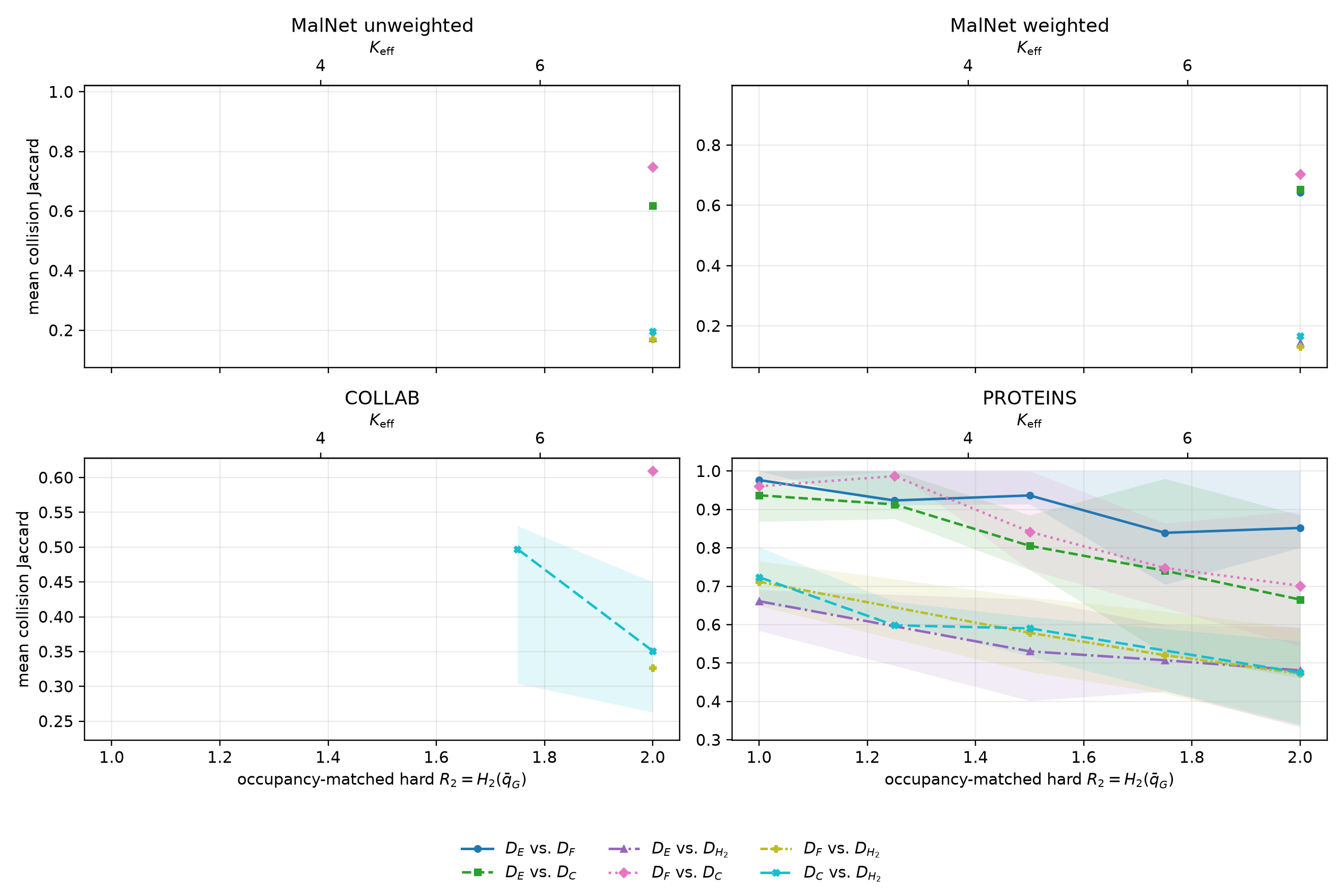}
\caption{All six pairwise collision-Jaccard comparisons. Jaccard compares the induced same-code equivalence relations directly, complementing chance-corrected ARI. Points require at least five matched graphs.}
\label{fig:transductive-full-rate-matched-jaccard}
\end{figure}

\section{Inductive Normalized-Cut Study Details}
\label{app:malnet-ncut-details}

MalNet-Tiny retains its official train/validation/test split before cleaning.
Graphs are converted to simple undirected graphs, reduced to their largest connected component, and retained only if that component contains at least 512 nodes.
\Cref{tab:malnet-preprocessing-summary} gives the resulting exclusions.

\begin{table}[htbp]
\centering\small
\begin{tabular}{lrrr}
\toprule
Split & Raw graphs & Retained graphs & Excluded below 512 LCC nodes \\
\midrule
Train & 3500 & 1995 & 1505 \\
Validation & 500 & 275 & 225 \\
Test & 1000 & 584 & 416 \\
\bottomrule
\end{tabular}
\caption{MalNet-Tiny source selection after cleaning.}
\label{tab:malnet-preprocessing-summary}
\end{table}

\begin{table}[htbp]
\centering\small
\begin{tabularx}{\linewidth}{>{\raggedright\arraybackslash}p{0.28\linewidth}X}
\toprule
Component & Setting \\
\midrule
Dataset & MalNet-Tiny, official source splits, labels unused \\
Cleaning & Undirected edge union, self-loop removal, coalesced duplicates, largest connected component \\
Minimum graph size & 512 vertices after LCC extraction \\
Node features & Constant and cleaned undirected degree features, original directed degree features and PageRank, deterministic random-walk return estimates \\
Random-walk estimates & 16 steps, 8 Rademacher probes, split/index-dependent deterministic seeds \\
Assignments & $K=8$ categorical states; $\tau=1.0$ \\
Encoder & Four GraphGPS-style layers, local GIN and global efficient attention, width 128, four heads \\
Additional encoder settings & GELU, no dropout, FFN multiplier 2, graph-context concatenation, linear assignment head \\
Organization weights $\lambda_{\mathrm{org}}$ & $0$, $0.03$, $0.10$, $0.20$, $0.50$ \\
Objective & Minibatch objective $\mathcal L_B$ from \cref{eq:malnet-ncut-loss} \\
Optimizer & AdamW, learning rate $5\times10^{-4}$, zero weight decay \\
Schedule & 50-step warmup from factor 0.1, cosine decay to $10^{-5}$ \\
Training & 250 epochs, batch size 80 \\
Seeds & 1337, 2024, 31415, 27182, 16180 \\
Checkpoint selection & Validation mean hard $\operatorname{Ncut}(G)$ under the fixed-$K$ convention of \cref{eq:hard-ncut} \\
Spectral reference & Sparse normalized-Laplacian embedding, eight smallest eigenvectors, row normalization, deterministic $K$-means discretization \\
Spectral eigensolver & SciPy \texttt{eigsh}, tolerance $10^{-5}$ \\
Spectral $K$-means & Seed 1337, $n_{\mathrm{init}}=8$, maximum 100 iterations, tolerance $10^{-5}$ \\
\bottomrule
\end{tabularx}
\caption{Settings for Study 3.}
\label{tab:malnet-ncut-settings}
\end{table}

The spectral embedding is computed with SciPy's sparse \texttt{eigsh} solver~\citep{virtanen2020scipy}, and its seeded $K$-means discretization uses scikit-learn's \texttt{KMeans} implementation~\citep{pedregosa2011scikit}.
Post-hoc $D_F$ is computed from the selected hard partitions using effective resistances of the cleaned source graph, as in \cref{eq:hard-partition-dirichlet-distortion}.
The spectral reference is computed separately for each retained validation and test graph.

\section{Reconstruction-Trained Flowers102 Study Details}
\label{app:flowers102-details}

All five image reconstruction settings use seed 1337, batch size 64, and 150 training epochs.
AdamW uses learning rate $5\times10^{-4}$, zero weight decay, 100-step warmup from factor 0.1, and cosine decay to $10^{-6}$.
Training applies random resized crops, horizontal flips, and mild color jitter; validation and test use resize followed by deterministic center crop.
The checkpoint criterion is hard-code validation MSE.

The nonstandard split uses the official Flowers102 test images for training, official validation images for validation, and official training images for final test evaluation.
The resulting sizes are 6149, 1020, and 1020, respectively.

\begin{table}[htbp]
\centering\small
\begin{tabularx}{\linewidth}{>{\raggedright\arraybackslash}p{0.34\linewidth}X}
\toprule
Component & Setting \\
\midrule
Input & $256\times256$ RGB, unit-range pixel values \\
Encoder width & 128 hidden channels \\
Residual stack & Two blocks, 64 residual hidden channels \\
Downsampling / upsampling & Three stride-2 analysis stages with GDN and three transposed-convolution synthesis stages, with IGDN after the first two synthesis stages, giving a $32\times32$ latent grid \\
Spatial collision unit & One $b$-bit token per latent position \\
Analysis output & $1\times1$ convolution producing $b$ logits per spatial position \\
Decoder output & Sigmoid RGB reconstruction; no encoder--decoder skips \\
Reconstruction loss & Mean squared error on $[0,1]$ pixels \\
Training objective & $\mathcal L$ from \cref{eq:flowers-objective} \\
Mean-group relaxation & Deterministic signs, separate positive and negative groups per channel, grouped over batch and spatial positions \\
Concentration weight & $\lambda_{\mathrm{conc}}=10$ \\
Threshold-scale EMA decay & 0.99 \\
Threshold-scale floor & $10^{-3}$ \\
Margin gain & $\gamma=3$ \\
Collision pair samples & 8192 minibatch spatial-token pairs \\
Augmentation & Random resized crop, horizontal flip, brightness/hue/saturation/contrast jitter of 0.05 \\
\bottomrule
\end{tabularx}
\caption{Shared settings for Study 4.}
\label{tab:flowers102-shared-settings}
\end{table}
Each collision pair is formed by two independent uniform draws, with replacement, from the minibatch's spatial tokens; self-pairs and same-image pairs are permitted.
These pairs form the minibatch estimate in \cref{eq:flowers-separation} using the threshold collision of \cref{eq:flowers-threshold-collision}.

The mean-group relaxation forms separate positive- and negative-sign groups within each channel over batch and spatial positions, subtracts each nonempty group's mean logit, and adds the resulting centered residual to the corresponding hard sign.
For minibatch size $B$, let $\tilde h_r(x_s,u)$ denote the relaxed latent corresponding to hard sign $h_r(x_s,u)$.
The concentration term is
\begin{equation}
\mathcal L_{\mathrm{conc}}
=\frac{1}{B b H_z W_z}
\sum_{s=1}^{B}\sum_{r=1}^{b}\sum_u
\bigl(\tilde h_r(x_s,u)-h_r(x_s,u)\bigr)^2.
\label{eq:concentration-reduction}
\end{equation}
Here $u$ ranges over the $H_zW_z$ latent spatial positions.
Reported reconstruction metrics use hard signs fed directly to the decoder.

\begin{table}[htbp]
\centering\small
\begin{tabular}{lrrrr}
\toprule
Setting & Bits & $\lambda_{\mathrm{org}}$ & Nominal bpp & Selected epoch \\
\midrule
No organization & 16 & $0$ & 0.25 & 117 \\
Weaker organization & 16 & $10^{-7}$ & 0.25 & 145 \\
Baseline & 16 & $10^{-6}$ & 0.25 & 148 \\
Stronger organization & 16 & $10^{-5}$ & 0.25 & 148 \\
Capacity ablation & 32 & $10^{-6}$ & 0.50 & 150 \\
\bottomrule
\end{tabular}
\caption{Ablations and validation-selected checkpoint epochs.}
\label{tab:flowers102-ablation-settings}
\end{table}

\begin{table}[htbp]
\centering\small
\begin{tabular}{lrrr}
\toprule
Setting & Test MSE & Test PSNR (dB) & Test MS-SSIM \\
\midrule
No organization & 0.01034 & 19.86 & 0.738 \\
Weaker organization & 0.00367 & 24.36 & 0.882 \\
Baseline & 0.00329 & 24.83 & 0.896 \\
Stronger organization & 0.00350 & 24.56 & 0.892 \\
Capacity ablation & 0.00258 & 25.88 & 0.920 \\
\bottomrule
\end{tabular}
\caption{Additional hard-code reconstruction metrics for the same checkpoints as \cref{tab:flowers102-results}. Each setting is a single-seed run.}
\label{tab:flowers102-reconstruction-metrics}
\end{table}

\begin{figure}[htbp]
\centering
\includegraphics[width=\linewidth]{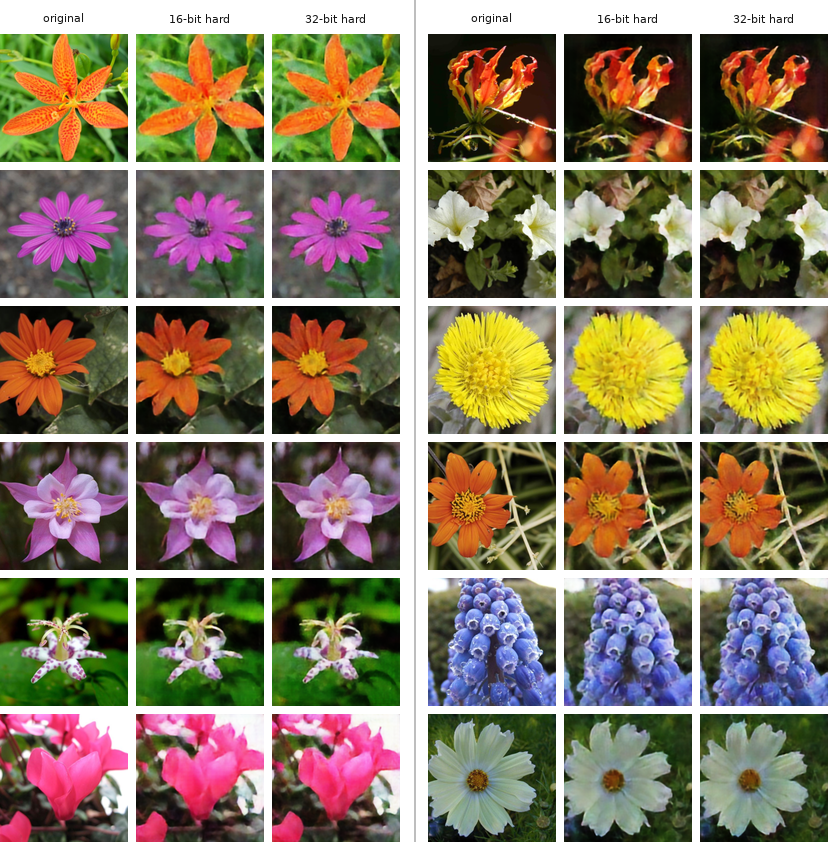}
\caption{Qualitative Flowers102 test reconstructions. Each triptych shows the original and hard-code reconstructions from the validation-selected 16-bit and 32-bit checkpoints. These examples use the hard sign representation evaluated quantitatively.}
\label{fig:flowers102-reconstruction-examples}
\end{figure}

\section{Teacher-Defined Flowers102 Study Details}
\label{app:flowers102-teacher-details}

The teacher-defined study uses the same source split, image size, and evaluation preprocessing as Study 4, with a single seed of 1337.
Representation learning uses only the baseline-relative teacher relation through favored and disfavored pair requirements.
The secondary synthesis stage begins after the representation checkpoint is selected and keeps the encoder frozen.

\begin{table}[htbp]
\centering\small
\begin{tabularx}{\linewidth}{>{\raggedright\arraybackslash}p{0.29\linewidth}X}
\toprule
Component & Setting \\
\midrule
Dataset split & 6149 train / 1020 validation / 1020 test, as in Study 4 \\
Input & $256\times256$ RGB \\
Teacher & Frozen DINO ViT-S/8, final spatial patch tokens, normalized before cosine evaluation \\
Teacher source & Official \texttt{facebookresearch/dino} \\
Teacher checkpoint & \texttt{dino\_deitsmall8\_pretrain.pth}, official release \\
Teacher grid & $32\times32$ patch-token embeddings \\
Student & The convolutional analysis transform used in Study 4 \\
Hard code & Sixteen zero-threshold sign bits per spatial token \\
Training token sample & 1024 spatial tokens sampled uniformly without replacement per minibatch \\
Evaluation token sample & 1024 spatial tokens sampled uniformly without replacement with fixed validation/test seeds \\
Candidate mask & Cross-image pairs among sampled tokens \\
Temperatures & Code $\tau_c=0.25$; teacher $\tau_T=0.1$ \\
Alignment weight & 1 \\
Representation objective & $\mathcal L_{\mathrm{teacher}}$ from \cref{eq:dino-teacher-loss} \\
Representation optimizer & AdamW, learning rate $5\times10^{-4}$, zero weight decay \\
Representation schedule & 100-step warmup from factor 0.1, cosine decay to $10^{-6}$ \\
Representation training & 150 epochs, batch size 32, seed 1337 \\
Representation selection & Validation teacher loss; selected epoch 146 \\
Decoder stage & Frozen encoder, hard signs only, MSE loss \\
Decoder architecture & The synthesis-transform form used in Study 4 \\
Decoder optimizer & AdamW, learning rate $5\times10^{-4}$, zero weight decay \\
Decoder schedule & 100-step warmup from factor 0.1, cosine decay to $10^{-6}$ \\
Decoder training/selection & 150 epochs, batch size 32; validation MSE selects epoch 140 \\
\bottomrule
\end{tabularx}
\caption{Settings for Study 5.}
\label{tab:flowers102-teacher-settings}
\end{table}

The DINO implementation uses commit \texttt{7c446df5b9f45747937fb0d72314eb9f7b66930a} of the official \texttt{facebookresearch/dino} repository.
The teacher receives the same augmented crop as the student after applying teacher-specific normalization.
The teacher relation and its favored and disfavored weights are defined in \cref{eq:teacher-row-relation,eq:teacher-signed-weights}.

For numerical stability, per-bit agreement probabilities are floored at the floating-point dtype's smallest positive normal value before their logarithms are summed, and the resulting whole-code probability is clipped to $[\varepsilon,1-\varepsilon]$, where $\varepsilon$ is machine epsilon, before evaluating the logit in \cref{eq:teacher-centered-margin}.
For the centered margin to retain a numerically representable region below the neutral collision level, the clipping floor must satisfy $\varepsilon<2^{-b}$; the reported 16-bit loss is evaluated in float32, for which this condition holds.
The clipped loss can differ from the ideal mathematical expression in saturated regions and can have zero gradients through the clipping operation there.
These operations affect only the smooth training loss; reported hard-collision metrics use hard code equality.

The reported favored-pair enrichment and disfavored-pair suppression use pooled empirical token collision as their baseline.
Reported teacher-collision, teacher-cosine, and reconstruction metrics use the conventions in \ref{app:metric-definitions}.

\begin{figure}[htbp]
\centering
\includegraphics[height=0.84\textheight,keepaspectratio]{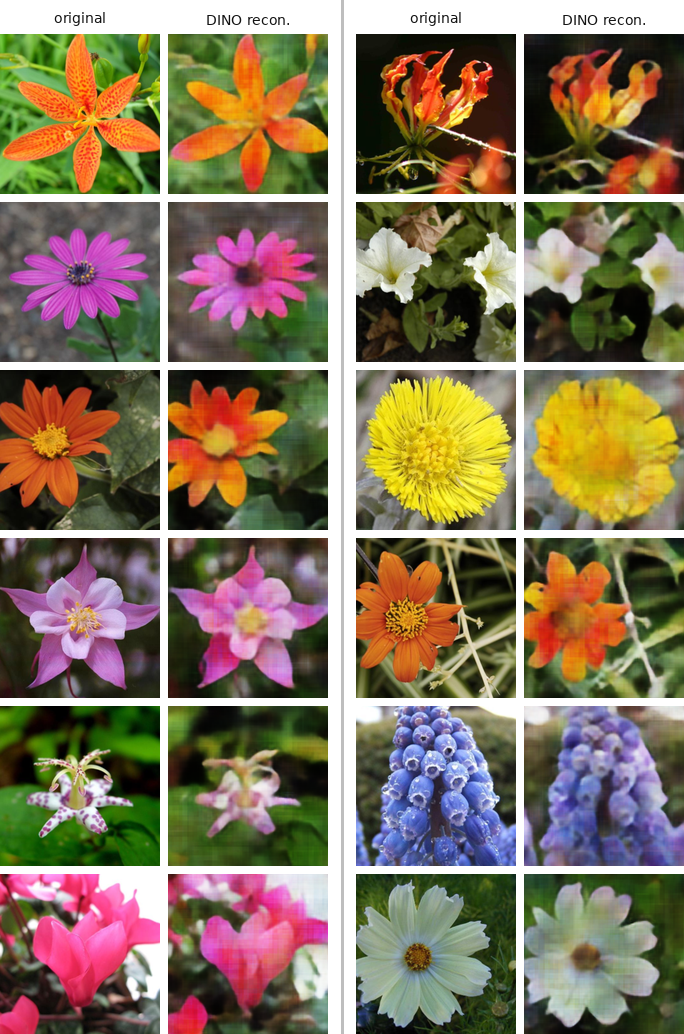}
\caption{Qualitative Flowers102 test examples from the secondary decoder trained on frozen teacher-guided hard codes. Each pair shows an original image and its reconstruction. Pixel reconstruction did not participate in learning the encoder.}
\label{fig:flowers102-teacher-reconstruction-examples}
\end{figure}

\end{document}